\documentclass[twoside]{article}
\usepackage[preprint]{robust_lmc}

\usepackage[round]{natbib}

\usepackage{amsmath,amsfonts,bm}

\def\1{\bm{1}}

\DeclareMathAlphabet{\mathsfit}{\encodingdefault}{\sfdefault}{m}{sl}
\SetMathAlphabet{\mathsfit}{bold}{\encodingdefault}{\sfdefault}{bx}{n}

\def\gG{{\mathcal{G}}}

\def\gN{{\mathcal{N}}}
\def\gO{{\mathcal{O}}}

\def\sR{{\mathbb{R}}}

\newcommand{\E}{\mathbb{E}}

\newcommand{\R}{\mathbb{R}}

\newcommand{\KL}{D_{\mathrm{KL}}}

\usepackage{float}
\usepackage[hidelinks]{hyperref}
\usepackage{url}
\usepackage{enumitem}
\usepackage{amsthm}
\usepackage{amssymb}
\usepackage{mathtools}
\usepackage{xcolor}
\usepackage{colortbl}
\usepackage{graphicx}
\usepackage{booktabs}
\usepackage{placeins}

\input{appendix_contents.tex}

\newtheorem{theorem}{Theorem}[section]
\newtheorem{lemma}[theorem]{Lemma}
\newtheorem{proposition}[theorem]{Proposition}
\newtheorem{corollary}[theorem]{Corollary}

\theoremstyle{definition}
\newtheorem{assumption}[theorem]{Assumption}
\newtheorem{definition}[theorem]{Definition}

\newtheorem{remark}{Remark}

\renewcommand{\KL}{\mathrm{KL}}
\newcommand{\FI}{\mathrm{FI}}

\newcommand{\norm}[1]{\left\lVert #1\right\rVert}

\newcommand{\dd}{\,\mathrm{d}}

\newcommand{\GeoMed}{\operatorname{GeoMed}}

\begin{document}

\runningtitle{Robust Langevin Sampling under State-Dependent Heavy-Tailed Noise}
\runningauthor{Sahin et al.}

\twocolumn[
\manuscripttitle{Robust Langevin Sampling under State-Dependent Heavy-Tailed Noise with Applications to LLM Fine-Tuning}

\manuscriptauthor{M. Berk Sahin$^1$\hspace{1.5em}
Ege C. Kaya$^1$\hspace{1.5em}
Arda Fazla$^1$\hspace{1.5em}
Behzad Sharif$^{1,2}$\hspace{1.5em}
Abolfazl Hashemi$^1$}
\manuscriptaddress{$^1$Elmore Family School of Electrical and Computer Engineering, Purdue University\\
$^2$Weldon School of Biomedical Engineering, Purdue University}
]

\begin{abstract}
Existing convergence guarantees for non-log-concave sampling with stochastic gradient Langevin Monte
Carlo (SG-LMC) typically assume globally Lipschitz gradients and uniformly bounded noise variance. Both assumptions can fail even in elementary sampling problems, and relaxing them better reflects practical applications such as Bayesian fine-tuning of large language models (LLMs). We relax both assumptions through generalized smoothness and our proposed state-dependent
heavy-tailed oracle with finite $p$-th moments, $1\!\!<\!\!p\!\!\leq\!\!2$. The oracle permits infinite variance when $p\!\!<\!\!2$ and includes the Blum–Gladyshev (BG-0) condition. We study these relaxations separately and jointly,
proposing three methods. Under globally Lipschitz gradients, robust LMC (R-LMC) handles the proposed
oracle through geometric-median centering and residual clipping. With exact gradients, adaptive LMC (A-LMC) controls discretization error under generalized smoothness. Robust adaptive LMC (RA-LMC) addresses both relaxations through robust estimation and a noise-aware step size. For all three methods, we establish non-asymptotic convergence guarantees in Fisher information (FI) without assuming dissipativity or a functional inequality. Synthetic experiments demonstrate FI
convergence and the benefit of each mechanism. Bayesian fine-tuning experiments with LLMs and
vision-language models support our oracle model and show that RA-LMC matches or outperforms state-of-the-art baselines in accuracy and uncertainty quantification.
\end{abstract}

\section{Introduction}
We consider the problem of sampling from a possibly non-log-concave target probability measure $\pi\!\propto \!e^{-V}$, where $V\!\!:\!\!\mathbb{R}^d\!\!\to\!\!\mathbb{R}$ is the potential function, using \textit{Langevin Monte Carlo (LMC)}. Earlier non-asymptotic analyses beyond log-concavity typically establish convergence under structural assumptions, including functional inequalities \citep{vempala2019rapid,wibisono2019proximal,ma2021analog,chewi2022analysis}, tail growth or dissipativity conditions \citep{durmus2017nonasymptotic,cheng2018sharp,majka2020nonasymptotic,erdogdu2021convergence}, or explicit mixture structure \citep{lee2018beyond}. Although these conditions enable convergence guarantees in strong distributional metrics, they do not provide a general measure of computational progress for targets with arbitrary multimodal geometry, where transitions between well-separated modes can be prohibitively slow.

To quantify this form of progress, \citet{balasubramanian22a} proposed a sampling analogue of first-order stationarity in non-convex optimization. The analogy comes from viewing sampling as minimizing KL divergence over probability measures \citep{jordan1998variational,wibisono2018sampling}. In this formulation, the Langevin diffusion is the Wasserstein gradient flow of KL divergence, and relative Fisher information (FI), defined in Definition~\ref{def:target-information}, plays the role of the squared gradient norm. FI therefore serves as a measure of stationarity and can be small even when the law assigns incorrect weights to well-separated modes. Building on this perspective, subsequent work has clarified the statistical meaning and computational limits of FI guarantees \citep{chewi2023fisher,cheng2023fast,niles2026reweighted} and developed parallel, proximal, variance-reduced, and zeroth-order sampling methods \citep{zhou2026adaptive,chewi2026complexity,sahin2026variance,sahin2026zerothorder}. Standard direct SG-LMC FI bounds use global smoothness and uniform second moments; existing extensions do not jointly address BG-$p$ noise and gradient-dependent curvature (Appendix~\ref{app:related-sampling}).

Although analytically convenient, these assumptions fail even for elementary models. Global $L$-smoothness excludes superquadratic potentials, including quartic double-well potentials and Ginzburg--Landau models \citep{sun2023convergence,lytras2024tamed}. The uniform variance bound can fail even for quadratic potentials, as minibatching in Bayesian linear regression can produce gradient noise whose variance grows quadratically with distance from the posterior mode \citep{brosse2018promises}. Such state-dependent variance is ensured by the Blum--Gladyshev (BG-0) condition~\citep{blum1954approximation, gladyshev1965stochastic}. 

Beyond these elementary examples, Bayesian fine-tuning of large language models (LLMs) and vision-language models (VLMs) provides important practical applications of non-log-concave sampling, where a sampler generates multiple parameter vectors that serve as approximate posterior samples. Predictions are then averaged across these vectors to account for uncertainty in the fitted parameters~\citep{kim2023bayestune,yang2024bayesianlora,samplawski2026bayesian}. Low-rank adaptation (LoRA) makes such fine-tuning more efficient by reducing the number of trainable parameters~\citep{hu2022lora}. Studies of language modeling report gradient-dependent local curvature during training, motivating a relaxation of global smoothness~\citep{crawshaw2022robustness}. In addition, empirical studies in optimization report heavy-tailed gradient noise in BERT training and ALBERT and RoBERTa fine-tuning~\citep{zhang2020adaptive,battash2024revisiting,chezhegov2024clipping}. These findings raise a natural question: does similar heavy-tailed gradient noise arise during Bayesian fine-tuning of LLMs with LoRA? If so, what mechanisms can enable LMC to accommodate this noise while retaining non-asymptotic FI guarantees for non-log-concave targets whose potentials need not be globally smooth?

Our experiments on Bayesian LoRA fine-tuning of LLMs provide an affirmative empirical answer to the first question, showing evidence of heavy-tailed stochastic gradient noise and state-dependent noise growth consistent with BG-0. For the second question, non-convex optimization offers a natural starting point. Normalized gradient methods achieve stationarity guarantees under generalized smoothness~\citep{chen2023generalized}, while clipping, normalization, and momentum handle heavy-tailed noise~\citep{cutkosky2021high}. Anchoring and regularization provide guarantees under BG-0~\citep{fazla2026lower}. These optimization results motivate seeking analogous FI guarantees for non-log-concave sampling.

The difficulty is that modifying an optimization update to promote descent need not preserve the distribution sought by a sampler. Even with exact gradients, clipping or normalizing the Langevin drift while leaving the Brownian noise unchanged can alter the invariant distribution. A further difficulty is that existing FI guarantees depend on the expected squared error in the gradient estimates. Rare but very large errors can make this quantity infinite, even when the estimates are unbiased and have a finite moment of order $1\!<\!p\!<\!2$. The standard analysis then yields an infinite error bound, providing no useful convergence guarantee. Controlling these extreme errors without introducing excessive bias is therefore essential. Generalized smoothness creates a separate obstacle because discretization error depends on gradient values along the random path between updates. Bounding this error requires controlling higher-order gradient moments that are not covered by the standard analysis. Adaptive step sizes can limit these errors, but choosing them from noisy gradient estimates couples the estimation error with the duration over which it acts. Establishing FI convergence therefore requires controlling both errors and their interaction without relying on log-concavity, dissipativity, or a functional inequality.

\subsection{Contributions}


We address these challenges through robust gradient estimation and adaptive discretization. Geometric-median centering and residual clipping control mean-squared estimation error even when oracle variance is infinite, while adaptive steps limit discretization error where curvature grows. Robust adaptive LMC (RA-LMC) integrates these mechanisms to handle state-dependent heavy-tailed noise and generalized smoothness simultaneously. Its analysis controls the dependence between estimation error and step duration through a noise-aware step size. For all three methods, we establish non-asymptotic FI guarantees and explicit iteration and oracle complexities. Our contributions are summarized below.

\begin{itemize}[leftmargin=*,labelsep=*,nosep]
    \item \textbf{State-dependent heavy-tailed oracle.}
    We propose BG-$p$, which generalizes both the BG-0 condition and
    heavy-tailed noise models with a state-independent bound on the $p$-th moment. Under global smoothness, we establish
    FI guarantees for robust LMC (R-LMC), allowing
    both state dependence and infinite variance.

    \item \textbf{Sampling under generalized smoothness.}
    We establish FI guarantees for adaptive LMC (A-LMC) under symmetric
    generalized smoothness, assuming access to exact gradients. The
    resulting complexity bounds recover the standard globally smooth
    LMC rate as a special case.

    \item \textbf{Robust adaptive sampling.}
    We propose robust adaptive LMC (RA-LMC) and control the dependence
    between gradient estimation errors and adaptive step sizes to
    establish FI guarantees under generalized smoothness and BG-$p$
    noise.

    \item \textbf{Empirical evaluation.}
    Synthetic experiments verify FI convergence of all three methods
    and demonstrate the benefits of each mechanism and their combination.
    Bayesian LoRA fine-tuning experiments provide empirical evidence
    supporting BG-$p$ oracle. On LLM and VLM experiments, RA-LMC matches or improves upon
    leading baselines in predictive accuracy and uncertainty quantification. Code to reproduce our experiments is publicly available. 
\end{itemize}

\section{Preliminaries}\label{sec:preliminaries}

\textbf{Notation:} We work on a probability space $(\Omega,\mathcal{F},\mathbb{P})$. Expectations with respect to $\mathbb{P}$ are denoted by $\E$. We use $x,y,z,x_{\mathrm{ref}}\in\mathbb{R}^d$ for deterministic points and $X_t$ for random states. In particular, $x_{\mathrm{ref}}$ is a deterministic reference point chosen at initialization and used as the oracle anchor, whereas $X_0\sim\mu_0$ is the random initial state. When there is no ambiguity, we suppress brackets and write $\E\|X\|^2$ in place of $\E[\|X\|^2]$.

\begin{definition}
\label{def:target-information}
Let $V\in C^1(\mathbb{R}^d)$ with
$0<Z:=\int_{\mathbb{R}^d}e^{-V(x)}\,\mathrm{d}x<\infty$, and set
$\pi(x)=Z^{-1}e^{-V(x)}$. We also write
$\pi(\mathrm{d}x)=\pi(x)\,\mathrm{d}x$ for the target measure. For $\mu\ll\pi$ and
$f=\mathrm{d}\mu/\mathrm{d}\pi$, write $\mathrm{KL}(\mu\|\pi):=\int f\log f\,\mathrm{d}\pi$ and
\begin{equation}\label{eq:relative-fi}
\mathrm{FI}(\mu\|\pi):=4\int\|\nabla\sqrt f\|^2\,\mathrm{d}\pi.
\end{equation}
\end{definition}

Following prior LMC analyses~\citep{vempala2019rapid,balasubramanian22a,chewi2025analysis},
we assume $K_0:=\mathrm{KL}(\mu_0\|\pi)=\mathcal{O}(d)$ for iteration and
oracle complexity bounds, while retaining $K_0$ explicitly in finite-time
bounds and parameter choices.

\textbf{Langevin Sampling:} Conventionally, the \textit{Langevin diffusion}, the solution to the stochastic differential equation
\begin{equation}\label{eq:cont-diff}
    \mathrm{d}X_t = -\nabla V(X_t)\,\mathrm{d}t + \sqrt{2}\,\mathrm{d}W_t,
\end{equation}
has $\pi$ as an invariant distribution whenever the diffusion is well posed; under standard additional conditions, this invariant distribution is unique and the diffusion converges to it as $t\to\infty$. Here, $(W_t)_{t\geq 0}$ denotes a standard $d$-dimensional Brownian motion. Euler discretization with step size $h>0$ yields the standard \textit{Langevin Monte Carlo} (LMC) algorithm
\begin{equation*}
    X_{(k+1)h} := X_{kh} - h\nabla V(X_{kh}) + \sqrt{2h}Z_k,
\end{equation*}

where $Z_k\!\sim\!\gN(0,I_d)$. For the continuous-time interpolation specified in Section~\ref{sec:methods}, let $\mu_t:=\operatorname{Law}(X_t)$ and $\bar\mu_T:=T^{-1}\int_0^T\mu_t\,\mathrm{d}t$. To ensure the stability of discretization, previous works~\citep{balasubramanian22a, chewi2023fisher, sahin2026variance} adopt the following assumption.
\begin{assumption}[Global $L$-smoothness]\label{ass:global-smoothness}
    The potential $V$ is globally $L$-smooth; that is, its gradient is globally $L$-Lipschitz:
    $\|\nabla V(x_1)-\nabla V(x_2)\|\leq L\|x_1-x_2\|$ for all $x_1,x_2\in\mathbb{R}^d$ and some $L>0$.
\end{assumption}
Under this assumption,~\citet{balasubramanian22a} showed that $\mathcal{O}(L^2d^2/\varepsilon^2)$ iterations are sufficient to achieve $\mathrm{FI}(\bar\mu_T\|\pi)\leq\varepsilon$. When exact gradients are unavailable, SG-LMC averages $b$ independent
stochastic-oracle evaluations per update:
\begin{equation*}
X_{(k+1)h}=X_{kh}-\frac{h}{b}\sum_{j=1}^b\gG(X_{kh},\xi_{k,j})+\sqrt{2h}\,Z_k,
\end{equation*} 
where $Z_k\sim\mathcal N(0,I_d)$. We assume
$\E_\xi[\mathcal{G}(x,\xi)]=\nabla V(x)$ and
$\E_\xi\|\mathcal{G}(x,\xi)-\nabla V(x)\|^2\leq\sigma^2$. Choosing
$b=\mathcal{O}(\sigma^2/\varepsilon)$ in the FI bound of
\citet[Theorem~15]{balasubramanian22a} gives
$\mathcal{O}(L^2d^2\sigma^2/\varepsilon^3)$ oracle calls; the
step-by-step derivation is in Appendix~\ref{app:sg-lmc-comparison}.
The SL-VRLD method of~\citet{sahin2026variance} attains the same total
oracle order with constant expected calls per iteration under an
additional mean-square smoothness assumption on stochastic gradients.

\textbf{Symmetric Generalized Smoothness:} Motivated by empirical evidence of gradient-dependent local curvature in language modeling and ResNet-20 image classification~\citep{Zhang2020Why},~\citet{chen2023generalized} proposed the following symmetric generalization of global $L$-smoothness.
\begin{assumption}\label{ass:generalized-smoothness}
    For some $L_0,L_1\geq 0$ and $\gamma\in(0,1]$, the potential
    $V:\mathbb{R}^d\to\mathbb{R}$ satisfies
    \begin{equation*}
        \|\nabla V(y)-\nabla V(x)\|
        \leq
        \|y-x\|(
            L_0+
            L_1\sup_{z\in[x,y]}
            \|\nabla V(z)\|^\gamma)
    \vspace{-1em}
    \end{equation*}
    for all $x,y\in\mathbb{R}^d$.
\end{assumption}
The term ``symmetric'' refers to the use of the supremum gradient norm
along the segment joining $x$ and $y$, which treats the two endpoints
symmetrically. This allows local curvature to grow with the gradient magnitude and reduces to global $L_0$-smoothness when $L_1=0$. Compared with Assumption~\ref{ass:global-smoothness}, this assumption covers examples such as superquadratic potentials, polynomial double-well models, and generalized Gaussian distributions.


\textbf{BG-0 and Heavy-Tailed Oracles:} Uniformly bounded variance restricts both the state
dependence and tail behavior of gradient noise.
In Bayesian linear regression, a noise component
linear in the parameter can cause quadratically
growing variance~\citep{brosse2018promises}.
Such state-dependent bounds are studied
in~\citet{baker2019control,putcha2023preferential}.
Heavy-tailed gradient noise is reported in BERT
training~\citep{zhang2020adaptive},
CLIP~\citep{battash2024revisiting}, and ALBERT
and RoBERTa fine-tuning~\citep{chezhegov2024clipping}.
These observations motivate two relaxations:
\begin{equation}\label{ass:uniform-relaxation}
    \begin{aligned}
    \E_\xi\|\mathcal{G}(x,\xi)-\nabla V(x)\|^2
        &\leq B\|x-x_\mathrm{ref}\|^2+G,\\
    \E_\xi\|\mathcal{G}(x,\xi)-\nabla V(x)\|^p
        &\leq \sigma^p, \quad \forall x, x_{\mathrm{ref}}\in\sR^d,
    \end{aligned}
\end{equation}where $p\in(1,2]$, $B,G\geq 0$, and $\sigma>0$ are constants. The first inequality is the Blum--Gladyshev (BG-0) condition~\citep{blum1954approximation,gladyshev1965stochastic} and it permits the upper bound to grow quadratically with the distance from the reference point $x_{\mathrm{ref}}$. The second inequality imposes a state-independent finite $p$-th-moment bound and permits infinite variance when $1<p<2$.

\section{Methods}\label{sec:methods}

We develop our framework in three stages. First, under global $L$-smoothness, we propose a stochastic oracle that generalizes both BG-0 and heavy-tailed noise assumptions and construct a robust gradient estimator to control the resulting estimation error. Second, with exact-gradient access, we replace global $L$-smoothness by Assumption~\ref{ass:generalized-smoothness} and develop an adaptive step-size scheme to control discretization error. Finally, we combine these ideas in RA-LMC, which handles both relaxations simultaneously by addressing the coupling between stochastic estimation error and gradient-dependent local curvature.

We express all three methods through a common Euler interpolation.
At iteration $k$, let $F_k$ be the Langevin force (the negative gradient or its estimate) and $H_k>0$ the step size, both chosen before the next Brownian increment.
Set $\tau_0=0$ and $\tau_{k+1}=\tau_k+H_k$. For
$t\in[\tau_k,\tau_{k+1}]$, define
\begin{equation}\label{eq:common-lmc}
X_t:=X_{\tau_k}+(t-\tau_k)F_k
    +\sqrt{2}(W_t-W_{\tau_k}).
\end{equation}
Standard LMC has $F_k=-\nabla V(X_{\tau_k})$ and $H_k=h$. For $T>0$, let $N_T:=\inf\{n\in\mathbb{N}_0:\tau_n\geq T\}$ be the number of updates needed to reach $T$. If update $k$ uses $q_k$ gradient-oracle evaluations, let $Q_T:=\sum_{k=0}^{N_T-1}q_k$ be the total oracle count through time $T$. We specify the force and step size for each method below and establish non-asymptotic FI guarantees for all three settings.

\subsection{Robust LMC under State-Dependent Heavy-Tailed Noise}\label{sec:robust-lmc}
Building on the BG-0 and heavy-tailed noise models in~(\ref{ass:uniform-relaxation}), we propose a generalized BG-$p$ oracle that permits both state-dependent noise growth and potentially infinite variance. This formulation recovers both models as special cases and allows us to address their combined challenges in non-log-concave sampling. Our proposed oracle is given below.
\begin{assumption}[BG-$p$ oracle]\label{ass:BG-p}
    For $1<p\leq 2$ and $B,G\geq 0$, $\mathcal{G}(x,\xi)$ satisfies $\E_\xi[\mathcal{G}(x,\xi)]=\nabla V(x)$ and
    \begin{equation*}
        \begin{aligned}
        \E_\xi\|\mathcal{G}(x,\xi)-\nabla V(x)\|^p
            &\leq B\|x-x_{\mathrm{ref}}\|^p+G.
        \end{aligned}
    \end{equation*}
\end{assumption}
This assumption allows heavy-tailed stochastic gradient noise whose magnitude can increase as the sampler moves away from the reference point. The exponent $p$ specifies both the finite moment required of the noise and the power of its permitted growth with distance. Setting $p=2$ recovers the BG-0 condition, whereas $B=0$ recovers a state-independent finite-$p$-moment bound. For $1<p<2$ and $B>0$, the model can accommodate infinite variance and state-dependent noise simultaneously.

To handle a BG-$p$ oracle with potentially infinite variance, we adapt the two-stage centering and residual clipping approach of \citet{lugosi2019mean,whitehouse2024mean} to stochastic LMC. Specifically, we use the geometric median of three independent oracle evaluations as a center, clip the residual of a fourth around it, and establish the mean-squared error control needed for FI analysis. At iteration $k$ and state $x$, we draw independent oracle evaluations $\mathcal{G}_{k,i,j}:=\mathcal{G}(x,\xi_{k,i,j})$ and define
\begin{equation}\label{eq:geomed}
    \begin{aligned}
    C_{k,i} &:= \mathrm{GeoMed}(\mathcal{G}_{k,i,1},\mathcal{G}_{k,i,2},\mathcal{G}_{k,i,3}),\\
    \widehat{\mathcal{G}}_{k,i} &:= C_{k,i}+
        \mathrm{clip}_{\lambda_k}(\mathcal{G}_{k,i,4}-C_{k,i}),
    \end{aligned}
\end{equation}
where $\mathrm{clip}_{\lambda_k}(v):=v\min\{1,\lambda_k/\|v\|\}$ for $v\neq0$, and $\mathrm{clip}_{\lambda_k}(0):=0$. Here, $C_{k,i}$ is the \emph{geometric-median center}: a large deviation from the true gradient requires at least two of its three independent oracle evaluations to deviate substantially, yielding a finite second moment for $p>1$. We clip the residual rather than the gradient itself to control extreme noise without directly truncating the Langevin drift, which could otherwise steer the sampler toward a different target distribution. Averaging $b$ independent estimates from~(\ref{eq:geomed}), we define robust LMC (R-LMC) through~(\ref{eq:common-lmc}) by
\begin{equation}\label{sec3:geo-lmc}
    F_k=-\frac{1}{b}\sum_{i=1}^{b}\widehat{\mathcal{G}}_{k,i} \quad \text{and} \quad H_k=h.
\end{equation}
Thus $\tau_k=kh$, and each update uses $q_k=4b$ oracle evaluations. Building on the FI
framework of \citet{balasubramanian22a}, we establish our first convergence guarantee under the BG-$p$ oracle.

\begin{theorem}
\label{thm:fixed-batch}
Suppose Assumptions~\ref{ass:global-smoothness} and~\ref{ass:BG-p} hold, and $K_0<\infty$.
Let $\mu_t$ be the law of the interpolation~(\ref{eq:common-lmc})
with~(\ref{sec3:geo-lmc}).
For any $\delta>0$, integer $N\geq1$, and $0<h\leq(8L)^{-1}$,
there exist a batch size $b$ and clipping radii $\lambda_k$ such that,
with $T=Nh$, the time-averaged law $\bar\mu_T$ satisfies
\begin{equation*}
    \mathrm{FI}(\bar\mu_{T}\|\pi)
    \leq
    \frac{2K_0}{Nh}
    +\frac{32}{3}L^2dh
    +\frac{49}{12}\delta^2.
\end{equation*}
If $K_0=\mathcal O(d)$, then for sufficiently small $\varepsilon>0$
we can choose $N,h,b$, and $\lambda_k$ so that
$\mathrm{FI}(\bar\mu_T\|\pi)\leq\varepsilon$ with
$N_T=\mathcal O(L^2d^2/\varepsilon^2)$ in both cases below.
When $B>0$, the oracle complexity is
\begin{equation*}
    Q_T=\mathcal O\!\left(
        \frac{B^{1/(p-1)}L^{2+p/[2(p-1)]}
              d^{2+3p/[2(p-1)]}}
             {\varepsilon^{2+3p/[2(p-1)]}}
        \right).
\end{equation*}
When $B=0$ and $G>0$, it improves to
\begin{equation*}
    Q_T=\mathcal O\!\left(
        \frac{L^2d^2G^{1/(p-1)}}
             {\varepsilon^{2+p/[2(p-1)]}}
        \right).
\end{equation*}
\end{theorem}
For fixed $B,G>0$, the state-dependent term determines the leading
high-accuracy oracle rate; $G$ remains in the finite bound.
The full statement and proof are provided in
Theorem~\ref{thm:full-robust} and Appendix~\ref{app:proof-robust}.

\begin{remark}
For $p=2$ and $B=0$, R-LMC matches the
$\mathcal O(L^2d^2G/\varepsilon^3)$ oracle complexity of
SG-LMC under a uniform variance bound $G\!>\!0$ (Lemma~\ref{lem:sg-lmc-fi}).
For $B\!=\!G\!=\!0$, it recovers LMC's
$\mathcal O(L^2d^2/\varepsilon^2)$ iteration complexity
\citep{balasubramanian22a}.
\end{remark}

\subsection{Adaptive LMC under Generalized Smoothness}\label{sec:gen-smooth}
We next develop an adaptive step-size mechanism for controlling the LMC discretization error when global $L$-smoothness is replaced by~Assumption~\ref{ass:generalized-smoothness}. Under global $L$-smoothness, the gradient varies at a uniform rate throughout the state space, allowing a single step size to control the discretization error along the entire trajectory. Under generalized smoothness, however, the gradient can change more rapidly as its norm increases. A fixed step size may therefore be safe where the gradient changes slowly but too aggressive where it changes rapidly, causing large discretization errors or instability. This motivates using shorter steps in regions of large gradients while allowing longer steps in more regular regions. We define adaptive LMC (A-LMC) through~(\ref{eq:common-lmc}) by
\begin{equation}\label{eq:adaptive-stepsize}
F_k=-\nabla V(X_{\tau_k}) \quad \text{and} \quad H_k=\frac{h}{1+\omega\|F_k\|^\gamma},
\end{equation}
where $\omega\geq 0$. Here, $\gamma$ determines how rapidly the step size decreases with the gradient norm. Matching this exponent to Assumption~\ref{ass:generalized-smoothness} offsets the corresponding growth in local curvature, while $\omega$ controls the overall strength of the adaptation. Because $\tau_{k+1}-\tau_k=H_k$ records elapsed time along the interpolation, the $h\to0$ limit of A-LMC is the original Langevin diffusion~(\ref{eq:cont-diff}), which preserves $\pi$. Indexing the same updates by $kh$ instead yields a time-rescaled diffusion with a different invariant law~\citep{leroy2024adaptive}, so our guarantee concerns the time-averaged law $\bar\mu_T$.

Extending the FI framework of~\citet{balasubramanian22a} to A-LMC
requires controlling how the gradient changes during each adaptive step.
Under global $L$-smoothness, this change is bounded by $L$ times the
displacement. Under Assumption~\ref{ass:generalized-smoothness}, the bound
depends on gradient values along the random interpolation path, producing
higher-order terms that displacement alone cannot control. We control
these terms with estimates of gradient changes along each adaptive
interval and a bound on the required gradient moments in terms of FI.
Combined with the adaptive step size, these estimates absorb the
state-dependent terms into FI and leave an explicit discretization
remainder. This yields the following theorem.

\begin{theorem}
\label{thm:adaptive-lmc}
Suppose Assumption~\ref{ass:generalized-smoothness} holds with
$0<\gamma\leq 1$ and
$K_0<\infty$.
Set $\omega=L_1/(1+L_0)$ and let $\mu_t$ be the law of the
interpolation~(\ref{eq:common-lmc}) with~(\ref{eq:adaptive-stepsize}).
For any horizon $T>0$ and sufficiently small $h>0$,
$N_T<\infty$ almost surely, and $\bar\mu_T$ satisfies
\begin{equation*}
    \mathrm{FI}(\bar\mu_T\|\pi)
    \leq
    \frac{4K_0}{3T}
    +\frac{8}{3}
    R_\gamma\!\left(
        h,1+C_\gamma+\frac{K_0}{T}
    \right),
\end{equation*}
where $C_\gamma$ and the discretization remainder $R_\gamma$ are defined
in Theorem~\ref{thm:full-adaptive}. If $K_0=\mathcal O(d)$, then for
sufficiently small $\varepsilon>0$ we can choose $T$ and $h$ so that
$\mathrm{FI}(\bar\mu_T\|\pi)\leq\varepsilon$ with $N_T=Q_T$.
Let $\Lambda:=\max\{1,L_0,L_1\}$. When $L_1>0$ and $0<\gamma\leq1$,
the expected oracle complexity is
\begin{equation*}
    \E Q_T=\mathcal{O}\!\left(
        \frac{
            \Lambda^{4/(2-\gamma)}
            d^{4/(2-\gamma)}
        }{\varepsilon^2}
    \right).
\end{equation*}
When $L_1=0$, the oracle count is deterministic and satisfies
\begin{equation*}
    Q_T=\mathcal{O}\!\left(
        \frac{L_0^2d^2}{\varepsilon^2}
    \right).
\end{equation*}
\end{theorem}
The complete statement and proof are provided in 
Theorem~\ref{thm:full-adaptive} and Appendix~\ref{app:proof-adaptive}.

\begin{remark}
When $L_1=0$, the adaptive step size becomes constant, and our
complexity bound recovers the rate of standard LMC under global
$L_0$-smoothness.
\end{remark}

\subsection{Robust Adaptive LMC (RA-LMC)}\label{sec:robust-gen-lmc}

For a non-log-concave target satisfying Assumption~\ref{ass:generalized-smoothness} and a stochastic oracle satisfying Assumption~\ref{ass:BG-p}, a natural approach is to combine the robust estimator~(\ref{eq:geomed}) with the adaptive step size~(\ref{eq:adaptive-stepsize}). This combination, however, is not straightforward. The adaptive step-size rule in~(\ref{eq:adaptive-stepsize}) assumes access to the exact gradient. Replacing it with the robust estimate makes the estimation error and step size dependent, so the same inaccurate estimate can affect both the drift and the duration over which its error is applied. Consequently, the conditional mean-squared error bound for~(\ref{eq:geomed}) does not directly control the error accumulated over elapsed time, and the two preceding analyses cannot be combined through a simple substitution.

\begin{figure*}[t]
  \centering
  \includegraphics[width=\textwidth]{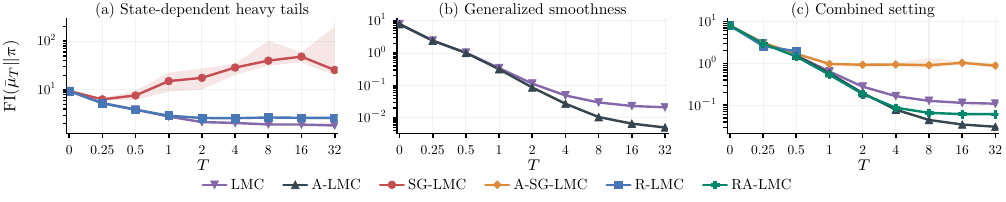}
  \vspace{-1em}
  \caption{Robust estimation and adaptive steps improve FI convergence under heavy-tailed noise and generalized smoothness, respectively, while RA-LMC improves convergence when both challenges arise by integrating robust estimation with a noise-aware step size. SG, R, and A denote stochastic gradient averaging, proposed robust estimator and adaptive steps, respectively. Panels (a)--(c) use a Gaussian mixture with a BG-$p$ oracle of infinite variance, a radial quartic target with exact gradients, and a radial cubic target with the same oracle, respectively. Each method’s curve shows the median across 10 seeds, with its shaded region indicating the interquartile range.}
  \vspace{-1.5em}
  \label{fig:synthetic-fi}
\end{figure*}

\begin{figure}[t]
  \centering
  \includegraphics[width=\columnwidth]{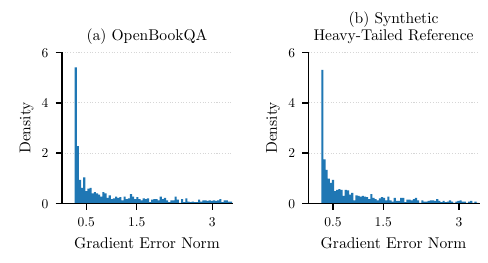}
  \vspace{-1em}
  \caption{Panel (a) shows minibatch gradient error norms measured at step 2,500
    of Qwen3-4B LoRA fine-tuning on OpenBookQA. Panel (b) shows synthetic
    norms generated from a one-sided L\'evy-stable reference construction with an infinite second moment. The similarity between the histograms suggests heavy-tailed behavior in the stochastic gradient errors.}
  \vspace{-1em}
  \label{fig:llm-noise-histograms}
\end{figure}

\begin{figure}[t]
  \centering
  \includegraphics[width=\columnwidth]{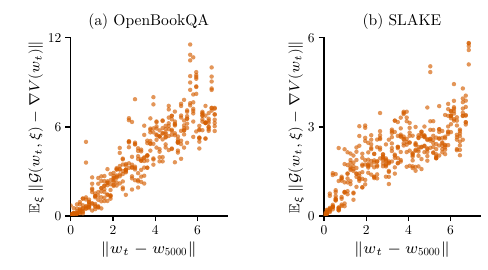}
  \vspace{-2em}
  \caption{Panels (a) and (b) show the empirical first moment of minibatch
  gradient-error norms versus distance from the final LoRA iterate during
  fine-tuning of Qwen3-4B on OpenBookQA and Qwen3-VL-4B-Instruct on SLAKE.
  Each point averages error norms from 32 independent minibatches of size
  four. The increase in the first moment with distance motivates the
  state-dependent term in the BG-$p$ oracle.}
  \vspace{-1em}
  \label{fig:llm-noise-first-moment}
\end{figure}

To address this dependence, we retain the robust estimator
in~(\ref{eq:geomed}) and define the local $L^p$-noise scale at state $x$ by
\begin{equation*}
    s_p(x):=\bigl[B\|x-x_\mathrm{ref}\|^p+G\bigr]^{1/p},
\end{equation*}
which bounds the conditional $L^p$ norm of the oracle error.
Using both this scale and the robust gradient estimate, we define robust adaptive LMC (RA-LMC) through~(\ref{eq:common-lmc}) by
\begin{equation}\label{eq:robust-lmc-adaptive-iterates}
F_k=-\frac{1}{b}\sum_{i=1}^{b}\widehat{\mathcal{G}}_{k,i}, \quad H_k=\frac{h}{1+\omega\|F_k\|^\gamma+\kappa s_p(X_{\tau_k})^2},
\end{equation}
where $\kappa,\omega\geq 0$. The term $\|F_k\|^\gamma$ adapts the step size to the gradient-dependent curvature allowed by generalized smoothness, while $s_p(X_{\tau_k})^2$ further shortens the step in regions of high oracle noise and controls the coupling between the estimation error and its duration. Building on Theorem~\ref{thm:adaptive-lmc}, we control the interaction between gradient estimation and discretization errors under generalized smoothness. The following theorem gives a non-asymptotic FI guarantee for the time-averaged law and the corresponding iteration and stochastic-oracle complexities.
\begin{theorem}
\label{thm:robust-adaptive-lmc}
Suppose Assumptions~\ref{ass:generalized-smoothness}
and~\ref{ass:BG-p} hold with $0<\gamma\leq 1$ and $1<p\leq 2$, and
$K_0<\infty$. If $B>0$, assume additionally that
$\E\norm{X_0-x_{\mathrm{ref}}}^2\leq M_{2,0}<\infty$.
Set $\omega=L_1/(1+L_0)$ and let $\mu_t$ be the law of the
interpolation~(\ref{eq:common-lmc})
with~(\ref{eq:robust-lmc-adaptive-iterates}).
For sufficiently small $\varepsilon>0$, we can choose
$T,h,b,\kappa$, and $\lambda_k$ so that $N_T<\infty$ almost surely
and $\bar\mu_T$ satisfies
$\mathrm{FI}(\bar\mu_T\|\pi)\leq\varepsilon$ with $Q_T=4bN_T$.
Let $\Lambda:=\max\{1,L_0,L_1\}$.
If $K_0=\mathcal O(d)$ and $L_1>0$, then the expected iteration complexity is
\begin{equation}
    \E N_T=\mathcal{O}\left(\frac{\Lambda^{(8-2\gamma)/(2-\gamma)}
d^{(6-\gamma)/(2-\gamma)}}{\varepsilon^2}\right).
\end{equation}
When $B>0$, the expected oracle complexity is
\begin{align*}
    \E Q_T
    &=\mathcal{O}\!\left(
        \frac{B^{1/(p-1)}
        \Lambda^{\frac{(8-2\gamma)(p-1)+2p}
                         {(p-1)(2-\gamma)}}
        d^{\frac{(3p-2)(6-\gamma)}
                {2(p-1)(2-\gamma)}}}{\varepsilon^{2+3p/[2(p-1)]}}
    \right).
\end{align*}
When $B=0$, the expected oracle complexity is
\begin{align*}
    \E Q_T
    &\!=\!\mathcal{O}\!\left(\!
        \frac{(1+G)^{1/(p-1)}\Lambda^{\frac{(8-2\gamma)(p-1)+p}
                            {(p-1)(2-\gamma)}}
        \!d^{1+\frac{8(p-1)+\gamma p}
                    {2(p-1)(2-\gamma)}}}{\varepsilon^{2+p/[2(p-1)]}}\!
    \right).
\end{align*}
\end{theorem}
The full statement, including parameter choices
and the $\mathrm{FI}$ bound, is given in
Theorem~\ref{thm:full-combined}, with complexity
specializations in
Corollary~\ref{cor:combined-special-cases}
and proofs in Appendix~\ref{app:proof-combined}.

\begin{remark}
For $p=2$, the oracle bounds reduce to
$\mathcal O(\varepsilon^{-3})$ under uniformly bounded variance
($B=0$) and $\mathcal O(\varepsilon^{-5})$ under BG-0 oracle
($B>0$), with other problem parameters fixed, as stated in
Corollary~\ref{cor:combined-special-cases}. The BG-0 oracle bound
has the same $\varepsilon^{-5}$ dependence as R-LMC under global smoothness
(Theorem~\ref{thm:fixed-batch}). When $L_1>0$, generalized smoothness
changes the dimension and smoothness factors.
When $L_1=\kappa=0$, RA-LMC reduces to R-LMC, and the parameter
choices of Theorem~\ref{thm:fixed-batch} recover its fixed-step
guarantees. For exact gradients ($B=G=0$), we bypass the estimator,
use one gradient evaluation per update, and adopt the schedule of
Theorem~\ref{thm:adaptive-lmc} to recover A-LMC. If $L_1=0$ as well,
this gives the standard LMC
complexity $\mathcal O(L_0^2d^2/\varepsilon^2)$. The parameter choices
and derivations for these cases are given in
Appendix~\ref{app:full-theorems}.
\end{remark}

\section{Experiments}

\subsection{Synthetic Experiments}

We conduct synthetic experiments to validate the FI convergence of our three methods and assess the benefit of each mechanism. We consider three non-log-concave sampling problems in $\sR^2$, testing R-LMC under a state-dependent heavy-tailed oracle, A-LMC under generalized smoothness with exact gradients, and RA-LMC under both relaxations. Figure~\ref{fig:synthetic-fi} reports estimated $\mathrm{FI}(\Bar{\mu}_T\|\pi)$ versus time $T$, showing medians across ten random seeds with interquartile bands. Appendix~\ref{app:synthetic-setting} provides the experimental settings and FI estimation details.


Figure~\ref{fig:synthetic-fi}(a) shows results for R-LMC on a Gaussian-mixture target with a globally smooth potential and a state-dependent Student oracle having a finite $p$-th moment but infinite variance. We compare R-LMC, SG-LMC, and exact-gradient LMC using the same fixed step size. Both stochastic methods use 16 oracle queries per update. R-LMC closely follows the FI decrease of exact-gradient LMC, whereas SG-LMC diverges, illustrating the benefit of robust estimation.

Figure~\ref{fig:synthetic-fi}(b) shows results for A-LMC on a radial quartic potential satisfying Assumption~\ref{ass:generalized-smoothness} with $\gamma=2/3$. Using exact gradients, we compare fixed-step LMC with A-LMC. Both methods initially reduce FI, but A-LMC reaches lower FI over the observed horizon, supporting the use of adaptive steps to control discretization error under unbounded curvature.

Figure~\ref{fig:synthetic-fi}(c) shows results for RA-LMC on a radial cubic potential satisfying Assumption~\ref{ass:generalized-smoothness} with $\gamma=1/2$. We compare six methods, using the same oracle as in panel (a) and 16 queries per update for each stochastic method. Fixed-step SG-LMC and R-LMC encounter numerical failures, while A-SG-LMC plateaus at higher FI. RA-LMC closely tracks exact-gradient A-LMC and has lower estimated FI than A-SG-LMC and fixed-step LMC. These results highlight the benefit of robust estimation with adaptive steps when state-dependent heavy-tailed noise and growing curvature occur together.

\begin{table*}[t]
  \centering
  \vspace{-1em}
  \caption{OpenBookQA and SLAKE test results (mean $\pm$ std. dev.).
    Accuracy (ACC, in percent) is higher-is-better; negative log-likelihood
    (NLL), Brier score, and expected calibration error (ECE) are lower-is-better. \textbf{Bold:} the best. \underline{Underline:} the second best.}
  \label{tab:obqa-slake-sensitivity}
  \begingroup
  \footnotesize
  \setlength{\tabcolsep}{1pt}
  \renewcommand{\arraystretch}{1.0}
  \newcommand{\msd}[2]{\mbox{$#1${\scriptsize$\pm #2$}}}
  \begin{tabular*}{\textwidth}{@{\extracolsep{\fill}}lcccc@{\hspace{2pt}}cccc@{}}
    \toprule
    & \multicolumn{4}{c}{\textbf{OpenBookQA}} & \multicolumn{4}{c}{\textbf{SLAKE}} \\
    \cmidrule(lr){2-5}\cmidrule(lr){6-9}
    Method & ACC $\uparrow$ & NLL $\downarrow$ & Brier $\downarrow$ & ECE $\downarrow$
           & ACC $\uparrow$ & NLL $\downarrow$ & Brier $\downarrow$ & ECE $\downarrow$ \\
    \midrule
    MLE & \msd{91.400}{0.849} & \msd{0.617}{0.050} & \msd{0.151}{0.012} & \msd{0.070}{0.007}
        & \msd{89.155}{0.934} & \msd{1.070}{0.286} & \msd{0.203}{0.017} & \msd{0.103}{0.010} \\
    MAP & \msd{91.500}{0.757} & \msd{0.583}{0.051} & \msd{0.147}{0.014} & \msd{0.069}{0.011}
        & \msd{89.366}{0.423} & \msd{0.954}{0.247} & \msd{0.198}{0.012} & \msd{0.099}{0.006} \\
    TempScale & \msd{91.400}{0.849} & \msd{0.293}{0.003} & \msd{0.138}{0.003} & \msd{0.031}{0.010}
              & \msd{89.155}{0.934} & \msd{0.278}{0.018} & \msd{0.167}{0.011} & \msd{0.036}{0.007} \\
    MC Dropout & \msd{91.250}{1.182} & \msd{0.545}{0.044} & \msd{0.151}{0.013} & \msd{0.066}{0.008}
               & \msd{89.366}{0.841} & \msd{0.832}{0.177} & \msd{0.197}{0.017} & \msd{0.098}{0.009} \\
    Deep Ensemble & \msd{91.600}{0.163} & \msd{0.459}{0.053} & \msd{0.141}{0.002} & \msd{0.058}{0.003}
                  & \msd{\underline{90.000}}{0.586} & \msd{0.660}{0.046} & \msd{0.177}{0.007} & \msd{0.084}{0.006} \\
    Laplace & \msd{\underline{91.900}}{0.600} & \msd{0.317}{0.010} & \msd{0.158}{0.005} & \msd{0.093}{0.012}
            & \msd{89.366}{0.809} & \msd{0.283}{0.011} & \msd{0.172}{0.007} & \msd{0.066}{0.004} \\
    BLoB & \msd{90.800}{0.566} & \msd{0.299}{0.002} & \msd{\underline{0.135}}{0.004} & \msd{\bm{0.028}}{0.002}
         & \msd{88.873}{0.598} & \msd{\bm{0.242}}{0.003} & \msd{\underline{0.152}}{0.002} & \msd{\bm{0.021}}{0.009} \\
    ScalaBL & \msd{90.850}{0.806} & \msd{\underline{0.285}}{0.006} & \msd{0.139}{0.002} & \msd{0.035}{0.006}
            & \msd{89.225}{0.931} & \msd{\underline{0.254}}{0.005} & \msd{0.160}{0.005} & \msd{\underline{0.035}}{0.004} \\
    TFB & \msd{91.050}{0.870} & \msd{0.441}{0.027} & \msd{0.142}{0.009} & \msd{0.049}{0.007}
        & \msd{89.155}{1.161} & \msd{0.701}{0.195} & \msd{0.191}{0.014} & \msd{0.090}{0.009} \\
    SG-LMC & \msd{86.650}{0.100} & \msd{0.370}{0.003} & \msd{0.181}{0.001} & \msd{0.045}{0.006}
           & \msd{89.859}{0.609} & \msd{0.289}{0.004} & \msd{0.159}{0.002} & \msd{0.059}{0.008} \\
    RA-LMC (ours) & \msd{\bm{92.050}}{0.300} & \msd{\bm{0.264}}{0.005} & \msd{\bm{0.125}}{0.001} & \msd{\underline{0.030}}{0.004}
                  & \msd{\bm{90.634}}{0.533} & \msd{\underline{0.254}}{0.008} & \msd{\bm{0.151}}{0.003} & \msd{0.051}{0.001} \\
    \bottomrule
  \end{tabular*}
  \endgroup
\end{table*}

\subsection{Bayesian Fine-Tuning Experiments}

We conduct LLM and VLM fine-tuning experiments to assess the empirical relevance of the BG-$p$ oracle model and
the practical benefits of combining robust gradient estimation with adaptive updates. The step-size rule
analyzed above requires smoothness and local noise-scale quantities that are difficult to specify for large models. In practice, we retain geometric-median centering and residual clipping, use AdamW~\citep{loshchilov2019decoupled} for adaptive
updates, and add Gaussian noise after each update. We consider the following two datasets.

\textbf{OpenBookQA.} This dataset provides a benchmark for reasoning in language models. Its questions require applying elementary science facts and commonsense knowledge to everyday situations. Each question has four answer choices, with one correct answer~\citep{mihaylov2018can}. It contains 4,957 training, 500 validation, and 500 test examples.

\textbf{SLAKE.} This dataset provides a benchmark for medical visual question answering in vision-language models. It pairs X-ray, CT, and MRI images with questions about anatomical structures and abnormalities~\citep{liu2021slake}. We use its yes/no subset, treating each image--question pair as a binary classification example. This subset contains 1,681 training, 358 validation, and 355 test examples.

\subsubsection{Empirical Evidence for BG-$p$ Oracle}

We examine the empirical relevance of the BG-$p$ oracle by studying minibatch-gradient errors during
Bayesian fine-tuning. We consider LMC with minibatch-averaged gradients and AdamW-based adaptive
updates rather than fixed-step LMC. This allows us to investigate whether heavy-tailed, state-dependent
errors occur when updates adapt to changing gradient scales, as motivated by generalized smoothness.
Their presence would motivate our robust estimator as a complement to step-size adaptation.

We fine-tune Qwen3-4B and Qwen3-VL-4B-Instruct on OpenBookQA and SLAKE, respectively, for 5,000 steps
with minibatch size four. Following~\citet{zhang2020adaptive}, we measure the Euclidean
norm of the difference between each minibatch gradient and the average gradient over the full
training set. Figure~\ref{fig:llm-noise-histograms}(a) shows the histogram of these error norms for
OpenBookQA at step 2,500, while Figure~\ref{fig:llm-noise-histograms}(b) shows synthetic norms generated
from a one-sided L\'evy-stable reference construction. The similarity between the
histograms suggests heavy-tailed behavior in stochastic gradient errors during LoRA fine-tuning with LMC. Appendix~\ref{app:oracle-tail-diagnostic} provides the details and corresponding
results at additional checkpoints for both datasets.

To examine the state-dependent component of the BG-$p$ oracle, we select 50 checkpoints from each
fine-tuning trajectory, uniformly spaced in distance from the final LoRA iterate.
At each checkpoint, we average the error norms over 32 independent minibatches and repeat this
calculation 8 times. Figure~\ref{fig:llm-noise-first-moment} plots these
first-moment estimates against the Euclidean distance from the final iterate. The observed
increase in error magnitude with distance supports modeling the noise scale as state dependent. Appendix~\ref{app:oracle-first-moment} provides the sampling and estimation details.

\subsubsection{Comparison with Baselines}
We compare RA-LMC against SG-LMC and nine baselines from Bayesian Adaptation
Gym~\citep{samplawski2026bayesian}. We use the benchmark's reference implementations for these
nine methods and repeat each fine-tuning experiment with four random seeds. For each seed, we average
class probabilities from ten predictions per test example and compute accuracy, negative
log-likelihood (NLL), Brier score, and expected calibration error (ECE). Table~\ref{tab:obqa-slake-sensitivity} reports the mean and standard deviation across seeds.

RA-LMC outperforms SG-LMC on all four metrics for both datasets. Compared with the benchmark
baselines, it has the highest accuracy and lowest Brier score on both datasets, as well as the
lowest NLL on OpenBookQA. It remains competitive on the other metrics. These results support the practical benefits of
combining our robust estimator with adaptive updates for Bayesian LoRA fine-tuning.

\section{Conclusion}

We extended the FI framework or non-log-concave Langevin sampling beyond globally Lipschitz gradients and uniformly bounded noise variance. The proposed BG-$p$ oracle and our three methods support non-asymptotic guarantees for time-averaged laws, with explicit iteration and oracle complexities and without assuming any dissipativity or functional inqualities. RA-LMC controls the interaction between robust estimation and adaptive step sizes, enabling guarantees under both relaxations. Synthetic experiments support these mechanisms. In Bayesian LoRA fine-tuning, we found evidence for BG-$p$ oracle and showed that RA-LMC, using our robust
estimator and adaptive step sizes, matches or outperforms leading baselines in accuracy and uncertainty quantification.

\FloatBarrier
\subsection*{AI Use Statement}
In this work, we used generative AI tools for assist in writing of proofs, implement methods, support qualitative and thematic data analysis, and formulate mathematical claims. We have not used generative AI tools for generate synthetic datasets, help develop theoretical models or conceptual frameworks, provide critical ingredients for proving mathematical claims, propose or refine hypothesis, design or provide feedback on research methodology or experiments, interpret results, and the remaining required-disclosure tasks are not applicable to this work.
Additionally, we used generative AI tools for create or edit software code, summarize or analyze existing literature, discover research topics or identify gaps, and edit a research paper to improve readability. We have reviewed all AI-assisted work and tested AI-generated experiment code to run as intended, and all conversations with AI are verified before being applied to the paper. We take responsibility for the final content of this work, including text, claims, or artifacts produced with the aid of generative AI.

\acknowledgments{This work was partially supported by National Insititutes of Health grant NIH/NHLBI R01-HL153430 (PI: Sharif), and supported in part by NSF DMS-2502560 and CNS-2313109. }
\subsection*{Acknowledgements}
This work was partially supported by National Insititutes of Health grant NIH/NHLBI R01-HL153430 (PI: Sharif), and supported in part by NSF DMS-2502560 and CNS-2313109.

\bibliographystyle{apalike}
\bibliography{references}

\clearpage
\appendix
\counterwithin{figure}{section}
\counterwithin{table}{section}
\thispagestyle{empty}
\onecolumn
\manuscripttitle{Robust Langevin Sampling under State-Dependent Heavy-Tailed Noise with Applications to LLM Fine-Tuning: Supplementary Materials}

\enableappendixcontents
\printappendixcontents
\clearpage

\section{Full Theorem Statements and Proofs}
\label{app:full-theorems}

\subsection{Information and interpolation conventions}
\label{app:prelim-conventions}

\paragraph{Information functionals.}
As in Definition~\ref{def:target-information}, we use $\pi(x)$ for the
target density and $\pi(\mathrm{d}x)=\pi(x)\,\mathrm{d}x$ for its measure. If
$\mu\ll\pi$ and $f:=\mathrm{d}\mu/\mathrm{d}\pi$, the KL convention is
$0\log 0:=0$; if $\mu\not\ll\pi$, we set
$\mathrm{KL}(\mu\|\pi):=+\infty$.  The Fisher information in
(\ref{eq:relative-fi}) uses the weak gradient of $\sqrt f$:
$\mathrm{FI}(\mu\|\pi):=+\infty$ unless $\mu\ll\pi$,
$\sqrt f\in W^{1,2}_{\mathrm{loc}}(\mathbb{R}^d)$, and
$\int_{\mathbb{R}^d}\|\nabla\sqrt f\|^2\,\mathrm{d}\pi<\infty$.
Here the weak gradient is distributional with respect to Lebesgue measure.
If $\mu$ has a smooth positive Lebesgue density $\rho$, then equivalently
\begin{equation*}
    \mathrm{FI}(\mu\|\pi)
    =\int_{\mathbb{R}^d}
      \|\nabla\log\rho(x)+\nabla V(x)\|^2\rho(x)\,\mathrm{d}x.
\end{equation*}

\begin{definition}[Interpolation and algorithmic laws]
\label{def:algorithmic-laws}
Let $\tau_0=0$. At event $k$, let a step size $H_k>0$ and a Langevin force
$F_k$ be selected using the current and past states and
any oracle samples drawn at that event, before the Brownian increment on the
next interval is generated, and set $\tau_{k+1}:=\tau_k+H_k$.
We use the continuous-time interpolation
\begin{equation*}
    X_t:=X_{\tau_k}+(t-\tau_k)F_k
    +\sqrt{2}\bigl(W_t-W_{\tau_k}\bigr),
\end{equation*}
for $t\in[\tau_k,\tau_{k+1})$.
Let $\tau_\infty:=\lim_{k\to\infty}\tau_k$. The interpolation is defined
pathwise on $[0,\tau_\infty)$; whenever
$\mathbb{P}(\tau_\infty>t)=1$, let $\mu_t:=\operatorname{Law}(X_t)$, where the
law includes both Brownian and stochastic oracle randomness. For $T>0$,
define
\begin{equation*}
\begin{aligned}
    N_T&:=\inf\{n\in\mathbb{N}_0:\tau_n\geq T\},\\
    \bar\mu_T(A)&:=\frac{1}{T}\int_0^T\mu_t(A)\,\mathrm{d}t,
\end{aligned}
\end{equation*}
where $\mathbb{N}_0:=\{0,1,\ldots\}$ and $\inf\emptyset:=\infty$, and
where the second identity holds for every Borel set $A$ whenever
$N_T<\infty$ almost surely. Thus,
$\bar\mu_T$ is a deterministic mixture law. Equivalently, if
$U\sim\mathrm{Unif}([0,T])$ is independent of the process, then
$X_U\sim\bar\mu_T$. If event $k$ uses $q_k$ individual gradient oracle
evaluations, its oracle count through time $T$ is
\begin{equation*}
    Q_T:=\sum_{k=0}^{N_T-1}q_k.
\end{equation*}
For a fixed step size and a horizon $T=Nh$,
we have $\tau_k=kh$ and $N_T=N$.
\end{definition}

\paragraph{Sampling convention.}
Throughout the stochastic methods, conditional on the past and current
state, all oracle evaluations drawn at an event are fresh and mutually
independent, and they are independent of the Brownian increment on the
following interval.

\subsection{Notation, constants, and deterministic schedules}
\label{app:notation-constants}

This appendix gives the complete statements and proofs of the three main
results. The definitions in this subsection make the theorem statements
explicit. Each proof recalls the relevant formulas where they enter the
argument, together with a brief explanation of their role.
For real numbers $a,b$, we write $a\wedge b:=\min\{a,b\}$ and
$a\vee b:=\max\{a,b\}$.
All laws, event times, interpolations, time-averaged laws, and
oracle counts follow Definition~\ref{def:algorithmic-laws}.  We write
\begin{equation*}
    \overline{\mathrm{FI}}_T
    :=\frac{1}{T}\int_0^T\mathrm{FI}(\mu_t\|\pi)\dd t.
\end{equation*}
Convexity of relative Fisher information gives
$\mathrm{FI}(\bar\mu_T\|\pi)\leq\overline{\mathrm{FI}}_T$.
Thus $\overline{\mathrm{FI}}_T$ averages the information of the individual
laws, whereas $\mathrm{FI}(\bar\mu_T\|\pi)$ is the information of their
time average; these quantities need not coincide.
The initialization quantity common to all three results is
$K_0:=\mathrm{KL}(\mu_0\|\pi)$. For the stochastic methods with $B>0$,
the proofs additionally require a known finite bound
$\E\norm{X_0-x_{\mathrm{ref}}}^2\leq M_{2,0}$. When $B=0$, no second moment bound on the initial state is needed, and we set
$M_{2,0}:=0$. Finite-time bounds retain the initialization quantities explicitly. To report iteration and oracle complexities, we adopt the convention $K_0=\mathcal O(d)$ and show $M_{2,0}$ explicitly when $B>0$.

For the BG-$p$ oracle, we define its local $L^p$-noise scale as
\begin{equation*}
    s_p(x):=\bigl[B\norm{x-x_{\mathrm{ref}}}^{p}+G\bigr]^{1/p}.
\end{equation*}
The robust estimator error bound uses the constant
\begin{equation}
 C_p:=2^{2/p-1}\Bigg[2^{2p-2}\bigl(1+2\sqrt3\,4^p\bigr)^2+\frac{32p\,3^{1/p}}{p-1}+2^p\bigl(1+2\sqrt3\,4^p\bigr)\Bigg].
\end{equation}
This constant depends only on $p$ and is obtained in
Lemma~\ref{app:robust-estimation} by bounding the squared bias and variance
of the robust estimator.
Throughout, $\GeoMed$ denotes a fixed measurable, translation-equivariant
selection from the minimizers of the sum of Euclidean distances.  The
proof uses only its defining minimization property.
Writing $u=\norm{x-x_{\mathrm{ref}}}^2$, the square of the local
$L^p$ noise scale becomes
\begin{equation*}
    s_p(x)^2=\bigl[Bu^{p/2}+G\bigr]^{2/p}.
\end{equation*}
As a function of the squared distance $u\geq0$, this expression is
nonnegative, nondecreasing, and concave for $1<p\leq2$. Indeed, away
from its zero set, its second
derivative equals
\begin{equation*}
 B\left(\frac p2-1\right)G
 u^{p/2-2}(Bu^{p/2}+G)^{2/p-2}\leq0.
\end{equation*}
Continuity covers the remaining points. Consequently, Jensen's inequality
gives
\begin{equation*}
 \E s_p(X)^2
 \leq\left[B\bigl(\E\norm{X-x_{\mathrm{ref}}}^2\bigr)^{p/2}+G\right]^{2/p}.
\end{equation*}
Thus, a second-moment bound on the trajectory controls its average squared
local $L^p$ noise scale in Theorems~\ref{thm:full-robust}
and~\ref{thm:full-combined}.

We next collect the generalized smoothness coefficients shared by the last
two results.  Set
\begin{equation}
 \omega:=\frac{L_1}{1+L_0},\quad \text{and} \quad C_\gamma:=4dL_0+(2-\gamma)\gamma^{\gamma/(2-\gamma)}
                  (2dL_1)^{2/(2-\gamma)}.
 \label{app:eq:common-adaptive-constants}
\end{equation}
The constant $C_\gamma$ appears in the gradient second-moment bound
$\E_\mu\norm{\nabla V}^2\leq2\mathrm{FI}(\mu\|\pi)+C_\gamma$.
The coefficients $A_\gamma$, $B_\gamma$, and $Q_\gamma$ below control
the discretization error for updates using exact gradients.
When $L_1>0$ and $0<\gamma<1$, let
\begin{align}
 A_\gamma(h)&=3L_0^2(1+2^{\gamma^2/(1-\gamma)})^2dh+3^{2\gamma+1}2^{(2\gamma^2+\gamma+2)/(1-\gamma)}
   (1-\gamma)^{2\gamma/(1-\gamma)}L_1^{2/(1-\gamma)}
   \left[\left(d+\frac{2}{1-\gamma}\right)h\right]^{1/(1-\gamma)},\notag\\
 B_\gamma(h)
 &=\left[L_0^2(1+2^{\gamma^2/(1-\gamma)})^2
   +3^{2\gamma}2^{2\gamma^2/(1-\gamma)}(1+L_0)^2\right]h^2\notag\\
 &\quad+3^{2\gamma+1}2^{(2\gamma^2+\gamma+1)/(1-\gamma)}
   (1-\gamma)^{2\gamma/(1-\gamma)}
   [(1+L_0)h]^{2/(1-\gamma)},\notag\\
 Q_\gamma(h)&=3^{2\gamma+1}2^{2\gamma^2/(1-\gamma)}L_1(1+L_0)dh.
 \label{app:eq:exact-coefficients-interior}
\end{align}
At $L_1>0$ and $\gamma=1$, define
\begin{align}
 E_1^{\mathrm{ex}}(h)
 &:=\exp\{2(1+L_0)h+2\sqrt2L_1\sqrt{dh}+8L_1^2h\},\notag\\
 A_1(h)&:=4E_1^{\mathrm{ex}}(h)L_0^2(d+1)h,\notag\\
 B_1(h)&:=\tfrac43E_1^{\mathrm{ex}}(h)
       \{L_0^2+(1+L_0)^2\}h^2,\notag\\
 Q_1(h)&:=4E_1^{\mathrm{ex}}(h)L_1(1+L_0)(d+1)h.
 \label{app:eq:exact-coefficients-one}
\end{align}
When $L_1=0$, the sharper coefficients, valid for every
$0<\gamma\leq1$, are
\begin{equation}
 A_\gamma(h)=dL_0^2h,\qquad
 B_\gamma(h)=\tfrac13L_0^2h^2,\qquad Q_\gamma(h)=0.
 \label{app:eq:exact-coefficients-smooth}
\end{equation}
The following quantities define the coefficient multiplying the discretization error and the terms independent of that error, respectively. 
\begin{align}
 \Delta_\gamma(h)
 &:=8B_\gamma(h)+8^{\gamma/2}Q_\gamma(h),\notag\\
 R_\gamma(h,M)
 &:=A_\gamma(h)+8B_\gamma(h)M
                  +8^{\gamma/2}Q_\gamma(h)M^{\gamma/2}.
 \label{app:eq:exact-remainder}
\end{align}
Thus, $R_\gamma$ contains the error remaining after applying the moment bounds, while
$\Delta_\gamma$ is the coefficient of the error that must be absorbed.

We use a tilde to distinguish the coefficients for updates using gradient
estimates. In the discretization bound of
Lemma~\ref{app:adaptive-power-freezing}, $\widetilde{A}_\gamma$ collects
terms independent of the gradient norms, while $\widetilde{U}_\gamma$
and $\widetilde{V}_\gamma$ multiply the squared norms of the exact gradient
and the applied force, respectively. For the stochastic methods, the force
is the negative gradient estimate and has the same norm.
If $L_1=0$, set
\begin{equation}
 \widetilde{A}_\gamma(h)=L_0^2dh,\qquad
 \widetilde{U}_\gamma(h)=0,\qquad
 \widetilde{V}_\gamma(h)=\tfrac13L_0^2h^2.
 \label{app:eq:power-coefficients-smooth}
\end{equation}
If $L_1>0$ and $0<\gamma<1$, set
\begin{align}
 \widetilde{A}_\gamma(h)
 &=3\left[L_0^2(1+2^{\gamma^2/(1-\gamma)})^2
   +3^{2\gamma}2^{2\gamma^2/(1-\gamma)}L_1^2\right]dh\notag\\
 &\quad+3^{2\gamma+1}2^{(2\gamma^2+\gamma+2)/(1-\gamma)}
   (1-\gamma)^{2\gamma/(1-\gamma)}L_1^{2/(1-\gamma)}
   \left[\left(d+\frac{2}{1-\gamma}\right)h\right]^{1/(1-\gamma)},\notag\\
 \widetilde{U}_\gamma(h)
 &=3^{2\gamma}2^{2\gamma^2/(1-\gamma)}
   \left[\gamma(1+L_0)^2h^2+3L_1^2dh\right],\notag\\
 \widetilde{V}_\gamma(h)
 &=\left[L_0^2(1+2^{\gamma^2/(1-\gamma)})^2
   +(1-\gamma)3^{2\gamma}2^{2\gamma^2/(1-\gamma)}(1+L_0)^2\right]h^2\notag\\
 &\quad+3^{2\gamma+1}2^{(2\gamma^2+\gamma+1)/(1-\gamma)}
   (1-\gamma)^{2\gamma/(1-\gamma)}
   [(1+L_0)h]^{2/(1-\gamma)}.
 \label{app:eq:power-coefficients-interior}
\end{align}
At $L_1>0$ and $\gamma=1$, put
\begin{align}
 \widetilde{A}_1(h)&=4E_1^{\mathrm{ex}}(h)L_0^2(d+1)h,\notag\\
 \widetilde{U}_1(h)&=E_1^{\mathrm{ex}}(h)
       \{\tfrac43(1+L_0)^2h^2+4L_1^2(d+1)h\},\notag\\
 \widetilde{V}_1(h)&=\tfrac43E_1^{\mathrm{ex}}(h)L_0^2h^2.
 \label{app:eq:power-coefficients-one}
\end{align}
All coefficients defined above are deterministic, continuous, and nondecreasing in
$h$, and vanish at $h=0$. The case $L_1=0$ uses the sharper bounds
available under global smoothness. The case $\gamma=1$ uses the
exponential gradient variation bound instead of the power bound for
$0<\gamma<1$.

\subsection{Complete convergence guarantees}
\label{app:complete-guarantees}

\begin{theorem}[Fixed-batch robust LMC]
\label{thm:full-robust}
Suppose Assumptions~\ref{ass:global-smoothness} and~\ref{ass:BG-p} hold.
Assume $K_0<\infty$ and, if $B>0$, a known bound
$\E\norm{X_0-x_{\mathrm{ref}}}^2\leq M_{2,0}<\infty$. Fix $\delta>0$, an integer $N\geq1$, and
$0<h\leq1/(8L)$, and write $T=Nh$.  Define
\begin{equation*}
 R_{T,\delta}
 :=\left(
 3M_{2,0}+\frac{64}{5}K_0T
 +\frac{332}{15}LdT^2+3dT+\frac{88}{3}T^2\delta^2
 \right)^{1/2}
\end{equation*}
and use at every step
\begin{align*}
 b(T,\delta)
 &:=\max\left\{1,\left\lceil
 \left(\frac{2C_p[B R_{T,\delta}^{p}+G]^{2/p}}{\delta^2}
 \right)^{p/[2(p-1)]}\right\rceil\right\}\quad \text{and} \quad \lambda_k:=s_p(X_{kh})b(T,\delta)^{1/p}.
\end{align*}
Then the updates in~(\ref{sec3:geo-lmc}) satisfy
\begin{align*}
 \mathrm{FI}(\bar\mu_T\|\pi)
 \leq\overline{\mathrm{FI}}_T&\leq \frac{2K_0}{T}+\frac{32}{3}L^2dh
   +\frac{49}{12}\delta^2.
\end{align*}

For any $\varepsilon>0$, choose $\delta,h$, and $T$ so that $\delta^2=\frac{\varepsilon}{16}$, $h=\min\left\{\frac1{8L},\frac{\varepsilon}{32L^2d}\right\}$ and $T=h\max\left\{1,\left\lceil\frac{6K_0}{\varepsilon h}\right\rceil\right\}$. Set $b=b(T,\delta)$. The resulting law satisfies
$\mathrm{FI}(\bar\mu_T\|\pi)\leq\varepsilon$ and
\begin{align*}
 N_T
 &\leq1+\frac{6K_0}{\varepsilon}
       \max\left\{8L,\frac{32L^2d}{\varepsilon}\right\},\quad Q_T=4bN_T \leq4N_T\left[1+\left(
  \frac{2C_p[B R_{T,\delta}^{p}+G]^{2/p}}
       {\delta^2}
 \right)^{p/[2(p-1)]}\right].
\end{align*}
For fixed problem parameters with $K_0,L>0$ and $B>0$, as
$\varepsilon\downarrow0$,
\begin{align*}
 N_T
 &=\mathcal O\!\left(\frac{K_0L^2d}{\varepsilon^2}\right), \quad Q_T=\mathcal O\!\left(
 \frac{B^{1/(p-1)}K_0^{1+p/(p-1)}
       L^{2+p/[2(p-1)]}d^{1+p/[2(p-1)]}}
      {\varepsilon^{2+3p/[2(p-1)]}}\right).
\end{align*}
When $B=0$ and $G>0$, the oracle rate is instead
\begin{equation*}
 Q_T
 =\mathcal O\!\left(
  \frac{K_0L^2d\,G^{1/(p-1)}}
       {\varepsilon^{2+p/[2(p-1)]}}\right).
\end{equation*}
\end{theorem}

\begin{corollary}[Special cases of fixed-batch robust LMC]
\label{cor:robust-special-cases}
Under the assumptions and accuracy schedule of
Theorem~\ref{thm:full-robust}, suppose $K_0>0$ and
$K_0=\mathcal O(d)$. Then
\[
 N_T=\mathcal O\!\left(\frac{L^2d^2}{\varepsilon^2}\right).
\]
With the remaining problem parameters fixed, as
$\varepsilon\downarrow0$, the corresponding oracle complexities
specialize as follows.
\begin{enumerate}
 \item If $p=2$, $B=0$, and $G>0$, then
 \[
  Q_T=\mathcal O\!\left(
       \frac{L^2d^2G}{\varepsilon^3}
       \right).
 \]
 \item Under BG-0 noise, corresponding to $p=2$ and $B>0$,
 \[
 Q_T=\mathcal O\!\left(
       \frac{BL^3d^5}{\varepsilon^5}
       \right).
 \]
 \item If $B=G=0$, the robust construction generates the standard
 exact-gradient LMC iterates. Removing its three redundant oracle calls
 per iteration leaves the iterates unchanged and gives
 \[
  Q_T=N_T=\mathcal O\!\left(
       \frac{L^2d^2}{\varepsilon^2}
       \right).
 \]
\end{enumerate}
\end{corollary}

\begin{remark}[The case $K_0=0$]
The finite non-asymptotic bounds in Theorem~\ref{thm:full-robust}
remain valid when $K_0=0$, and the prescribed schedule gives $N_T=1$.
In this case, use the finite bounds
~(\ref{eq:app-robust-iteration-count}) and
~(\ref{eq:app-robust-finite-oracle-count}), since the asymptotic
reductions in Corollary~\ref{cor:robust-special-cases} require $K_0>0$.
\end{remark}

\begin{theorem}[Exact-gradient adaptive-step LMC]
\label{thm:full-adaptive}
Let $V\in C^1(\R^d)$, let $0<Z<\infty$, and suppose
$K_0<\infty$. Under Assumption~\ref{ass:generalized-smoothness}, run
the adaptive LMC in~(\ref{eq:adaptive-stepsize}), truncate only its
last occupied interval at a prescribed horizon $T>0$. Let
$M_T:=1+C_\gamma+K_0/T$ and fix a base step satisfying
$0<h\leq[1+2\Delta_\gamma(1)]^{-1}$.
Then $N_T<\infty$ almost surely, and
\begin{equation}
 \mathrm{FI}(\bar\mu_T\|\pi)
 \leq\overline{\mathrm{FI}}_T
 \leq\frac{4K_0}{3T}+\frac83R_\gamma(h,M_T).
 \label{app:eq:adaptive-fi}
\end{equation}
Moreover, $Q_T=N_T$, and
\begin{equation}
 \E N_T\leq1+\frac{T}{h}
 \left[1+\omega
 \{8(M_T+2R_\gamma(h,M_T))\}^{\gamma/2}\right].
 \label{app:eq:adaptive-events}
\end{equation}

For $0<\varepsilon\leq1$ and $K_0>0$, choose
\begin{equation}
 T:=\frac{4K_0}{\varepsilon},
 \qquad
 h:=\max\left\{u\in[0,(1+2\Delta_\gamma(1))^{-1}]:
 R_\gamma(u,M_T)\leq\frac{3\varepsilon}{16}\right\}.
 \label{app:eq:adaptive-accuracy-schedule}
\end{equation}
This choice is well defined and yields $h>0$. It gives $\mathrm{FI}(\bar\mu_T\|\pi)\leq\varepsilon$ and
\begin{equation}
 \E Q_T=\E N_T\leq1+\frac{4K_0}{\varepsilon h}
       \left(1+4\omega M_T^{\gamma/2}\right).
 \label{app:eq:adaptive-accuracy-count}
\end{equation}
If $L_1>0$, let $\Lambda:=\max\{1,L_0,L_1\}$. With the remaining
problem parameters fixed, as $\varepsilon\downarrow0$,
\begin{equation}
 \E N_T=\E Q_T
 =\mathcal O\!\left(
 \frac{K_0\Lambda^{4/(2-\gamma)}
 d^{(2+\gamma)/(2-\gamma)}}{\varepsilon^2}\right).
 \label{app:eq:adaptive-compact-rate}
\end{equation}
\end{theorem}

\begin{corollary}[Complexity specializations of adaptive-step LMC]
\label{cor:adaptive-special-cases}
Under the assumptions and accuracy schedule of
Theorem~\ref{thm:full-adaptive}, suppose $K_0>0$ and
$K_0=\mathcal O(d)$. Then the following statements hold as
$\varepsilon\downarrow0$.
\begin{enumerate}
\item If $L_1>0$ and $\Lambda:=\max\{1,L_0,L_1\}$, then
\begin{equation*}
\E N_T=\E Q_T
 =\mathcal O\!\left(
 \frac{\Lambda^{4/(2-\gamma)}
 d^{4/(2-\gamma)}}{\varepsilon^2}\right).
\end{equation*}
\item If $L_1=0$, the step size is constant and the oracle count is
deterministic. For all sufficiently small $\varepsilon$,
\begin{equation}
 Q_T=N_T
 \leq1+\frac{128K_0dL_0^2}{3\varepsilon^2}
 \label{app:eq:adaptive-smooth-rate}
\end{equation}
and hence $Q_T=N_T=\mathcal O(d^2L_0^2/\varepsilon^2)$ under
$K_0=\mathcal O(d)$.
\end{enumerate}
\end{corollary}

\begin{remark}[The case $K_0=0$]
If $K_0=0$, then $\mu_0=\pi$, so $X_0$ is already an exact target sample
and no simulation or gradient evaluation is required. The finite horizon
bounds in Theorem~\ref{thm:full-adaptive} remain valid, but its
positive $K_0$ accuracy schedule is unnecessary.
\end{remark}

\begin{theorem}[Robust adaptive step LMC]
\label{thm:full-combined}
Let $V\in C^1(\R^d)$, let $0<Z<\infty$, and suppose $K_0<\infty$.
Suppose Assumptions~\ref{ass:generalized-smoothness} and~\ref{ass:BG-p}
hold with $0<\gamma\leq1$ and $1<p\leq2$. If $B>0$, assume $\E\norm{X_0-x_{\mathrm{ref}}}^2\leq M_{2,0}<\infty$.
When $B=0$, set $M_{2,0}=0$. Define
\begin{equation*}
 \mathcal D_\gamma(u)
 :=\widetilde{A}_\gamma(u)
   +\{\widetilde{U}_\gamma(u)+2\widetilde{V}_\gamma(u)\}(1+C_\gamma).
\end{equation*}
For $0<\varepsilon\leq1$, choose $T:=\frac{8(K_0+1)}{\varepsilon}$, $M_T:=1+C_\gamma+K_0/T$, and $h:=\max\{u\in[0,1]:\mathcal D_\gamma(u)\leq\varepsilon/64\}.$
Set
\begin{equation*}
      \kappa:=(1+\bar s_T^2)^{-1}, \quad \text{and} \quad b:=\left\lceil\left[
 \frac{260C_p\kappa^{-1}(1+\omega M_T^{\gamma/2})}{\varepsilon}
 \right]^{p/[2(p-1)]}\right\rceil,
\end{equation*}where $\bar s_T^2:=\left[B\left(6M_{2,0}+8dT+80T^2M_T\right)^{p/2}+G\right]^{2/p}$ and $\omega:=\frac{L_1}{1+L_0}.$
Run~(\ref{eq:robust-lmc-adaptive-iterates}) using $b$ independent
quadruples per iteration and clipping radius
$\lambda_k=s_p(X_{\tau_k})b^{1/p}$, truncating the last occupied interval
at $T$. Then $h>0$, $N_T<\infty$ almost surely, and
\begin{equation*}
 \mathrm{FI}(\bar\mu_T\|\pi)
 \leq\overline{\mathrm{FI}}_T
 \leq\frac{2K_0}{T}+9\mathcal D_\gamma(h)+\frac{\varepsilon}{2}
 \leq\varepsilon.
\end{equation*}
Moreover, $Q_T=4bN_T$ pathwise, with
\begin{equation*}
 \E N_T\leq2+\frac{24T}{h}(1+\omega M_T^{\gamma/2}),
 \qquad \E Q_T=4b\E N_T.
\end{equation*}
The asymptotic complexity specializations are given in
Corollary~\ref{cor:combined-special-cases}.
\end{theorem}

\begin{corollary}[Complexity specializations of robust adaptive step LMC]
\label{cor:combined-special-cases}
Under the assumptions and accuracy schedule of
Theorem~\ref{thm:full-combined}, suppose $K_0>0$ and
$K_0=\mathcal O(d)$. With the remaining problem parameters fixed, the
following statements hold as $\varepsilon\downarrow0$.
\begin{enumerate}
\item If $L_1>0$ and $\Lambda:=\max\{1,L_0,L_1\}$, then
\begin{equation*}
 \E N_T
 =\mathcal O\!\left(
 \frac{\Lambda^{(8-2\gamma)/(2-\gamma)}
       d^{(6-\gamma)/(2-\gamma)}}{\varepsilon^2}
 \right).
\end{equation*}
When $B>0$,
\begin{align*}
 \E Q_T
 &=\mathcal O\!\left(
 \frac{B^{1/(p-1)}
 \Lambda^{\frac{(8-2\gamma)(p-1)+2p}
                  {(p-1)(2-\gamma)}}
 d^{\frac{(3p-2)(6-\gamma)}{2(p-1)(2-\gamma)}}}
 {\varepsilon^{2+3p/[2(p-1)]}}
 \right),
\end{align*}
whereas, when $B=0$,
\begin{align*}
 \E Q_T
 &=\mathcal O\!\left(
 \frac{(1+G)^{1/(p-1)}
 \Lambda^{\frac{(8-2\gamma)(p-1)+p}
                  {(p-1)(2-\gamma)}}
 d^{1+\frac{8(p-1)+\gamma p}
               {2(p-1)(2-\gamma)}}}
 {\varepsilon^{2+p/[2(p-1)]}}
 \right).
\end{align*}
In particular, for $p=2$, the oracle complexity has accuracy dependence
$\varepsilon^{-5}$ under BG-0 noise ($B>0$) and
$\varepsilon^{-3}$ under state-independent finite-variance noise
($B=0$ and $G>0$). With the constant in $K_0=\mathcal O(d)$ fixed,
the hidden numerical constants depend only on
$\gamma$ for the iteration bound and on $p,\gamma$ for the oracle bounds.
At $\gamma=1$, the same displayed powers follow by using the corresponding
coefficients directly.

\item If $L_1=0$, the noise-aware iteration count can remain random, but
\begin{equation*}
 \E N_T
 =\mathcal O\!\left(
 \frac{(K_0+1)dL_0^2}{\varepsilon^2}\right)
 =\mathcal O\!\left(
 \frac{d^2L_0^2}{\varepsilon^2}\right).
\end{equation*}
\end{enumerate}
\end{corollary}

\begin{remark}[The case $K_0=0$]
If $K_0=0$, then $\mu_0=\pi$, so $X_0$ is already an exact target sample
and no simulation or oracle evaluation is required. The bounds in
Theorem~\ref{thm:full-combined} remain valid, but its accuracy
schedule is unnecessary.
\end{remark}

\begin{remark}[Algorithmic reductions]
When $L_1=0$, $\omega=0$, but Theorem~\ref{thm:full-combined} still
prescribes $\kappa=(1+\bar s_T^2)^{-1}>0$, so its noise-aware step
count can remain random. If one instead sets $\kappa=0$, RA-LMC becomes
the fixed-step robust method, and Theorem~\ref{thm:full-robust} and
Corollary~\ref{cor:robust-special-cases} apply. If $B=G=0$, the oracle
equals $\nabla V$ almost surely. Bypassing the four-call estimator and
using one exact gradient per update gives the adaptive exact-gradient
method, so Theorem~\ref{thm:full-adaptive} and
Corollary~\ref{cor:adaptive-special-cases} give the sharper bounds.
\end{remark}


\subsection{Proof of Theorem~\ref{thm:full-robust}}
\label{app:proof-robust}

The proof has three components. First, in
Subsection~\ref{app:proof-robust-estimator}, we show that geometric-median
centering followed by residual clipping converts a conditional
$p$-th-moment oracle bound into a conditional mean-squared error bound.
Second, Subsection~\ref{app:proof-robust-fisher} uses a one-step
KL-dissipation argument to convert the mean-squared drift error into a bound
on time-averaged relative Fisher information. Finally,
Subsection~\ref{app:proof-robust-trajectory} controls the average squared
distance of the chain from $x_{\mathrm{ref}}$. This estimate controls the
dependence of the gradient-estimation error on the state-dependent noise
scale while retaining a fixed batch size.
Subsection~\ref{app:proof-robust-complexity} then verifies the accuracy
schedule and derives the iteration and oracle complexities.
Throughout this proof, $H_k=h$, $\tau_k=kh$, and
\[
    \widehat{\mathcal G}_k
    :=\frac1b\sum_{i=1}^b\widehat{\mathcal G}_{k,i}=-F_k.
\]
Thus $\widehat{\mathcal G}_k$ estimates the gradient, whereas $F_k$ is
the Langevin force in~(\ref{eq:common-lmc}). All expectations below include
both the oracle and Brownian randomness unless a conditioning event or
state is specified.

\subsubsection{Mean-squared error of the robust estimator}
\label{app:proof-robust-estimator}

We begin with the geometric median of three oracle evaluations. It provides
a center with a finite second moment even when one oracle call has only a
finite $p$-th moment. A large error in this center requires at least two of
the three independent oracle calls to be large, which squares the tail
probability.

The good-points implication below is the three-point specialization of
\citet[Lemma~2.1]{minsker2015geometric}; its use of independent pilot
estimators is also related to the amplification principle in their
Theorem~3.1.  We retain the short calculation because it supplies the
particular tail integral and constants needed here.

\begin{lemma}[Tail amplification by the geometric-median center]
\label{lem:robust-pilot}
Let $Y_1,Y_2,Y_3$ be independent copies of a random vector $Y$ with mean
$\theta$ and
\[
    \E\|Y-\theta\|^p\leq s^p,
    \qquad 1<p\leq2.
\]
Let $C=\mathrm{GeoMed}(Y_1,Y_2,Y_3)$. For every $t>0$ and
$0<a<2p$,
\begin{align}
    \mathbb P(\|C-\theta\|>t)
    &\leq \min\{1,3(4s/t)^{2p}\},
    \label{eq:app-pilot-tail}\\
    \E\|C-\theta\|^a
    &\leq \frac{2p}{2p-a}\,4^a3^{a/(2p)}s^a.
    \label{eq:app-pilot-moment}
\end{align}
\end{lemma}

\begin{proof}
Suppose two of the three observations, say $Y_1$ and $Y_2$, lie within
distance $u$ of $\theta$. Since $C$ minimizes the sum of distances,
\[
    \sum_{j=1}^3\|C-Y_j\|
    \leq \sum_{j=1}^3\|\theta-Y_j\|.
\]
The triangle inequality bounds the left-hand side from below by
\[
    2\|C-\theta\|-2u
    +\|\theta-Y_3\|-\|C-\theta\|,
\]
whereas the right-hand side is at most
$2u+\|\theta-Y_3\|$. Hence $\|C-\theta\|\leq4u$.
Consequently, $\|C-\theta\|>t$ requires at least two of the events
$\{\|Y_j-\theta\|>t/4\}$ to occur. A union bound over the three pairs,
independence, and Markov's inequality give
\[
    \mathbb P(\|C-\theta\|>t)
    \leq3\,\mathbb P(\|Y-\theta\|>t/4)^2
    \leq3\cdot4^{2p}(s/t)^{2p}.
\]
Every probability is at most one, so the preceding estimate proves
~(\ref{eq:app-pilot-tail}).

When $s>0$, set $t_*=4\,3^{1/(2p)}s$, where the two bounds in
~(\ref{eq:app-pilot-tail}) coincide. Tail integration yields
\begin{align*}
    \E\|C-\theta\|^a
    &\leq a\int_0^{t_*}t^{a-1}\,\mathrm{d}t
      +3a(4s)^{2p}\int_{t_*}^{\infty}t^{a-1-2p}\,\mathrm{d}t\\
    &=t_*^a\left(1+\frac{a}{2p-a}\right)
      =\frac{2p}{2p-a}\,4^a3^{a/(2p)}s^a,
\end{align*}
which is~(\ref{eq:app-pilot-moment}). If $s=0$, then
$Y=\theta$ almost surely and both conclusions are immediate.
\end{proof}

We call each four-call estimate $\widehat{\mathcal G}_{k,i}$
in~(\ref{eq:geomed}) a \emph{robust replicate}. Three oracle evaluations
form its geometric-median center, and the residual of a fourth evaluation
is clipped around this center. The next lemma bounds the mean-squared error
of the average of $b$ independent robust replicates. Clipping limits the
variance of each replicate but introduces bias. Increasing the clipping
radius reduces this bias but weakens the variance control, while averaging
independent replicates reduces the variance by a factor of $1/b$. We choose
the clipping radius to balance the squared bias and the variance of the
averaged estimator.
Recall the final estimator error coefficient
\begin{align*}
 C_p:=2^{2/p-1}\Bigg[2^{2p-2}\bigl(1+2\sqrt3\,4^p\bigr)^2+\frac{32p\,3^{1/p}}{p-1}2^p\bigl(1+2\sqrt3\,4^p\bigr)\Bigg].
\end{align*}

\begin{lemma}[Conditional mean-squared error of the robust estimator]
\label{app:robust-estimation}
Fix a deterministic state $x$ and let
$s=s_p(x)=[B\|x-x_{\mathrm{ref}}\|^p+G]^{1/p}$. For an integer $b\geq1$, form
$b$ independent four-call replicates as in~(\ref{eq:geomed}) and average
them to obtain $\widehat{\mathcal G}_{b,\lambda}(x)$. If $\lambda \geq s>0$,
then
\begin{align}
    \E\!\left[
      \|\widehat{\mathcal G}_{b,\lambda}(x)-\nabla V(x)\|^2
      \right]
    &\leq2^{2p-2}\bigl(1+2\sqrt3\,4^p\bigr)^2
      s^{2p}\lambda^{2-2p}+\frac1b\left[
      \frac{32p\,3^{1/p}}{p-1}
      +2^p\bigl(1+2\sqrt3\,4^p\bigr)\right]s^p\lambda^{2-p}.
    \label{eq:app-robust-bias-variance}
\end{align}
The first term bounds the squared bias, and the second bounds the variance
after averaging.
In particular, for $\lambda=s b^{1/p}$,
\begin{equation*}
    \E\!\left[
      \|\widehat{\mathcal G}_{b,\lambda}(x)-\nabla V(x)\|^2
      \right]
    \leq C_p s_p(x)^2 b^{-2(p-1)/p}.
\end{equation*}
If $s=0$, the oracle and estimator are exact almost surely, so the bound
remains valid with $\lambda=0$.
\end{lemma}

\begin{proof}
The assertion is immediate when $s=0$, so assume $s>0$. Suppress the
replicate index and write $Y_j=\mathcal G(x,\xi_j)$,
$\theta=\nabla V(x)$,
\[
    C=\mathrm{GeoMed}(Y_1,Y_2,Y_3),
    \qquad
    \widehat Y=C+\mathrm{clip}_{\lambda}(Y_4-C).
\]
Let $A=C-\theta$, $U=Y_4-\theta$, and $R=Y_4-C=U-A$.
Lemma~\ref{lem:robust-pilot} gives
\[
    \E\|A\|^p\leq2\sqrt3\,4^p s^p,
    \qquad
    \E\|A\|^2\leq\frac{16p\,3^{1/p}}{p-1}s^2.
\]
Since $\|u+v\|^p\leq2^{p-1}(\|u\|^p+\|v\|^p)$,
\begin{equation*}
    \E\|R\|^p\leq2^{p-1}\bigl(1+2\sqrt3\,4^p\bigr)s^p.
\end{equation*}

Define the discarded residual
$T_{\lambda}(R):=R-\mathrm{clip}_{\lambda}(R)$. Because
$\widehat Y=Y_4-T_{\lambda}(R)$ and $\E Y_4=\theta$,
\begin{align}
    \|\E\widehat Y-\theta\|
    &\leq \E\|T_{\lambda}(R)\|
      \leq \lambda^{1-p}\E\|R\|^p\notag\\
    &\leq2^{p-1}\bigl(1+2\sqrt3\,4^p\bigr)s^p\lambda^{1-p}.
    \label{eq:app-replicate-bias}
\end{align}
Here we used
$(z-\lambda)_+\leq z\mathbf 1_{\{z>\lambda\}}
\leq\lambda^{1-p}z^p$ for $z\geq0$. Similarly,
$\min\{z,\lambda\}^2\leq\lambda^{2-p}z^p$, and therefore
\begin{align}
    \E\|\widehat Y-\theta\|^2
    &\leq2\E\|A\|^2
      +2\E\|\mathrm{clip}_{\lambda}(R)\|^2
      \notag\\
    &\leq\frac{32p\,3^{1/p}}{p-1}s^2
      +2^p\bigl(1+2\sqrt3\,4^p\bigr)s^p\lambda^{2-p}\notag\\
    &\leq\left[\frac{32p\,3^{1/p}}{p-1}
      +2^p\bigl(1+2\sqrt3\,4^p\bigr)\right]s^p\lambda^{2-p}.
    \label{eq:app-replicate-second-moment}
\end{align}
The last inequality uses $\lambda\geq s$ and $p\leq2$.

Let $\widehat Y_1,\ldots,\widehat Y_b$ be the independent replicates. The exact
bias--variance decomposition gives
\[
    \E\left\|\frac1b\sum_{j=1}^b\widehat Y_j-\theta\right\|^2
    =\|\E\widehat Y-\theta\|^2
     +\frac1b\E\|\widehat Y-\E\widehat Y\|^2.
\]
Using~(\ref{eq:app-replicate-bias}) and
~(\ref{eq:app-replicate-second-moment}) proves
~(\ref{eq:app-robust-bias-variance}). With
$\lambda=s b^{1/p}$, both terms are proportional to
$s^2b^{-2(p-1)/p}$, so
\begin{align*}
    \E\|\widehat{\mathcal G}_{b,\lambda}(x)-\theta\|^2
    &\leq2^{1-2/p}C_p s^2b^{-2(p-1)/p}\\
    &\leq C_ps^2b^{-2(p-1)/p}.
\end{align*}
After the substitution, factoring out $s^2b^{-2(p-1)/p}$ in
~(\ref{eq:app-robust-bias-variance}) leaves the coefficient
$2^{1-2/p}C_p$, where $C_p$ is defined above.
The final inequality uses $2^{1-2/p}\leq1$ for $1<p\leq2$.
\end{proof}

The fixed-batch argument uses only a second-moment trajectory bound. The
following elementary observation is what allows that bound to control the
state-dependent local noise scale.

\begin{lemma}[Concavity of the squared local $L^p$ noise scale]
\label{lem:robust-noise-concavity}
For $1<p\leq2$, the squared local noise scale
\begin{equation}
    s_p(x)^2
    =\bigl[B\norm{x-x_{\mathrm{ref}}}^{p}+G\bigr]^{2/p}
    \label{eq:app-squared-noise-profile}
\end{equation}
is a nonnegative, nondecreasing, and concave function of the squared
distance $\norm{x-x_{\mathrm{ref}}}^2$. Moreover,
\begin{equation}
    s_p(x)^2
    \leq(B+G)^{2/p}
       \bigl(1+\norm{x-x_{\mathrm{ref}}}^2\bigr).
    \label{eq:app-noise-linear-envelope}
\end{equation}
\end{lemma}

\begin{proof}
If $B=0$ or $G=0$, the squared local noise scale is respectively constant
or linear in the squared distance, and the conclusions follow directly.
Suppose now that $B,G>0$. At a positive squared distance $r$, direct
differentiation gives
\begin{align*}
    \frac{\mathrm{d}^2}{\mathrm{d}r^2}
      [Br^{p/2}+G]^{2/p}=
      B\left(\frac p2-1\right)G
      r^{p/2-2}(Br^{p/2}+G)^{2/p-2}
      \leq0,
\end{align*}
because $p\leq2$. Continuity extends concavity to $r=0$.
Monotonicity is clear from the first derivative.

For $0\leq r\leq1$, monotonicity gives the bound
in~(\ref{eq:app-noise-linear-envelope}). For $r>1$, concavity at
$1=r^{-1}r+(1-r^{-1})0$ and nonnegativity at zero give
\[
    [B+G]^{2/p}
    \geq r^{-1}[Br^{p/2}+G]^{2/p}.
\]
This is stronger than~(\ref{eq:app-noise-linear-envelope}) for $r>1$.
\end{proof}

\subsubsection{From drift error to relative Fisher information}
\label{app:proof-robust-fisher}

We restate the following two results from
\citet{balasubramanian22a} in our notation and use them without repeating
their proofs. The first gives KL dissipation for an interpolation that uses
a possibly random frozen gradient vector.

\begin{lemma}[{\cite[Lemma~23]{balasubramanian22a}}]
\label{lem:robust-frozen-drift}
Let $(X,g)$ be jointly distributed random vectors with finite second
moments, and let $(W_t)_{t\geq0}$ be a Brownian motion independent of
$(X,g)$. Define
\[
    Y_t=X-tg+\sqrt2 W_t,
    \qquad \nu_t=\operatorname{Law}(Y_t).
\]
Then, for $t>0$,
\begin{align}
    \frac{\mathrm{d}}{\mathrm{d}t}\mathrm{KL}(\nu_t\|\pi)
    &\leq-\frac34\mathrm{FI}(\nu_t\|\pi)
      +\E\bigl\|\nabla V(Y_t)-\E[g\mid Y_t]\bigr\|^2
      \notag\\
    &\leq-\frac34\mathrm{FI}(\nu_t\|\pi)
      +\E\|\nabla V(Y_t)-g\|^2.
    \label{eq:app-frozen-drift}
\end{align}
\end{lemma}

The second result controls the gradient second moment by relative Fisher
information. It is the point at which dimension enters the globally smooth
analysis.

\begin{lemma}[{\cite[Lemma~24]{balasubramanian22a}}]
\label{lem:robust-gradient-energy}
Under Assumption~\ref{ass:global-smoothness}, every probability measure
$\nu$ satisfies
\begin{equation}
    \E_{\nu}\|\nabla V\|^2
    \leq \mathrm{FI}(\nu\|\pi)+2Ld
    \label{eq:app-gradient-energy}.
\end{equation}
\end{lemma}

The preceding two lemmas yield the following transfer principle. It isolates
the Langevin discretization error from the mean-squared error of the robust
drift estimates.
Here $K_0:=\mathrm{KL}(\mu_0\|\pi)$ is the initial KL divergence.

\begin{proposition}[Mean-squared drift error implies Fisher control]
\label{prop:robust-mse-to-fi}
Suppose Assumption~\ref{ass:global-smoothness} holds,
$K_0<\infty$, and $\E\|X_0-x_{\mathrm{ref}}\|^2<\infty$. Let
\[
    e_k:=\E\|\widehat{\mathcal G}_k-\nabla V(X_{kh})\|^2,
    \qquad
    \bar e_N:=\frac1N\sum_{k=0}^{N-1}e_k,
\]
and suppose $e_k<\infty$ for $k=0,\ldots,N-1$.
If $0<h\leq(8L)^{-1}$, then
\begin{align}
    \frac1{Nh}\int_0^{Nh}\mathrm{FI}(\mu_t\|\pi)\,\mathrm{d}t
    &\leq \frac{2K_0}{Nh}+8L^2dh+\frac{64}{3}L^3dh^2+
      \left(4+\frac{16}{3}L^2h^2\right)\bar e_N.
    \label{eq:app-mse-to-fi}
\end{align}
\end{proposition}

\begin{proof}
Fix $k$ and $u\in[0,h]$, and set
\[
    \Delta_{k,u}:=\E\|X_{kh+u}-X_{kh}\|^2,
    \qquad
    I_{kh+u}:=\mathrm{FI}(\mu_{kh+u}\|\pi).
\]
Independence and centering of the Brownian increment give
\begin{equation}
    \Delta_{k,u}
    =u^2\E\|\widehat{\mathcal G}_k\|^2+2du.
    \label{eq:app-increment}
\end{equation}
Twice applying $\|a+b\|^2\leq2\|a\|^2+2\|b\|^2$ and using global
$L$-smoothness yields
\[
    \E\|\widehat{\mathcal G}_k\|^2
    \leq4\E\|\nabla V(X_{kh+u})\|^2
       +4L^2\Delta_{k,u}+2e_k.
\]
Since $u\leq h\leq(8L)^{-1}$,
$(1-4L^2u^2)^{-1}\leq16/15<4/3$. Solving
~(\ref{eq:app-increment}) for $\Delta_{k,u}$ therefore gives
\begin{equation}
    \Delta_{k,u}
    \leq\frac{16}{3}u^2\E\|\nabla V(X_{kh+u})\|^2
      +\frac83u^2e_k+\frac83du.
    \label{eq:app-increment-bound}
\end{equation}
Apply Lemma~\ref{lem:robust-gradient-energy} with
$\nu=\mu_{kh+u}$. Since
$I_{kh+u}=\mathrm{FI}(\mu_{kh+u}\|\pi)$, it gives
\begin{equation}
    \E\|\nabla V(X_{kh+u})\|^2
    \leq I_{kh+u}+2Ld.
    \label{eq:app-gradient-energy-applied}
\end{equation}
Moreover, global $L$-smoothness and the definition of $e_k$ imply
\[
    \E\|\nabla V(X_{kh+u})-\widehat{\mathcal G}_k\|^2
    \leq 2L^2\Delta_{k,u}+2e_k.
\]
Substituting~(\ref{eq:app-increment-bound}) and
~(\ref{eq:app-gradient-energy-applied}) into this inequality gives
\begin{align}
    \E\|\nabla V(X_{kh+u})-\widehat{\mathcal G}_k\|^2
    &\leq\frac{32}{3}L^2u^2I_{kh+u}
      +\frac{64}{3}L^3du^2+\frac{16}{3}L^2du
      +\left(2+\frac{16}{3}L^2u^2\right)e_k
      \notag\\
    &\leq
      16L^2u^2I_{kh+u}+32L^3du^2+8L^2du
      +(2+8L^2u^2)e_k.
    \label{eq:app-within-step-drift-error}
\end{align}

We next apply Lemma~\ref{lem:robust-frozen-drift} with $X=X_{kh}$, $g=\widehat{\mathcal G}_k$, and $W_u=W_{kh+u}-W_{kh}$. The fresh Brownian motion above is independent of this pair
by the standing sampling convention. Under this substitution,
\[
    Y_u=X_{kh}-u\widehat{\mathcal G}_k
          +\sqrt2(W_{kh+u}-W_{kh})=X_{kh+u},
    \qquad \nu_u=\mu_{kh+u}.
\]
Consequently, the second inequality in~(\ref{eq:app-frozen-drift}) gives,
for $u\in(0,h]$,
\begin{equation}
    \frac{\mathrm d}{\mathrm du}
       \mathrm{KL}(\mu_{kh+u}\|\pi)
    \leq-\frac34 I_{kh+u}
       +\E\|\nabla V(X_{kh+u})-\widehat{\mathcal G}_k\|^2.
    \label{eq:app-frozen-drift-applied}
\end{equation}
Thus, we apply the lemma to the joint law rather than conditioning on the
realized state, which would replace $\mu_{kh}$ by a point mass.

Substitute~(\ref{eq:app-within-step-drift-error}) into
~(\ref{eq:app-frozen-drift-applied}). Since
$16L^2u^2\leq16L^2h^2\leq1/4$, the resulting Fisher information term is
absorbed into the coefficient $3/4$. Integrating over $u\in[0,h]$ gives
\begin{equation}
    \mathrm{KL}(\mu_{(k+1)h}\|\pi)
      -\mathrm{KL}(\mu_{kh}\|\pi)\leq
      -\frac12\int_{kh}^{(k+1)h}\mathrm{FI}(\mu_t\|\pi)\,\mathrm{d}t
      +\frac{32}{3}L^3dh^3+4L^2dh^2
      +\left(2h+\frac83L^2h^3\right)e_k.
    \label{eq:app-one-step-kl}
\end{equation}
Summing~(\ref{eq:app-one-step-kl}) over $k$, discarding the nonnegative
terminal KL divergence, and dividing by $Nh/2$ proves
~(\ref{eq:app-mse-to-fi}).
\end{proof}

\subsubsection{Trajectory moment control and completion of the finite bound}
\label{app:proof-robust-trajectory}

We first consider $B>0$, using the known initial second-moment bound
$\E\|X_0-x_{\mathrm{ref}}\|^2\leq M_{2,0}$ in
Theorem~\ref{thm:full-robust}. The case $B=0$ is treated
separately below. Write $T=Nh$ and let
\begin{equation*}
    \overline M_N
    :=\frac1N\sum_{k=0}^{N-1}\E\|X_{kh}-x_{\mathrm{ref}}\|^2.
\end{equation*}
The prescribed clipping radius is
$\lambda_k=s_p(X_{kh})b^{1/p}$, where
$s_p(x)=[B\|x-x_{\mathrm{ref}}\|^p+G]^{1/p}$.
Lemma~\ref{app:robust-estimation}, conditioning first on the current
state, and two applications of Jensen's inequality using
Lemma~\ref{lem:robust-noise-concavity} give
\begin{align}
    e_k
    &\leq C_pb^{-2(p-1)/p}
      \E\!\left[(B\|X_{kh}-x_{\mathrm{ref}}\|^p+G)^{2/p}\right]
      \notag\\
    &\leq C_pb^{-2(p-1)/p}
      [B(\E\|X_{kh}-x_{\mathrm{ref}}\|^2)^{p/2}+G]^{2/p},
    \label{eq:app-error-at-step}\\
    \bar e_N
    &\leq C_pb^{-2(p-1)/p}
      [B\overline M_N^{p/2}+G]^{2/p}.
    \label{eq:app-average-error-moment}
\end{align}
All quantities in~(\ref{eq:app-error-at-step}) are finite. Indeed, the initial second-moment bound starts an induction. If
$X_{kh}$ has a finite second moment, Lemma~\ref{app:robust-estimation}
and~(\ref{eq:app-noise-linear-envelope}) imply that $e_k=\E\|\widehat{\mathcal G}_k-\nabla V(X_{kh})\|^2<\infty$. Global $L$-smoothness then gives $\E\|\nabla V(X_{kh})\|^2<\infty$, so
$\E\|\widehat{\mathcal G}_k\|^2<\infty$. The update~(\ref{sec3:geo-lmc}) therefore gives a finite second moment for
$X_{(k+1)h}$.

Proposition~\ref{prop:robust-mse-to-fi} and $Lh\leq1/8$ imply
\begin{equation}
    \int_0^T\mathrm{FI}(\mu_t\|\pi)\,\mathrm{d}t\leq2K_0+\frac43LdT+\frac{49}{12}T\bar e_N.
    \label{eq:app-integrated-fi-simple}
\end{equation}
To see the constants, after multiplying~(\ref{eq:app-mse-to-fi}) by
$T$, the two discretization terms are at most $LdT$ and $LdT/3$, while
$4+(16/3)L^2h^2\leq49/12$.

We next control the energy of the applied robust drifts. The estimate used
to derive~(\ref{eq:app-increment-bound}), together with
~(\ref{eq:app-increment}), gives
\[
    (1-4L^2u^2)\E\|\widehat{\mathcal G}_k\|^2
    \leq4\E\|\nabla V(X_{kh+u})\|^2+8L^2du+2e_k.
\]
Use $(1-4L^2u^2)^{-1}\leq16/15$, integrate over $u\in[0,h]$, and apply
Lemma~\ref{lem:robust-gradient-energy} to obtain
\[
    h\E\|\widehat{\mathcal G}_k\|^2
    \leq\frac{64}{15}\int_{kh}^{(k+1)h}
          \mathrm{FI}(\mu_t\|\pi)\,\mathrm{d}t
      +\frac{128}{15}Ldh+\frac{64}{15}L^2dh^2
      +\frac{32}{15}he_k.
\]
Summing this inequality, substituting~(\ref{eq:app-integrated-fi-simple}),
and using $Lh\leq1/8$ gives
\begin{equation}
    h\sum_{k=0}^{N-1}\E\|\widehat{\mathcal G}_k\|^2
    \leq\frac{128}{15}K_0+
      \frac{664}{45}LdT+
      \frac{176}{9}T\bar e_N.
    \label{eq:app-applied-drift-energy}
\end{equation}

For $t_k=kh$, the summed recursion is
\[
    X_{kh}-x_{\mathrm{ref}}
    =X_0-x_{\mathrm{ref}}-h\sum_{j=0}^{k-1}\widehat{\mathcal G}_j
      +\sqrt2(W_{kh}-W_0).
\]
The three term square inequality, Cauchy--Schwarz for the drift sum, and
$\E\|W_{kh}-W_0\|^2=dt_k$ yield
\[
    \E\|X_{kh}-x_{\mathrm{ref}}\|^2
    \leq3M_{2,0}
      +3t_kh\sum_{j=0}^{N-1}\E\|\widehat{\mathcal G}_j\|^2
      +6dt_k.
\]
Since $N^{-1}\sum_{k=0}^{N-1}t_k=(N-1)h/2\leq T/2$, averaging over $k$
and using~(\ref{eq:app-applied-drift-energy}) proves
\begin{equation}
    \overline M_N
    \leq3M_{2,0}+\frac{64}{5}K_0T
      +\frac{332}{15}LdT^2+3dT
      +\frac{88}{3}T^2\bar e_N.
    \label{eq:app-trajectory-bound}
\end{equation}

It remains to close~(\ref{eq:app-average-error-moment}) and
~(\ref{eq:app-trajectory-bound}) without assuming the desired error bound.
Denote the right-hand side of~(\ref{eq:app-trajectory-bound}) by $S_N$.
Then
\begin{equation}
    \bar e_N
    \leq C_pb^{-2(p-1)/p}[BS_N^{p/2}+G]^{2/p}.
    \label{eq:app-error-self-consistency}
\end{equation}
At an arbitrary error level $\delta^2$, the corresponding value of $S_N$
is exactly the deterministic radius squared from
Theorem~\ref{thm:full-robust}, which we recall here:
\[
    R_{T,\delta}^2
    =3M_{2,0}+\frac{64}{5}K_0T
      +\frac{332}{15}LdT^2+3dT+\frac{88}{3}T^2\delta^2.
\]
The prescribed fixed batch size is
\[
    b=b(T,\delta)
    :=\max\left\{1,\left\lceil
      \left(
      \frac{2C_p[B R_{T,\delta}^{p}+G]^{2/p}}{\delta^2}
      \right)^{p/[2(p-1)]}\right\rceil\right\}.
\]
This choice ensures
\begin{equation}
    C_pb^{-2(p-1)/p}
    [B(R_{T,\delta}^2)^{p/2}+G]^{2/p}
    \leq\frac{\delta^2}{2}.
    \label{eq:app-batch-absorption}
\end{equation}

Suppose for contradiction that $\bar e_N\geq\delta^2>0$, and set
$\vartheta=\delta^2/\bar e_N\in(0,1]$. Comparing $S_N$ with
$R_{T,\delta}^2$ gives
$R_{T,\delta}^2\geq\vartheta S_N$. By the monotonicity and concavity in
squared distance established in
Lemma~\ref{lem:robust-noise-concavity},
\[
    [B(R_{T,\delta}^2)^{p/2}+G]^{2/p}
    \geq[B(\vartheta S_N)^{p/2}+G]^{2/p}
    \geq\vartheta[BS_N^{p/2}+G]^{2/p}.
\]
Combining this inequality with~(\ref{eq:app-error-self-consistency}) and
~(\ref{eq:app-batch-absorption}) yields
\[
    \bar e_N
    \leq\vartheta^{-1}\frac{\delta^2}{2}
    =\frac{\bar e_N}{2},
\]
a contradiction. Therefore $\bar e_N<\delta^2$. The factor two in the
batch choice avoids a boundary case in this non-strict concavity argument.

Substituting $\bar e_N<\delta^2$ into
~(\ref{eq:app-mse-to-fi}) and using convexity of relative Fisher
information under mixtures gives
\begin{align}
    \mathrm{FI}(\bar\mu_{Nh}\|\pi)
    &\leq\frac1{Nh}\int_0^{Nh}\mathrm{FI}(\mu_t\|\pi)\,\mathrm{d}t
    \notag\\
    &\leq\frac{2K_0}{Nh}+8L^2dh+\frac{64}{3}L^3dh^2
      +\left(4+\frac{16}{3}L^2h^2\right)\delta^2
    \notag\\
    &\leq\frac{2K_0}{Nh}+\frac{32}{3}L^2dh
      +\frac{49}{12}\delta^2,
    \label{eq:app-full-robust-fisher}
\end{align}
where the last line again uses $Lh\leq1/8$. This proves the finite-time
Fisher-information bound in Theorem~\ref{thm:full-robust}.

\paragraph{Removing the initial second-moment condition when $B=0$.}
In the state-independent case, the local-noise scale is the constant
$G^{1/p}$, and the batch size reduces to
\[
    b(T,\delta)
    =\max\left\{1,\left\lceil
      \left(\frac{2C_pG^{2/p}}{\delta^2}\right)^{p/[2(p-1)]}
      \right\rceil\right\}.
\]
Lemma~\ref{app:robust-estimation} gives the uniform bound
$e_k\leq C_pG^{2/p}b^{-2(p-1)/p}\leq\delta^2/2$ under the stated
batch choice. Thus, whenever the initial second moment is finite,
Proposition~\ref{prop:robust-mse-to-fi} directly
proves~(\ref{eq:app-full-robust-fisher}). No trajectory moment argument is
needed in this case.

We now remove the finite initial second-moment condition. Assuming only
$K_0<\infty$, condition the entire construction on
$E_R:=\{\|X_0-x_{\mathrm{ref}}\|\leq R\}$, for sufficiently large
$R$ that $p_R:=\mathbb P(E_R)>0$. The conditioned process has bounded
initial states and uses the same oracle kernels and fresh Brownian
increments. Consequently, its grid-point second moments are finite, and
the same uniform error bound holds independently of $R$. If
$f_0=\mathrm{d}\mu_0/\mathrm{d}\pi$, its initial KL divergence is
\[
    K_{0,R}=\frac1{p_R}\int_{\{\|x-x_{\mathrm{ref}}\|\leq R\}}
                \log f_0\,\mathrm{d}\mu_0-\log p_R
      \longrightarrow K_0.
\]
For each $t\leq Nh$, the conditional law
$\mu_{t,R}:=\operatorname{Law}(X_t\mid E_R)$ converges to $\mu_t$
in total variation as $p_R\uparrow1$. Lower semicontinuity of relative FI
and Fatou's lemma therefore pass the integrated FI bound to the original
process. Convexity then gives~(\ref{eq:app-full-robust-fisher}) for
$\bar\mu_{Nh}$. This argument includes exact gradients ($B=G=0$) and
does not require continuity of the oracle kernel in the state.

For $B>0$, the state-dependent error is controlled by
~(\ref{eq:app-average-error-moment}), so the known initial second-moment
bound remains an assumption of this proof. Finite KL alone does not imply
it for general non-log-concave targets.

\subsubsection{Accuracy schedule and oracle complexity}
\label{app:proof-robust-complexity}

We finish by deriving the stated iteration and stochastic-oracle
complexities. Set
\begin{equation}
    \delta^2=\frac{\varepsilon}{16},
    \qquad
    h=\min\left\{\frac1{8L},
      \frac{\varepsilon}{32L^2d}\right\},
    \qquad
    T=h\max\left\{1,
      \left\lceil\frac{6K_0}{\varepsilon h}\right\rceil
      \right\}.
    \label{eq:app-robust-accuracy-schedule}
\end{equation}
Since the method uses fixed steps of length $h$,
$N_T=T/h$ equals the maximum in~(\ref{eq:app-robust-accuracy-schedule}).
The three terms in the last line of
~(\ref{eq:app-full-robust-fisher}) are bounded, respectively, by
$\varepsilon/3$, $\varepsilon/3$, and
$49\varepsilon/192<\varepsilon/3$. Hence $\mathrm{FI}(\bar\mu_T\|\pi)\leq\varepsilon$.
Since $\max\{1,\lceil z\rceil\}\leq1+z$ for $z\geq0$,
\begin{equation}
    N_T
    \leq1+\frac{6K_0}{\varepsilon}
      \max\left\{8L,\frac{32L^2d}{\varepsilon}\right\}.
    \label{eq:app-robust-iteration-count}
\end{equation}

Abbreviate $R=R_{T,\delta}$. When $B=0$, we set $M_{2,0}=0$ in this deterministic
schedule. This convention asserts no second-moment bound, and the batch
size depends only on $G$. Therefore
\begin{equation}
    Q_T=4bN_T \leq4N_T\left[
      1+\left(
      \frac{2C_p[B R^p+G]^{2/p}}
           {\delta^2}
      \right)^{p/[2(p-1)]}
      \right].
    \label{eq:app-robust-finite-oracle-count}
\end{equation}
This is the claimed finite stochastic oracle bound.

For completeness, we make every accuracy exponent explicit. Suppose
$K_0>0$ and
\begin{equation*}
    0<\varepsilon\leq
    \min\{4Ld,\sqrt{32K_0L^2d}\}.
\end{equation*}
Then $h=\varepsilon/(32L^2d)\leq K_0/\varepsilon$, and
\begin{equation}
    T
    \leq h+\frac{6K_0}{\varepsilon}
    \leq\frac{7K_0}{\varepsilon}.
    \label{eq:app-robust-time-bound}
\end{equation}
Substituting~(\ref{eq:app-robust-time-bound}) and
$\delta^2=\varepsilon/16$ into the definition of
$R$ gives
\begin{align}
    R^2\leq \widehat R^2:=3M_{2,0}
      +\frac{16268}{15}\frac{LdK_0^2}{\varepsilon^2}
      +\frac{(5383/30)K_0^2+21dK_0}{\varepsilon}.
    \label{eq:app-robust-radius-envelope}
\end{align}
The arithmetic follows by bounding each occurrence of $T$ by
$7K_0/\varepsilon$ in $R_{T,\delta}^2$ and collecting the
$\varepsilon^{-2}$ and $\varepsilon^{-1}$ terms.
Equations~(\ref{eq:app-robust-iteration-count}),
~(\ref{eq:app-robust-finite-oracle-count}), and monotonicity of the local
$L^p$ noise scale imply
\begin{align*}
    Q_T
    =\mathcal O\!\left[
      \left(1+\frac{K_0L^2d}{\varepsilon^2}\right)
      \left\{1+
      \frac{[B\widehat R^p+G]^{1/(p-1)}}
           {\varepsilon^{p/[2(p-1)]}}
      \right\}
      \right],
\end{align*}
where the hidden numerical constant depends only on $p$.

For fixed positive $K_0,L$ and fixed remaining problem parameters,
$\widehat R\leq64K_0\sqrt{Ld}/\varepsilon$ for all
sufficiently small $\varepsilon$. Indeed, the leading coefficient
$16268/15$ in~(\ref{eq:app-robust-radius-envelope}) is smaller than
$64^2$, and all other terms have lower order in $\varepsilon^{-1}$.
Consequently, when $B>0$,
\begin{equation}
    Q_T
    =\mathcal O\!\left(
      \frac{B^{1/(p-1)}K_0^{1+p/(p-1)}
        L^{2+p/[2(p-1)]}d^{1+p/[2(p-1)]}}
      {\varepsilon^{2+3p/[2(p-1)]}}
      \right).
    \label{eq:app-robust-growing-noise-rate}
\end{equation}
Under the convention $K_0=\mathcal O(d)$ from
Section~\ref{sec:preliminaries}, this becomes
\begin{equation*}
    Q_T
    =\mathcal O\!\left(
      \frac{B^{1/(p-1)}
        L^{2+p/[2(p-1)]}d^{2+3p/[2(p-1)]}}
      {\varepsilon^{2+3p/[2(p-1)]}}
      \right),
    \qquad
    N_T
    =\mathcal O\!\left(\frac{L^2d^2}{\varepsilon^2}\right).
\end{equation*}

\paragraph{State-independent finite-$p$ noise.}
If $B=0$, the radius drops out of the batch size, so it is important to
specialize the finite bound directly rather than use
~(\ref{eq:app-robust-growing-noise-rate}). Equations
~(\ref{eq:app-robust-iteration-count}) and
~(\ref{eq:app-robust-finite-oracle-count}) give
\begin{equation}
    Q_T
    =\mathcal O\!\left(
      \frac{K_0L^2d\,G^{1/(p-1)}}
           {\varepsilon^{2+p/[2(p-1)]}}
      \right),
    \label{eq:app-robust-uniform-p-rate}
\end{equation}
up to the lower-order iteration term. Under $K_0=\mathcal O(d)$, replace
$K_0d$ in~(\ref{eq:app-robust-uniform-p-rate}) by $d^2$.

This completes the proof of Theorem~\ref{thm:full-robust}.

\subsection{Proof of Corollary~\ref{cor:robust-special-cases}}

The iteration bound follows directly from
~(\ref{eq:app-robust-iteration-count}) under $K_0=\mathcal O(d)$.

\paragraph{State-independent second-moment noise.}
Taking $p=2$ and $B=0$ in
~(\ref{eq:app-robust-uniform-p-rate}) yields
\begin{equation*}
    Q_T
    =\mathcal O\!\left(
      \frac{K_0L^2dG}{\varepsilon^3}\right)
    =\mathcal O\!\left(
      \frac{L^2d^2G}{\varepsilon^3}\right)
    \quad\text{when }K_0=\mathcal O(d).
\end{equation*}
Because $G>0$ is fixed, this term absorbs the lower-order iteration
complexity as $\varepsilon\downarrow0$.

\paragraph{BG-0 noise.}
Taking $p=2$ in~(\ref{eq:app-robust-growing-noise-rate}) gives
\begin{equation*}
    Q_T
    =\mathcal O\!\left(
      \frac{BK_0^3L^3d^2}{\varepsilon^5}\right)
    =\mathcal O\!\left(
      \frac{BL^3d^5}{\varepsilon^5}\right)
    \quad\text{when }K_0=\mathcal O(d).
\end{equation*}

\paragraph{Exact gradients.}
If $B=G=0$, the conditional moment bound in
Assumption~\ref{ass:BG-p} is zero. Thus, conditionally on each current
state, every oracle call equals the gradient at that state almost surely.
Since $b=1$, the four calls in the single robust replicate are identical,
so the construction generates the standard exact-gradient LMC iterates.
Replacing these four identical calls by one exact-gradient evaluation
leaves the iterates unchanged. Thus $Q_T=N_T$, and
~(\ref{eq:app-robust-iteration-count}) gives
\begin{equation*}
    Q_T=N_T
    =\mathcal O\!\left(
      \frac{K_0L^2d}{\varepsilon^2}\right)
    =\mathcal O\!\left(
      \frac{L^2d^2}{\varepsilon^2}\right)
    \quad\text{when }K_0=\mathcal O(d).
\end{equation*}

This proves Corollary~\ref{cor:robust-special-cases}.

\subsection{Minibatched SG-LMC under uniform variance}
\label{app:sg-lmc-comparison}

For comparison with the $p=2$, $B=0$ specialization of R-LMC, we make
the batch-size choice in the SG-LMC Fisher information bound explicit.
The next lemma specializes \citet[Theorem~15]{balasubramanian22a} to
unbiased minibatches.

\begin{lemma}[SG-LMC Fisher-information and oracle bounds]
\label{lem:sg-lmc-fi}
Suppose $V\in C^2(\R^d)$ satisfies Assumption~\ref{ass:global-smoothness},
$K_0:=\mathrm{KL}(\mu_0\|\pi)<\infty$, and the stochastic oracle satisfies,
for every $x$,
\[
    \E_\xi\mathcal G(x,\xi)=\nabla V(x),
    \qquad
    \E_\xi\|\mathcal G(x,\xi)-\nabla V(x)\|^2\leq \sigma^2.
\]
Run~(\ref{eq:common-lmc}) with $H_k=h$ and $F_k=-\frac1b\sum_{j=1}^{b}\mathcal G(X_{kh},\xi_{k,j}),$
where the $b$ oracle calls are conditionally independent at each
step, and independent of the next Brownian increment. For integers
$N,b\geq1$ and $0<h<(14L)^{-1}$, set $T=Nh$. Then $Q_T=Nb$ and
\begin{equation}
    \mathrm{FI}(\bar\mu_T\|\pi)
    \leq \overline{\mathrm{FI}}_T
    \leq \frac{2K_0}{T}+16L^2dh+\frac{8\sigma^2}{b}.
    \label{eq:app-sg-lmc-fi}
\end{equation}
Moreover, for every $\varepsilon>0$ there are choices of $h,b,T$ such
that $\mathrm{FI}(\bar\mu_T\|\pi)\leq\varepsilon$. For these choices, set
$N_T:=T/h$ and $Q_T:=bN_T$. If $K_0=\mathcal O(d)$ and $\sigma^2>0$ is
fixed, then, for sufficiently small $\varepsilon$,
\[
    N_T=\mathcal O\!\left(\frac{L^2d^2}{\varepsilon^2}\right),
    \qquad
    Q_T=\mathcal O\!\left(
        \frac{L^2d^2\sigma^2}{\varepsilon^3}\right).
\]
\end{lemma}

\begin{proof}
Fix a state $x$ and write
$\eta_j=\mathcal G(x,\xi_{k,j})-\nabla V(x)$. Independence and
$\E_\xi\eta_j=0$ make the cross terms vanish, so
\[
    \E_\xi\left\|\frac1b\sum_{j=1}^b\eta_j\right\|^2
    =\frac1{b^2}\sum_{j=1}^b\E_\xi\|\eta_j\|^2
    \leq\frac{\sigma^2}{b}.
\]
Thus the batch average is unbiased and has variance at most
$\sigma^2/b$ at every state. Its expected gradient is $\nabla V$,
which is $L$-Lipschitz by Assumption~\ref{ass:global-smoothness}.
Assumption~14 of~\citet{balasubramanian22a} therefore holds with
their $\widehat V=V$, $\widehat L=L$, squared bias zero, and variance
$\sigma^2/b$. Their Theorem~15 gives the time-averaged FI inequality
in~(\ref{eq:app-sg-lmc-fi}); convexity of FI gives the bound for
$\bar\mu_T$. Each of the $N$ steps uses $b$ oracle calls, so $Q_T=Nb$.

To establish the complexity claim, we bound each term
of~(\ref{eq:app-sg-lmc-fi}) by $\varepsilon/3$. It suffices to require
\[
    N\geq\frac{6K_0}{\varepsilon h},
    \qquad
    h\leq\frac{\varepsilon}{48L^2d},
    \qquad
    \frac{8\sigma^2}{b}\leq\frac{\varepsilon}{3}
    \ \Longleftrightarrow\
    b\geq\frac{24\sigma^2}{\varepsilon}.
\]
The batch condition does not involve $L$ or $d$. Choose the integer
batch size and a step size satisfying the lemma's restriction by setting
\begin{equation}
    h:=\min\left\{\frac1{28L},\frac{\varepsilon}{48L^2d}\right\}, \quad b:=\min\left\{\frac1{28L},\frac{\varepsilon}{48L^2d}\right\}, \quad N_T
    :=\left\lceil\frac{6(K_0+1)}{\varepsilon h}\right\rceil, \quad T:=N_Th, \quad Q_T:=N_Tb.
    \label{eq:app-sg-lmc-schedule}
\end{equation}
These choices make $h<(14L)^{-1}$ and control the three
terms of~(\ref{eq:app-sg-lmc-fi}) separately:
\begin{align*}
    \frac{2K_0}{N_Th}\leq\frac{\varepsilon K_0}{3(K_0+1)}
      \leq\frac{\varepsilon}{3}, \quad 16L^2d h \leq\frac{\varepsilon}{3},\quad \frac{8\sigma^2}{b} \leq\frac{\varepsilon}{3}.
\end{align*}
Substituting these bounds into~(\ref{eq:app-sg-lmc-fi}) proves the FI
claim. The last line also shows why it suffices to take
$b$ proportional to $\sigma^2/\varepsilon$, without an
additional factor of $Ld$.

For $\varepsilon\leq(12/7)Ld$, the second branch of
$h$ is active. The ceiling bounds then give, step by step,
\[
    N_T
    \leq 1+\frac{288(K_0+1)L^2d}{\varepsilon^2},
    \qquad
    b
    \leq 1+\frac{24\sigma^2}{\varepsilon},
\]
and hence
\begin{equation*}
    Q_T
    =N_Tb
    \leq
    \left(1+\frac{288(K_0+1)L^2d}{\varepsilon^2}\right)
    \left(1+\frac{24\sigma^2}{\varepsilon}\right).
\end{equation*}
For fixed $\sigma^2>0$ and sufficiently small $\varepsilon$, the
leading term is
$\mathcal O((K_0+1)L^2d\sigma^2/\varepsilon^3)$.
Using $K_0=\mathcal O(d)$ gives the displayed iteration and oracle rates.
\end{proof}

The SG-LMC calculation in \citet[Appendix~B.3]{sahin2026variance} uses
$b=\lceil\sigma^2\sqrt{N_T}\rceil$. Substituting
$N_T=\mathcal O(L^2d^2/\varepsilon^2)$ into that schedule gives
$b=\mathcal O(Ld\sigma^2/\varepsilon)$ and the reported
$Q_T=\mathcal O(L^3d^3\sigma^2/\varepsilon^3)$ when
$K_0=\mathcal O(d)$ and $\sigma^2>0$ is fixed.
The accuracy-based choice~(\ref{eq:app-sg-lmc-schedule}) follows from
the same FI bound and avoids the extra $Ld$ factor in that schedule.
The SL-VRLD method of~\citet{sahin2026variance} achieves the same total
oracle order with $\gO(1)$ expected oracle calls per update under
an additional mean-square Lipschitz condition on stochastic gradients.
The selected SG-LMC schedule above uses a batch of size
$b=\gO(\sigma^2/\varepsilon)$ at every update, so its oracle
work per update grows as $\varepsilon$ decreases.

\subsection{Proof of Theorem~\ref{thm:full-adaptive}}
\label{app:proof-adaptive}

This subsection proves the exact gradient adaptive-step result. We first set
up the adaptive intervals and continuous-time interpolation and explain the
roles of the constants used in the finite horizon bound.
Subsection~\ref{app:proof-adaptive-endpoint} establishes endpoint control
under symmetric generalized smoothness, and
Subsection~\ref{app:proof-adaptive-gradient-moment} gives a gradient
second moment bound. Subsection~\ref{app:proof-adaptive-kl} proves KL
dissipation for an adapted approximate drift, while
Subsection~\ref{app:proof-adaptive-freezing} bounds the discretization error
caused by holding the gradient fixed over each adaptive interval.
Subsection~\ref{app:proof-adaptive-closure} then absorbs the resulting error
term and controls the random number of adaptive intervals. Finally,
Subsection~\ref{app:proof-adaptive-complexity} verifies the accuracy schedule
and derives the iteration and oracle complexities. No finite initial second moment
assumption is needed here, we require only
$K_0=\mathrm{KL}(\mu_0\|\pi)<\infty$.

Throughout this subsection, $K_0:=\mathrm{KL}(\mu_0\|\pi)$ and
$0<\gamma\leq1$.
For the adaptive-step method in~(\ref{eq:adaptive-stepsize}), let
$\mathcal F_k$ be the history at $\tau_k$, before the next Brownian
increment. For each interval with $\tau_k<T$, set
$\ell_k:=\min\{H_k,T-\tau_k\}$. Thus $\ell_k=H_k$ on every complete
interval, and only the last occupied interval can be shortened.
For $0\leq s\leq\ell_k$, the interpolation in
Definition~\ref{def:algorithmic-laws}, with
$F_k=-\nabla V(X_{\tau_k})$, is
\begin{equation}
\label{app:adaptive:eq:interpolation}
X_{\tau_k+s}
=X_{\tau_k}-s\nabla V(X_{\tau_k})
+\sqrt{2}\bigl(W_{\tau_k+s}-W_{\tau_k}\bigr).
\end{equation}

We recall each coefficient below when it enters the argument, together with
the term it controls.

\subsubsection{Endpoint and gradient moment estimates}
\label{app:proof-adaptive-endpoint}

The next lemma states the gradient variation bounds of
\citet[Proposition~1, items 1--2]{chen2023generalized} and establishes the
local regularity properties needed below. The polynomial and exponential cases
must be kept separate because the displacement exponent $2/(1-\gamma)$
diverges as $\gamma\uparrow1$.
The local regularity argument additionally shows that the Jacobian bound used
in Lemma~\ref{app:adaptive-fisher-gradient} holds almost everywhere without
assuming that $V$ is twice differentiable.

\begin{lemma}[Gradient variation bounds and local regularity]
\label{app:adaptive:lem:endpoint}
The gradient $\nabla V$ is locally Lipschitz.  If $0<\gamma<1$, then
\begin{equation}
\begin{aligned}
\|\nabla V(y)-\nabla V(x)\|
&\leq L_0(1+2^{\gamma^2/(1-\gamma)})\|y-x\|\\
&\quad+L_1\,2^{\gamma^2/(1-\gamma)}3^\gamma
 \|\nabla V(x)\|^\gamma\|y-x\|\\
&\quad+L_1^{1/(1-\gamma)}2^{\gamma^2/(1-\gamma)}3^\gamma
 (1-\gamma)^{\gamma/(1-\gamma)}\|y-x\|^{1/(1-\gamma)},
\end{aligned}
\label{app:adaptive:eq:endpoint}
\end{equation}
and
\begin{equation}
\begin{aligned}
\|\nabla V(y)-\nabla V(x)\|^2
&\leq3L_0^2(1+2^{\gamma^2/(1-\gamma)})^2\|y-x\|^2\\
&\quad+3^{2\gamma+1}2^{2\gamma^2/(1-\gamma)}L_1^2
 \|\nabla V(x)\|^{2\gamma}\|y-x\|^2\\
&\quad+3^{2\gamma+1}2^{2\gamma^2/(1-\gamma)}L_1^{2/(1-\gamma)}
 (1-\gamma)^{2\gamma/(1-\gamma)}\|y-x\|^{2/(1-\gamma)}.
\end{aligned}
\label{app:adaptive:eq:endpoint-square}
\end{equation}
If $\gamma=1$, then
\begin{equation}
\label{app:adaptive:eq:endpoint-one}
\|\nabla V(y)-\nabla V(x)\|
\leq
\|y-x\|\bigl(L_0+L_1\|\nabla V(x)\|\bigr)
e^{L_1\|y-x\|}.
\end{equation}
Moreover,
$\|D(\nabla V)(x)\|_{\mathrm{op}}
\leq L_0+L_1\|\nabla V(x)\|^\gamma$ almost everywhere.
\end{lemma}

\begin{proof}
The endpoint inequalities~(\ref{app:adaptive:eq:endpoint}) and
~(\ref{app:adaptive:eq:endpoint-one}) are given by
\citet[Proposition~1, items 1--2]{chen2023generalized}, applied with
$f=V$ and $\alpha=\gamma$. Their statement takes $L_0,L_1>0$. If either
constant is zero, Assumption~\ref{ass:generalized-smoothness} remains valid
with $L_0+\delta$ and $L_1+\delta$. Applying the cited result and letting
$\delta\downarrow0$ gives the displayed bounds. Squaring
~(\ref{app:adaptive:eq:endpoint}) and using
$(a+b+c)^2\leq3(a^2+b^2+c^2)$ proves
~(\ref{app:adaptive:eq:endpoint-square}), since
$2+2\gamma/(1-\gamma)=2/(1-\gamma)$.

It remains to establish the local regularity claims. Since
$V\in C^1(\mathbb R^d)$, its gradient is continuous. Fix a radius $R>0$.
Compactness gives
\[
\sup_{\|z\|\leq R}\|\nabla V(z)\|<\infty.
\]
If $\|x\|,\|y\|\leq R$, the line segment $[x,y]$ lies in the same ball.
Assumption~\ref{ass:generalized-smoothness} therefore yields
\[
\|\nabla V(y)-\nabla V(x)\|
\leq
\left[L_0+L_1
\left(\sup_{\|z\|\leq R}\|\nabla V(z)\|\right)^\gamma\right]
\|y-x\|.
\]
Thus, $\nabla V$ is Lipschitz on every compact ball.  Rademacher's theorem
implies differentiability almost everywhere.  At every differentiability
point, apply Assumption~\ref{ass:generalized-smoothness} with $y=x+tu$,
divide by $|t|$, and let $t\to0$.  Continuity of $\nabla V$ along the
shrinking segment gives, for every unit vector $u$,
\[
\|D(\nabla V)(x)u\|
\leq L_0+L_1\|\nabla V(x)\|^\gamma.
\]
Taking the supremum over unit vectors proves the almost everywhere Jacobian
bound.
\end{proof}

\subsubsection{Gradient second moment bound}
\label{app:proof-adaptive-gradient-moment}

The endpoint estimates introduce powers of $\|\nabla V\|$ into the
discretization error. The following lemma bounds the second moment of the
gradient using relative Fisher information, extending the corresponding
globally smooth bounds of \citet[Lemma~20]{chewi2025analysis} and
\citet[Lemma~24]{balasubramanian22a}.

\begin{lemma}[Gradient second moment bound]
\label{app:adaptive-fisher-gradient}
For every probability measure $\mu$ with finite relative Fisher information,
\begin{equation}
\label{app:adaptive:eq:fisher-gradient}
\int_{\mathbb R^d}\|\nabla V(x)\|^2\,\mu(\mathrm dx)
\leq
2\,\mathrm{FI}(\mu\|\pi)
+4dL_0
+(2-\gamma)\gamma^{\gamma/(2-\gamma)}
 (2dL_1)^{2/(2-\gamma)}.
\end{equation}
If $L_1=0$, the sharper globally smooth bound
\begin{equation}
\label{app:adaptive:eq:fisher-gradient-global}
\int_{\mathbb R^d}\|\nabla V(x)\|^2\,\mu(\mathrm dx)
\leq \mathrm{FI}(\mu\|\pi)+2dL_0
\end{equation}
holds, recovering \citet[Lemma~24]{balasubramanian22a}.
\end{lemma}

\begin{proof}
First suppose that $\mathrm d\mu=q^2\,\mathrm d\pi$, where
$q\in C_c^\infty(\mathbb R^d)$ and $\int q^2\,\mathrm d\pi=1$.  Weak
integration by parts is valid because $\nabla V$ is locally Lipschitz and
$\nabla\pi(x)=-\pi(x)\nabla V(x)$.  Let
\[
J:=\int q^2\|\nabla V\|^2\,\mathrm d\pi.
\]
Then
\[
J
=\int q^2\operatorname{div}(\nabla V)\,\mathrm d\pi
+2\int q\langle\nabla q,\nabla V\rangle\,\mathrm d\pi.
\]
The Jacobian bound in Lemma~\ref{app:adaptive:lem:endpoint} implies
$\operatorname{div}(\nabla V)
\leq dL_0+dL_1\|\nabla V\|^\gamma$ almost everywhere.  Jensen's inequality
and $2ab\leq a^2+b^2$ therefore give
\[
J\leq dL_0+dL_1J^{\gamma/2}
+\frac12\,\mathrm{FI}(\mu\|\pi)+\frac12J.
\]
When $L_1=0$, rearranging this inequality directly gives
~(\ref{app:adaptive:eq:fisher-gradient-global}).
For $L_1>0$, rearranging the same inequality gives
\[
J\leq
2dL_0+2dL_1J^{\gamma/2}+\mathrm{FI}(\mu\|\pi).
\]
Young's inequality states that
$uv\leq u^r/r+v^{r'}/r'$ for $u,v\geq0$ and $r,r'>1$, satisfying $1/r+1/r'=1$. Apply it with $r=2/\gamma$,
$r'=2/(2-\gamma)$,
$u=(J/\gamma)^{\gamma/2}$, and
$v=2dL_1\gamma^{\gamma/2}$ to obtain
\[
2dL_1J^{\gamma/2}
\leq
\frac{J}{2}
+\frac{2-\gamma}{2}\gamma^{\gamma/(2-\gamma)}
 (2dL_1)^{2/(2-\gamma)}.
\]
Substituting this bound and moving $J/2$ to the left gives
\[
J\leq
2\,\mathrm{FI}(\mu\|\pi)
+4dL_0
+(2-\gamma)\gamma^{\gamma/(2-\gamma)}
 (2dL_1)^{2/(2-\gamma)},
\]
which is~(\ref{app:adaptive:eq:fisher-gradient}).

For a general $\mu$ of finite Fisher information,
$q=\sqrt{\mathrm d\mu/\mathrm d\pi}$ belongs to the weighted Sobolev space
with energy
$\int\|\nabla q\|^2\,\mathrm d\pi
=\mathrm{FI}(\mu\|\pi)/4$.  Compact cutoffs followed by local
mollification approximate $q$ in this norm because the positive continuous
density $\pi(x)$ is bounded above and away from zero on compact sets.
Normalize the approximants in $L^2(\pi)$.  Their squared densities converge
in total variation, their energies converge, and the integral of the
nonnegative continuous function $\|\nabla V\|^2$ is lower semicontinuous.
Passing to the limit proves the result.
\end{proof}

\subsubsection{Exact continuation and KL dissipation}
\label{app:proof-adaptive-kl}

The proof first analyzes a process truncated after $n$ adaptive steps and then
lets $n\to\infty$. The next lemma ensures that the truncated process
can be continued by the exact Langevin diffusion without imposing a global
growth condition on $\nabla V$.

\begin{lemma}[Conservative exact continuation]
\label{app:adaptive-continuation}
For every initial point in $\mathbb R^d$, the diffusion
\[
\mathrm dY_t=-\nabla V(Y_t)\,\mathrm dt+\sqrt2\,\mathrm dW_t
\]
has a unique strong solution for all $t\geq0$, and $\pi$ is invariant.
It can continue the interpolation of Definition~\ref{def:algorithmic-laws}
after any fixed finite number of adaptive intervals.
\end{lemma}

\begin{proof}
Apply \citet{rockner2016nonsymmetric} with density $\rho=e^{-V}$ and
their additional drift vector field identically zero.  Since $V\in C^1$,
$\sqrt\rho=e^{-V/2}\in H^{1,2}_{\mathrm{loc}}$ and
$\|\nabla\rho\|/\rho=\|\nabla V\|$ is locally bounded, so their
hypotheses (H1)--(H3) hold.  Moreover, $\int\rho\,\mathrm dx=Z<\infty$
gives (H4), conservativity, by their Remark~2.8.
Their Theorems~3.6 and~4.1 therefore give a unique strong solution from
every initial point, with infinite lifetime, for
\[
\mathrm d\widetilde Y_t
=-\tfrac12\nabla V(\widetilde Y_t)\,\mathrm dt
 +\mathrm d\widetilde W_t.
\]
Setting $Y_t=\widetilde Y_{2t}$ gives the normalization in the statement.
The zero additional drift makes their semigroup symmetric, so their
identity~(2.10) gives invariance of $\pi$.  This invariance is preserved
by the constant time rescaling.

For the continuation, Definition~\ref{def:algorithmic-laws} selects both
the finite force $F_k$ and the finite positive length $H_k$ using only
information available at $\tau_k$.  Hence each $\tau_n$ is a finite
stopping time and $X_{\tau_n}\in\mathbb R^d$ almost surely for finite $n$.
By the strong Markov property of Brownian motion,
$W_{\tau_n+t}-W_{\tau_n}$ is a Brownian motion independent of that history.
Starting the exact strong solution at $X_{\tau_n}$ with these increments
therefore gives the required continuation.
\end{proof}

The next lemma is the analytical bridge from drift approximation to Fisher
information. Its one interval mechanism is the frozen drift calculation of
\citet[Lemma~23]{balasubramanian22a}. The formulation below permits a
progressively selected drift and is therefore suitable for adaptive step
sizes.

\begin{lemma}[Adapted-drift KL-dissipation inequality]
\label{app:adaptive-kl}
Let a conservative continuous process satisfy
\[
\mathrm dX_t=F_t\,\mathrm dt+\sqrt2\,\mathrm dW_t,
\]
where the force $F_t$ is progressive and locally square integrable. Write
$\mu_t:=\operatorname{Law}(X_t)$ and suppose
$K_0<\infty$ and
\[
\mathcal A_T
:=\E\int_0^T\|\nabla V(X_t)+F_t\|^2\,\mathrm dt<\infty.
\]
Then
\begin{equation}
\label{app:adaptive:eq:kl-dissipation}
\mathrm{KL}(\mu_T\|\pi)
+\frac34\int_0^T\mathrm{FI}(\mu_t\|\pi)\,\mathrm dt
\leq K_0+\mathcal A_T.
\end{equation}
\end{lemma}

\begin{proof}
Let $P$ be the law of $(X_t)_{0\leq t\leq T}$ on
$C([0,T],\mathbb R^d)$, and let $R$ be the law of the exact Langevin
diffusion started from $\pi$. Lemma~\ref{app:adaptive-continuation}
provides this conservative reference process. The symmetry used in its
proof makes $R$ reversible. Local Lipschitzness of $\nabla V$ gives
uniqueness of its martingale problem.

\paragraph{Finite relative entropy of the trajectory.}
Set $e_t:=F_t+\nabla V(X_t)$ and
$\sigma_n:=\inf\{t\in[0,T]:\int_0^t\|e_s\|^2\,\mathrm ds\geq n\}\wedge T$,
with $\inf\varnothing=\infty$. Follow $X$ until $\sigma_n$, then continue
with the exact Langevin diffusion using the same Brownian motion.
Denote the resulting trajectory law by $P_n$. On the original probability
space, the exponential
\[
Z_n=\exp\left(
-\frac1{\sqrt2}\int_0^{\sigma_n}\langle e_t,\mathrm dW_t\rangle
-\frac14\int_0^{\sigma_n}\|e_t\|^2\,\mathrm dt
\right)
\]
is the density of a new probability measure relative to the original one.
Novikov's criterion applies because the stochastic integral in its exponent
has quadratic variation at most $n/2$.
Its conditional expectation given the initial information is one, so the
initial law remains $\mu_0$. Girsanov's theorem turns the continued process
into exact Langevin dynamics under this new measure.
Writing $R_{\mu_0}$ for its law, data processing and the expectation of
$-\log Z_n$ give
\[
\mathrm{KL}(P_n\|R_{\mu_0})
\leq\frac14\E\int_0^{\sigma_n}\|e_t\|^2\,\mathrm dt
\leq\frac14\mathcal A_T.
\]
Since $P_n$ has initial law $\mu_0$ and
$\mathrm dR_{\mu_0}/\mathrm dR
=(\mathrm d\mu_0/\mathrm d\pi)(X_0)$, the relative entropy chain rule adds
$K_0$ when the reference is changed to $R$. The continued and original
paths agree on $\{\int_0^T\|e_t\|^2\,\mathrm dt<n\}$, whose probability
tends to one. Thus $P_n\to P$ in total variation, and lower semicontinuity
gives
\begin{equation}
\label{app:adaptive:eq:path-entropy}
\mathrm{KL}(P\|R)\leq K_0+\frac14\mathcal A_T<\infty.
\end{equation}

\paragraph{A Markov law with the same marginals.}
Minimize $\mathrm{KL}(Q\|R)$ over trajectory laws satisfying
$Q_t=\mu_t$ for every $t\in[0,T]$.
Equation~(\ref{app:adaptive:eq:path-entropy}) supplies a feasible law
with finite relative entropy. By
\citet[Proposition~4.3 and Theorem~4.5]{baradat2020minimizing},
the minimizer exists and is Markov. Hence
\begin{equation}
\label{app:adaptive:eq:markov-entropy}
Q_t=\mu_t,\qquad
\mathrm{KL}(Q\|R)\leq K_0+\frac14\mathcal A_T.
\end{equation}
This auxiliary law preserves all quantities in the lemma's conclusion.
The original process need not be Markov.

\paragraph{Time reversal and Fisher information.}
Write the forward drift of $Q$ as $-\nabla V(x)+u_t(x)$ and the
drift of its time reversal at time $T-t$ as $-\nabla V(x)+v_t(x)$.
The finite relative entropy theorem of
\citet[Theorem~2.4]{leonard2012girsanov}, applied in both directions
with covariance $2I_d$, gives
\begin{equation}
\label{app:adaptive:eq:forward-backward-entropy}
\begin{aligned}
\mathrm{KL}(Q\|R)
&=K_0+\frac14\int_0^T\int\|u_t\|^2\,\mathrm d\mu_t\,\mathrm dt\\
&=\mathrm{KL}(\mu_T\|\pi)
 +\frac14\int_0^T\int\|v_t\|^2\,\mathrm d\mu_t\,\mathrm dt.
\end{aligned}
\end{equation}
The equality condition in that theorem follows from uniqueness of the
reference martingale problem. Reversing both path laws preserves relative
entropy, and reversibility makes the reversed reference equal to $R$.

We apply \citet[Theorem~1.16]{cattiaux2023time} to $Q$ and $R$.
Its reference hypotheses hold because the forward and reversed reference
drifts are $-\nabla V$, which is locally Lipschitz, the covariance is
$2I_d$, and every reference marginal has density $\pi$.
The theorem gives, for $\mathrm dt\,\mu_t(\mathrm dx)$ almost every $(t,x)$,
\begin{equation}
\label{app:adaptive:eq:relative-score}
u_t+v_t=2\nabla\log\frac{\mathrm d\mu_t}{\mathrm d\pi},
\qquad
\int_0^T\mathrm{FI}(\mu_t\|\pi)\,\mathrm dt<\infty,
\end{equation}
where the gradient is understood weakly.
Subtracting the two identities in
~(\ref{app:adaptive:eq:forward-backward-entropy}) and substituting
~(\ref{app:adaptive:eq:relative-score}) yields
\[
\begin{aligned}
\mathrm{KL}(\mu_T\|\pi)-K_0
={}&-\int_0^T\mathrm{FI}(\mu_t\|\pi)\,\mathrm dt\\
&+\int_0^T\int
\left\langle u_t,\nabla\log\frac{\mathrm d\mu_t}{\mathrm d\pi}\right\rangle
\,\mathrm d\mu_t\,\mathrm dt.
\end{aligned}
\]
Young's inequality
$\langle u,s\rangle\leq\|u\|^2+\|s\|^2/4$
bounds the last term. Moreover,
~(\ref{app:adaptive:eq:markov-entropy}) and the first identity in
~(\ref{app:adaptive:eq:forward-backward-entropy}) give
$\int_0^T\int\|u_t\|^2\,\mathrm d\mu_t\,\mathrm dt\leq\mathcal A_T$.
Together these prove~(\ref{app:adaptive:eq:kl-dissipation}).
\end{proof}

\subsubsection{Gaussian estimates and a reusable approximate-drift lemma}
\label{app:proof-adaptive-gaussian}

Both the exact gradient and robust adaptive-step analyses require moments of a
Gaussian displacement.  We collect the three estimates used below.

\begin{lemma}[Gaussian moment and exponential bounds]
\label{app:adaptive:lem:gaussian}
If $Z\sim\mathcal N(0,I_d)$, then, for $a\geq2$ and $\lambda\geq0$,
\begin{equation}
\label{app:adaptive:eq:gaussian}
\E\|Z\|^a\leq(d+a)^{a/2},
\qquad
\E e^{\lambda\|Z\|}\leq e^{\lambda\sqrt d+\lambda^2/2},
\qquad
\E\!\left[\|Z\|^2e^{\lambda\|Z\|}\right]
\leq(d+1)e^{\lambda\sqrt d+\lambda^2}.
\end{equation}
\end{lemma}

\begin{proof}
Let $R:=\|Z\|$, whose density on $[0,\infty)$ is proportional to
$r^{d-1}e^{-r^2/2}$. Radial integration by parts gives
\[
 \E R^a=(d+a-2)\E R^{a-2}.
\]
For $a>2$, H\"older's inequality with conjugate exponents
$a/(a-2)$ and $a/2$ gives
\[
 \E R^{a-2}\leq(\E R^a)^{(a-2)/a}.
\]
Combining the two displays yields
$\E R^a\leq(d+a-2)^{a/2}\leq(d+a)^{a/2}$. For $a=2$, the same claimed
bound follows from $\E R^2=d$. For $d>1$, let
$M(\lambda):=\E e^{\lambda\|Z\|}$. A second radial
integration by parts gives
\[
M'(\lambda)
=\lambda M(\lambda)+(d-1)\E[\|Z\|^{-1}e^{\lambda\|Z\|}].
\]
The covariance of an increasing and a decreasing function of the same real
random variable is nonpositive, and therefore
\[
\frac{M'(\lambda)}{M(\lambda)}
\leq\lambda+(d-1)\E\|Z\|^{-1}
=\lambda+\E\|Z\|
\leq\lambda+\sqrt d.
\]
For $d=1$, the integration by parts boundary term is $\sqrt{2/\pi}$, which
gives the same claimed bound with $\sqrt d=1$.  Integration from zero to
$\lambda$ proves the exponential-moment inequality.  Finally,
Cauchy--Schwarz, $\E\|Z\|^4=d(d+2)\leq(d+1)^2$, and the exponential bound at
$2\lambda$ prove the last inequality.
\end{proof}

The next lemma is not required to complete the exact gradient analysis. We
state it here because it bounds the discretization error under generalized
smoothness for any approximate drift. This result is used directly in the
proof of the robust adaptive method. Its coefficients distinguish the
constant terms from those involving the population gradient or the
approximate drift. The frozen gradient bound under global smoothness in the
proof of \citet[Theorem~4]{balasubramanian22a} provides the basis for this
result. The lemma extends that argument to generalized smoothness and a
general approximate drift.

Recall $\omega:=L_1/(1+L_0)$, which is the coefficient of the force term
in the adaptive step size. In the bound below, $\widetilde{A}_\gamma$ is
the constant term, $\widetilde{U}_\gamma$ multiplies the squared
population gradient, and $\widetilde{V}_\gamma$ multiplies the squared
approximate drift. Their formulas depend on the smoothness regime.
If $L_1=0$, they are
\[
\widetilde{A}_\gamma(h)=L_0^2dh,\qquad
\widetilde{U}_\gamma(h)=0,\qquad
\widetilde{V}_\gamma(h)=\tfrac13L_0^2h^2.
\]
If $L_1>0$ and $0<\gamma<1$, the coefficients are
\begin{align*}
\widetilde{A}_\gamma(h)
&=3\left[L_0^2(1+2^{\gamma^2/(1-\gamma)})^2
       +3^{2\gamma}2^{2\gamma^2/(1-\gamma)}L_1^2\right]dh\\
&\quad+3^{2\gamma+1}2^{(2\gamma^2+\gamma+2)/(1-\gamma)}
       (1-\gamma)^{2\gamma/(1-\gamma)}L_1^{2/(1-\gamma)}
       \left[\left(d+\frac{2}{1-\gamma}\right)h\right]^{1/(1-\gamma)},\\
\widetilde{U}_\gamma(h)
&=3^{2\gamma}2^{2\gamma^2/(1-\gamma)}
       \left[\gamma(1+L_0)^2h^2+3L_1^2dh\right],\\
\widetilde{V}_\gamma(h)
&=\left[L_0^2(1+2^{\gamma^2/(1-\gamma)})^2
 +(1-\gamma)3^{2\gamma}2^{2\gamma^2/(1-\gamma)}(1+L_0)^2\right]h^2\\
&\quad+3^{2\gamma+1}2^{(2\gamma^2+\gamma+1)/(1-\gamma)}
       (1-\gamma)^{2\gamma/(1-\gamma)}[(1+L_0)h]^{2/(1-\gamma)}.
\end{align*}
For $L_1>0$ and $\gamma=1$, the Gaussian exponential bound introduces
\[
E_1^{\mathrm{ex}}(h)
:=\exp\{2(1+L_0)h+2\sqrt2L_1\sqrt{dh}+8L_1^2h\},
\]
and the corresponding coefficients are
\begin{align*}
\widetilde{A}_1(h)&=4E_1^{\mathrm{ex}}(h)L_0^2(d+1)h,\\
\widetilde{U}_1(h)&=E_1^{\mathrm{ex}}(h)
       \{\tfrac43(1+L_0)^2h^2+4L_1^2(d+1)h\},\\
\widetilde{V}_1(h)&=\tfrac43E_1^{\mathrm{ex}}(h)L_0^2h^2.
\end{align*}

\begin{lemma}[Adaptive-step discretization bound for approximate drifts]
\label{app:adaptive-power-freezing}
Let $\mathcal F$ be a sigma-field and let $X,F$ be
$\mathcal F$-measurable $\mathbb R^d$-valued random vectors.  Let $H,\ell$
be nonnegative $\mathcal F$-measurable random variables satisfying
\[
0\leq\ell\leq H
\leq\frac{h}{1+\omega\|F\|^\gamma}.
\]
Let $(W_s)_{s\geq0}$ be a standard $d$-dimensional Brownian motion
independent of $\mathcal F$, and define
\[
\mathcal D_s:=sF+\sqrt2W_s.
\]
Use the coefficients $\widetilde{A}_\gamma,\widetilde{U}_\gamma,
\widetilde{V}_\gamma$ defined above for the applicable regime.
Then, almost surely,
\begin{equation*}
\E\!\left[
 \int_0^\ell
 \|\nabla V(X+\mathcal D_s)-\nabla V(X)\|^2\,\mathrm ds
 \;\middle|\;\mathcal F\right]
\leq
\ell\left\{
 \widetilde{A}_\gamma(h)
 +\widetilde{U}_\gamma(h)\|\nabla V(X)\|^2
 +\widetilde{V}_\gamma(h)\|F\|^2
\right\}.
\end{equation*}
Each coefficient is nondecreasing in $h$ and, for $0<h\leq1$, is at most
$h$ times its value at one.
\end{lemma}

\begin{proof}
Condition on $\mathcal F$ and write
$g:=\|\nabla V(X)\|$ and $w:=\|F\|$. All expectations in this proof
are conditional on $\mathcal F$. Direct Gaussian calculation gives
\begin{equation}
\label{app:adaptive:eq:approx-second moment}
\E[\|\mathcal D_s\|^2\mid\mathcal F]=s^2w^2+2ds.
\end{equation}

Suppose first that $L_1>0$ and $0<\gamma<1$.  Substituting
~(\ref{app:adaptive:eq:approx-second moment}) into the first two terms of
~(\ref{app:adaptive:eq:endpoint-square}), integrating, and using the
step-size bound gives
\begin{align*}
&3L_0^2(1+2^{\gamma^2/(1-\gamma)})^2
\int_0^\ell\E\|\mathcal D_s\|^2\,\mathrm ds \leq\ell L_0^2(1+2^{\gamma^2/(1-\gamma)})^2
 (h^2w^2+3dh),\\
&3^{2\gamma+1}2^{2\gamma^2/(1-\gamma)}L_1^2g^{2\gamma}
\int_0^\ell\E\|\mathcal D_s\|^2\,\mathrm ds \leq
\ell\,2^{2\gamma^2/(1-\gamma)}3^{2\gamma}(1+L_0)^2h^2
\{\gamma g^2+(1-\gamma)w^2\}\\
&\qquad\qquad\qquad\qquad\quad\qquad\qquad\qquad\qquad+3^{2\gamma+1}2^{2\gamma^2/(1-\gamma)}
 \ell L_1^2dh(1+g^2).
\end{align*}
For the first term on the right of the second inequality, we used
\[
\ell^2g^{2\gamma}w^2
\leq h^2\omega^{-2}g^{2\gamma}w^{2-2\gamma}
\leq h^2\omega^{-2}\{\gamma g^2+(1-\gamma)w^2\},
\]
where the last inequality is weighted Young's inequality.  For the Brownian
term, we used $g^{2\gamma}\leq1+g^2$.

For every $a\geq2$, Lemma~\ref{app:adaptive:lem:gaussian} gives
\begin{equation*}
\E\|\mathcal D_s\|^a
\leq2^{a-1}s^aw^a
+2^{3a/2-1}s^{a/2}(d+a)^{a/2}.
\end{equation*}
Set $a=2/(1-\gamma)$, integrate over $[0,\ell]$, and bound the integration
denominators below by one. The adaptive step bound gives
\[
\ell^{2/(1-\gamma)}w^{2/(1-\gamma)}
\leq (h/\omega)^{2/(1-\gamma)}w^2,
\qquad
\ell^{1/(1-\gamma)}\leq h^{1/(1-\gamma)}.
\]
These are exactly the final terms in the coefficients
$\widetilde{V}_\gamma$ and $\widetilde{A}_\gamma$ defined above, proving the claim
when $0<\gamma<1$.

Now suppose that $L_1>0$ and $\gamma=1$.  Since
$sw\leq h/\omega=(1+L_0)h/L_1$, Lemma~\ref{app:adaptive:lem:gaussian}
implies
\[
\E\!\left[
 \|\mathcal D_s\|^2e^{2L_1\|\mathcal D_s\|}
 \;\middle|\;\mathcal F\right]
\leq E_1^{\mathrm{ex}}(h)\{2s^2w^2+4s(d+1)\}.
\]
Square~(\ref{app:adaptive:eq:endpoint-one}), use
$(L_0+L_1g)^2\leq2(L_0^2+L_1^2g^2)$, and integrate to obtain
\[
\E\!\left[
 \int_0^\ell
 \|\nabla V(X+\mathcal D_s)-\nabla V(X)\|^2\,\mathrm ds
 \;\middle|\;\mathcal F\right]
\leq
\ell E_1^{\mathrm{ex}}(h)(L_0^2+L_1^2g^2)
\left\{\frac43\ell^2w^2+4(d+1)\ell\right\}.
\]
For terms multiplied by $L_0^2$, use $\ell\leq h$. For the mixed term,
use $\ell^2w^2\leq h^2\omega^{-2}$. This gives the coefficients
$\widetilde{A}_1,\widetilde{U}_1,\widetilde{V}_1$ defined above.

Finally, if $L_1=0$, global $L_0$-smoothness and
~(\ref{app:adaptive:eq:approx-second moment}) give directly
\[
\E\!\left[
 \int_0^\ell
 \|\nabla V(X+\mathcal D_s)-\nabla V(X)\|^2\,\mathrm ds
 \;\middle|\;\mathcal F\right]
\leq
\ell\left\{L_0^2dh+\frac13L_0^2h^2w^2\right\}.
\]
This gives the coefficients defined above for $L_1=0$. The stated
monotonicity follows directly from the formulas. At $\gamma=1$,
$E_1^{\mathrm{ex}}(h)$ is increasing.
\end{proof}

\subsubsection{Discretization error for adaptive LMC with exact gradients}
\label{app:proof-adaptive-freezing}

We now exploit that the drift and the step-size rule use the same exact
gradient.  This produces sharper coefficients than the reusable
approximate drift lemma and is essential for recovering the globally smooth
rate when $L_1=0$. We use the coefficients $A_\gamma,B_\gamma,Q_\gamma$
defined in Subsection~\ref{app:notation-constants}.
We restate their formulas where they are used in the proof.

\begin{lemma}[Discretization error for adaptive LMC with exact gradients]
\label{app:adaptive:lem:exact-freezing}
Fix $T>0$ and consider a path on $[0,T]$ formed by finitely many intervals of
the adaptive step method~(\ref{eq:adaptive-stepsize}), where the decision to
take an interval is made using its starting history. If these intervals end
at a time $\sigma<T$, continue the path by the exact Langevin diffusion on
$[\sigma,T]$. Define the piecewise force, with the sign convention
of~(\ref{eq:common-lmc}), by
\begin{equation*}
F_t^{\mathrm{fr}}
:=
\begin{cases}
F_k=-\nabla V(X_{\tau_k}),
  &t\in[\tau_k,\tau_k+\ell_k),\\
-\nabla V(X_t),
  &t\in[\sigma,T]\text{ on the exact continuation}.
\end{cases}
\end{equation*}
A final adaptive interval cut off at $T$ uses the first line until $T$.
Define
\begin{equation}
  e_T :=\frac1T\E\int_0^T\|\nabla V(X_t)+F_t^{\mathrm{fr}}\|^2\,\mathrm dt \quad \text{and} \quad m_T:=\frac1T\E\int_0^T\|F_t^{\mathrm{fr}}\|^2\,\mathrm dt.
\end{equation}
Then
\begin{equation}
\label{app:adaptive:eq:exact-freezing}
e_T
\leq A_\gamma(h)+B_\gamma(h)m_T+Q_\gamma(h)m_T^{\gamma/2}.
\end{equation}
\end{lemma}

\begin{proof}
On an adaptive interval, condition on the complete information at time
$\tau_k$, and let $x$ denote the resulting value of $X_{\tau_k}$.  Put
$\ell:=\ell_k$ and $g:=\|\nabla V(x)\|$. All expectations until the
summation over intervals below are conditional on this history. For every fixed
$0\leq s\leq\ell$, equation~(\ref{app:adaptive:eq:interpolation}) gives
\begin{equation}
\label{app:adaptive:eq:exact-displacement}
X_{\tau_k+s}-X_{\tau_k}
\stackrel{\mathrm d}{=}
-s\nabla V(x)+\sqrt{2s}Z,
\qquad Z\sim\mathcal N(0,I_d).
\end{equation}
The step-size rule and final-interval truncation imply
\begin{equation}
\label{app:adaptive:eq:step-bound}
0\leq\ell\leq H_k
=\frac{h}{1+\omega g^\gamma}\leq h,
\qquad
\omega\ell g^\gamma\leq h.
\end{equation}
Consequently,
\begin{equation}
\label{app:adaptive:eq:exact-second moment}
\E\|X_{\tau_k+s}-X_{\tau_k}\|^2=s^2g^2+2ds.
\end{equation}

If $L_1=0$, global $L_0$-smoothness and
~(\ref{app:adaptive:eq:exact-second moment}) give
\begin{equation}
\E\int_0^\ell
\|\nabla V(X_{\tau_k+s})-\nabla V(X_{\tau_k})\|^2\,\mathrm ds\leq
L_0^2\left(\frac{\ell^3g^2}{3}+d\ell^2\right)
\leq
\ell\left\{dL_0^2h+\frac13L_0^2h^2g^2\right\}.
\label{app:adaptive:eq:local-smooth}
\end{equation}
Now suppose $L_1>0$ and $0<\gamma<1$.  For every $a\geq2$, the triangle
inequality, scalar convexity, and Lemma~\ref{app:adaptive:lem:gaussian} give
\begin{align}
\E\|X_{\tau_k+s}-X_{\tau_k}\|^a
&\leq2^{a-1}s^ag^a
+2^{3a/2-1}s^{a/2}(d+a)^{a/2}.
\label{app:adaptive:eq:exact-moment}
\end{align}
For each $s\in[0,\ell]$, apply~(\ref{app:adaptive:eq:endpoint-square})
from Lemma~\ref{app:adaptive:lem:endpoint} with $x=X_{\tau_k}$ and
$y=X_{\tau_k+s}$. The bound holds for every pair of points, so it applies
to these states on each realization of the path. Using
$\|\nabla V(X_{\tau_k})\|=g$, we obtain
\[
\begin{aligned}
\|\nabla V(X_{\tau_k+s})-\nabla V(X_{\tau_k})\|^2
&\leq3L_0^2(1+2^{\gamma^2/(1-\gamma)})^2
 \|X_{\tau_k+s}-X_{\tau_k}\|^2\\
&\quad+3^{2\gamma+1}2^{2\gamma^2/(1-\gamma)}L_1^2g^{2\gamma}
 \|X_{\tau_k+s}-X_{\tau_k}\|^2\\
&\quad+3^{2\gamma+1}2^{2\gamma^2/(1-\gamma)}L_1^{2/(1-\gamma)}
 (1-\gamma)^{2\gamma/(1-\gamma)}
 \|X_{\tau_k+s}-X_{\tau_k}\|^{2/(1-\gamma)}.
\end{aligned}
\]
Conditioning on the history at $\tau_k$ fixes $g$ and $\ell$, so the factors
outside the displacement norms can be taken outside the conditional
expectations. Taking these expectations and integrating over $s\in[0,\ell]$
gives the following bound. Tonelli's theorem permits interchanging the
expectation and integral because the integrands are nonnegative.
\begin{align}
&\E\int_0^\ell
 \|\nabla V(X_{\tau_k+s})-\nabla V(X_{\tau_k})\|^2\,\mathrm ds
\notag\\
&\quad\leq3L_0^2(1+2^{\gamma^2/(1-\gamma)})^2
 \int_0^\ell\E\|X_{\tau_k+s}-X_{\tau_k}\|^2\,\mathrm ds
\notag\\
&\quad+3^{2\gamma+1}2^{2\gamma^2/(1-\gamma)}L_1^2g^{2\gamma}
 \int_0^\ell\E\|X_{\tau_k+s}-X_{\tau_k}\|^2\,\mathrm ds
\notag\\
&\quad+3^{2\gamma+1}2^{2\gamma^2/(1-\gamma)}L_1^{2/(1-\gamma)}
 (1-\gamma)^{2\gamma/(1-\gamma)}
 \int_0^\ell\E\|X_{\tau_k+s}-X_{\tau_k}\|^{2/(1-\gamma)}\,\mathrm ds.
\label{app:adaptive:eq:endpoint-decomposition}
\end{align}
By~(\ref{app:adaptive:eq:exact-second moment}) and
~(\ref{app:adaptive:eq:step-bound}), the first term on the right of
~(\ref{app:adaptive:eq:endpoint-decomposition}) is at most
\begin{equation}
\label{app:adaptive:eq:first-contribution}
\ell L_0^2(1+2^{\gamma^2/(1-\gamma)})^2(h^2g^2+3dh).
\end{equation}
For the second term,
\begin{equation*}
\ell^2g^{2+2\gamma}\leq\frac{h^2}{\omega^2}g^2
\quad \text{and} \quad
\ell g^{2\gamma}\leq\frac h\omega g^\gamma,
\end{equation*}
and hence the second term is at most
\begin{equation}
\label{app:adaptive:eq:second-contribution}
\ell\,2^{2\gamma^2/(1-\gamma)}3^{2\gamma}
\left\{(1+L_0)^2h^2g^2+3L_1(1+L_0)dhg^\gamma\right\}.
\end{equation}
For the third term, set $a=2/(1-\gamma)$ in
~(\ref{app:adaptive:eq:exact-moment}) and integrate.  Bounding the two
integration denominators below by one gives
\begin{align}
&3^{2\gamma+1}2^{2\gamma^2/(1-\gamma)}L_1^{2/(1-\gamma)}
 (1-\gamma)^{2\gamma/(1-\gamma)}
\int_0^\ell\E\|X_{\tau_k+s}-X_{\tau_k}\|^{2/(1-\gamma)}\,\mathrm ds
\notag\\
&\quad\leq
\ell\,3^{2\gamma+1}2^{(2\gamma^2+\gamma+1)/(1-\gamma)}
 L_1^{2/(1-\gamma)}(1-\gamma)^{2\gamma/(1-\gamma)}(\ell g)^{2/(1-\gamma)}
\notag\\
&\qquad+
\ell\,3^{2\gamma+1}2^{(2\gamma^2+\gamma+2)/(1-\gamma)}
 L_1^{2/(1-\gamma)}(1-\gamma)^{2\gamma/(1-\gamma)}
 \left[\left(d+\frac{2}{1-\gamma}\right)\ell\right]^{1/(1-\gamma)}.
\label{app:adaptive:eq:third-before-step}
\end{align}
Since $2\gamma/(1-\gamma)=2/(1-\gamma)-2$,
~(\ref{app:adaptive:eq:step-bound}) implies
\begin{equation}
\label{app:adaptive:eq:third-step-powers}
\ell^{2/(1-\gamma)}g^{2/(1-\gamma)}
\leq(h/\omega)^{2/(1-\gamma)}g^2,
\qquad
\ell^{1/(1-\gamma)}\leq h^{1/(1-\gamma)}.
\end{equation}
Substituting~(\ref{app:adaptive:eq:third-step-powers}) into
~(\ref{app:adaptive:eq:third-before-step}) shows that the third term
is at most
\begin{equation}
\label{app:adaptive:eq:third-contribution}
\begin{aligned}
&\ell\,3^{2\gamma+1}2^{(2\gamma^2+\gamma+1)/(1-\gamma)}
 (1-\gamma)^{2\gamma/(1-\gamma)}[(1+L_0)h]^{2/(1-\gamma)}g^2\\
&\quad+
\ell\,3^{2\gamma+1}2^{(2\gamma^2+\gamma+2)/(1-\gamma)}
 L_1^{2/(1-\gamma)}(1-\gamma)^{2\gamma/(1-\gamma)}
 \left[\left(d+\frac{2}{1-\gamma}\right)h\right]^{1/(1-\gamma)}.
\end{aligned}
\end{equation}
Combining~(\ref{app:adaptive:eq:first-contribution}),
~(\ref{app:adaptive:eq:second-contribution}), and
~(\ref{app:adaptive:eq:third-contribution}) gives the coefficients
\begin{align*}
A_\gamma(h)
&=3L_0^2(1+2^{\gamma^2/(1-\gamma)})^2dh\\
&\quad+3^{2\gamma+1}2^{(2\gamma^2+\gamma+2)/(1-\gamma)}
 (1-\gamma)^{2\gamma/(1-\gamma)}L_1^{2/(1-\gamma)}
 \left[\left(d+\frac{2}{1-\gamma}\right)h\right]^{1/(1-\gamma)},\\
B_\gamma(h)
&=\left[L_0^2(1+2^{\gamma^2/(1-\gamma)})^2
 +3^{2\gamma}2^{2\gamma^2/(1-\gamma)}(1+L_0)^2\right]h^2\\
&\quad+3^{2\gamma+1}2^{(2\gamma^2+\gamma+1)/(1-\gamma)}
 (1-\gamma)^{2\gamma/(1-\gamma)}[(1+L_0)h]^{2/(1-\gamma)},\\
Q_\gamma(h)&=3^{2\gamma+1}2^{2\gamma^2/(1-\gamma)}L_1(1+L_0)dh,
\end{align*}
and therefore
\begin{equation}
  \E\int_0^\ell\|\nabla V(X_{\tau_k+s})-\nabla V(X_{\tau_k})\|^2\,\mathrm ds\leq
\ell\left\{
A_\gamma(h)+B_\gamma(h)g^2+Q_\gamma(h)g^\gamma
\right\}.
\label{app:adaptive:eq:local-interior}
\end{equation}

It remains to treat $L_1>0$ and $\gamma=1$.  Equations
~(\ref{app:adaptive:eq:exact-displacement}) and
~(\ref{app:adaptive:eq:step-bound}) imply
\begin{equation*}
L_1\|X_{\tau_k+s}-X_{\tau_k}\|
\leq(1+L_0)h+L_1\sqrt{2h}\|Z\|,
\qquad
\|X_{\tau_k+s}-X_{\tau_k}\|^2
\leq2s^2g^2+4s\|Z\|^2.
\end{equation*}
Recall the exponential factor
\[
E_1^{\mathrm{ex}}(h)
=\exp\{2(1+L_0)h+2\sqrt2L_1\sqrt{dh}+8L_1^2h\}.
\]
Lemma~\ref{app:adaptive:lem:gaussian} then gives
\begin{equation}
\E\!\left[\|X_{\tau_k+s}-X_{\tau_k}\|^2e^{2L_1\|X_{\tau_k+s}-X_{\tau_k}\|}\right]\leq E_1^{\mathrm{ex}}(h)\{2s^2g^2+4s(d+1)\}.
\label{app:adaptive:eq:one-exponential-moment}
\end{equation}
Square~(\ref{app:adaptive:eq:endpoint-one}), use
$(L_0+L_1g)^2\leq2(L_0^2+L_1^2g^2)$, and integrate
~(\ref{app:adaptive:eq:one-exponential-moment}) to obtain
\begin{equation}
\E\int_0^\ell
\|\nabla V(X_{\tau_k+s})-\nabla V(X_{\tau_k})\|^2\,\mathrm ds\leq
E_1^{\mathrm{ex}}(h)(L_0^2+L_1^2g^2)
\left\{\frac43\ell^3g^2+4(d+1)\ell^2\right\}.
\label{app:adaptive:eq:one-before-step}
\end{equation}
The step-size bound gives
\begin{equation*}
\ell\leq h,
\qquad
\ell^2g^2\leq h^2g^2,
\qquad
L_1^2\ell^2g^4\leq(1+L_0)^2h^2g^2,
\qquad
L_1^2\ell g^2\leq L_1(1+L_0)hg.
\end{equation*}
Substituting these bounds into~(\ref{app:adaptive:eq:one-before-step}) gives
\begin{equation}
\E\int_0^\ell
\|\nabla V(X_{\tau_k+s})-\nabla V(X_{\tau_k})\|^2\,\mathrm ds\leq
\ell\{A_1(h)+B_1(h)g^2+Q_1(h)g\},
\label{app:adaptive:eq:local-one}
\end{equation}
where $A_1(h)=4E_1^{\mathrm{ex}}(h)L_0^2(d+1)h$, $B_1(h)=\tfrac43E_1^{\mathrm{ex}}(h)
       \{L_0^2+(1+L_0)^2\}h^2$, and $Q_1(h)=4E_1^{\mathrm{ex}}(h)L_1(1+L_0)(d+1)h$. 

Take expectations over the histories at interval starts and sum
~(\ref{app:adaptive:eq:local-smooth}),
~(\ref{app:adaptive:eq:local-interior}), or
~(\ref{app:adaptive:eq:local-one}) over the adaptive intervals.  The
discretization error is zero on the exact continuation.  This gives
\[
e_T
\leq A_\gamma(h)+B_\gamma(h)m_T
+Q_\gamma(h)\frac1T\E\int_0^T
\|F_t^{\mathrm{fr}}\|^\gamma\,\mathrm dt.
\]
Jensen's inequality for the probability measure
$T^{-1}\mathrm dt\,\mathrm d\mathbb P$ yields
\begin{equation}
\label{app:adaptive:eq:frozen-jensen}
\frac1T\E\int_0^T\|F_t^{\mathrm{fr}}\|^\gamma\,\mathrm dt
\leq m_T^{\gamma/2}.
\end{equation}
This proves~(\ref{app:adaptive:eq:exact-freezing}).
\end{proof}

\subsubsection{Closure, nonaccumulation, and finite horizon bounds}
\label{app:proof-adaptive-closure}

We now use the preceding bounds and the constraint on the base step size to control the average discretization error $e_T$.

\begin{proof}[Proof of the finite horizon claims in
Theorem~\ref{thm:full-adaptive}]
We first perform the calculation when all terms are finite, and then justify
it using a finite step approximation. This order avoids assuming the
gradient integrability that the argument is meant to establish.
Recall the constant term in the gradient second moment bound in
Lemma~\ref{app:adaptive-fisher-gradient},
\[
C_\gamma
:=4dL_0+(2-\gamma)\gamma^{\gamma/(2-\gamma)}
                  (2dL_1)^{2/(2-\gamma)}.
\]
Write $M_T:=1+C_\gamma+K_0/T$ and
$\overline{\mathrm{FI}}_T:=T^{-1}\int_0^T
\mathrm{FI}(\mu_t\|\pi)\,\mathrm dt$.
For the interpolated process, the drift is $F_t^{\mathrm{fr}}$ and the
error action in Lemma~\ref{app:adaptive-kl} is $\mathcal A_T=Te_T$.
Dividing~(\ref{app:adaptive:eq:kl-dissipation}) by $T$ therefore gives
\[
\frac{\mathrm{KL}(\mu_T\|\pi)}{T}
+\frac34\overline{\mathrm{FI}}_T
\leq\frac{K_0}{T}+e_T.
\]
Since $\mathrm{KL}(\mu_T\|\pi)\geq0$, we obtain
\begin{equation}
\label{app:adaptive:eq:closure-fi}
\overline{\mathrm{FI}}_T
\leq\frac43\left(\frac{K_0}{T}+e_T\right).
\end{equation}

We next bound $m_T$, the average squared norm of the applied drift.
Write $F_t^{\mathrm{fr}}=-\nabla V(X_t)
+(F_t^{\mathrm{fr}}+\nabla V(X_t))$ and use
$\|u+v\|^2\leq2\|u\|^2+2\|v\|^2$ to obtain
\[
\|F_t^{\mathrm{fr}}\|^2
\leq2\|\nabla V(X_t)\|^2
+2\|F_t^{\mathrm{fr}}+\nabla V(X_t)\|^2.
\]
Taking expectations, integrating over $[0,T]$, and dividing by $T$ gives
\[
m_T\leq\frac2T\E\int_0^T\|\nabla V(X_t)\|^2\,\mathrm dt+2e_T.
\]
Lemma~\ref{app:adaptive-fisher-gradient}, applied to the law $\mu_t$ of
$X_t$, bounds the gradient second moment by
$\E\|\nabla V(X_t)\|^2\leq2\mathrm{FI}(\mu_t\|\pi)+C_\gamma$.
Substituting this bound into the preceding inequality yields
\begin{equation}
\label{app:adaptive:eq:closure-gradient}
m_T\leq4\overline{\mathrm{FI}}_T+2C_\gamma+2e_T.
\end{equation}
Finally, substituting~(\ref{app:adaptive:eq:closure-fi}) into
~(\ref{app:adaptive:eq:closure-gradient}) gives
\begin{equation}
\label{app:adaptive:eq:closure}
\begin{aligned}
m_T
&\leq\frac{16}{3}\frac{K_0}{T}+2C_\gamma+\frac{22}{3}e_T\\
&\leq8(M_T+e_T).
\end{aligned}
\end{equation}
The last inequality uses $16/3\leq8$, $2\leq8$, and $22/3\leq8$, together
with nonnegativity and $M_T=1+C_\gamma+K_0/T$.
Since $0<\gamma/2\leq1/2$ and $M_T\geq1$, concavity gives
\begin{equation}
\label{app:adaptive:eq:concavity-closure}
(M_T+e_T)^{\gamma/2}
\leq M_T^{\gamma/2}
+\frac\gamma2M_T^{\gamma/2-1}e_T
\leq M_T^{\gamma/2}+e_T.
\end{equation}

After applying the moment bound, we define
\[
R_\gamma(h,M)
:=A_\gamma(h)+8B_\gamma(h)M
  +8^{\gamma/2}Q_\gamma(h)M^{\gamma/2}
\]
to collect the terms independent of the discretization
error $e_T$. The coefficient multiplying $e_T$ is
\[
\Delta_\gamma(h)
:=8B_\gamma(h)+8^{\gamma/2}Q_\gamma(h).
\]
Substitute the bound for $m_T$ in
~(\ref{app:adaptive:eq:closure}) into the discretization bound for exact gradients
~(\ref{app:adaptive:eq:exact-freezing}), and then apply
~(\ref{app:adaptive:eq:concavity-closure}) to obtain
\begin{align}
e_T
&\leq A_\gamma(h)+8B_\gamma(h)(M_T+e_T)
 +8^{\gamma/2}Q_\gamma(h)(M_T+e_T)^{\gamma/2}\notag\\
&\leq R_\gamma(h,M_T)+\Delta_\gamma(h)e_T.
\label{app:adaptive:eq:scalar-feedback}
\end{align}
If $\Delta_\gamma(h)\leq1/2$, then
\begin{equation}
\label{app:adaptive:eq:error-absorbed}
e_T\leq2R_\gamma(h,M_T).
\end{equation}
Substituting~(\ref{app:adaptive:eq:error-absorbed}) into the Fisher
information bound~(\ref{app:adaptive:eq:closure-fi}) proves
\[
\overline{\mathrm{FI}}_T
\leq\frac{4K_0}{3T}+\frac83R_\gamma(h,M_T).
\]
Convexity of relative Fisher information gives
$\mathrm{FI}(\bar\mu_T\|\pi)\leq\overline{\mathrm{FI}}_T$, proving
~(\ref{app:eq:adaptive-fi}). The coefficients
$A_\gamma,B_\gamma,Q_\gamma$ defined above are
continuous and nondecreasing in $h$, vanish at zero, and for $0<h\leq1$
are at most $h$ times their values at one. Consequently,
$\Delta_\gamma(h)\leq h\Delta_\gamma(1)$. The base step constraint in
Theorem~\ref{thm:full-adaptive}, namely
$0<h\leq[1+2\Delta_\gamma(1)]^{-1}$, therefore implies
\[
\Delta_\gamma(h)
\leq\frac{\Delta_\gamma(1)}{1+2\Delta_\gamma(1)}<\frac12,
\]
which verifies the stated explicit sufficient condition.

\paragraph{Finite step approximation.}
The rearrangement leading to~(\ref{app:adaptive:eq:error-absorbed})
requires $e_T<\infty$, which has not yet been established for the full
adaptive process. We therefore first truncate the adaptive scheme after $n$
steps and prove finiteness for each fixed $n$. The preceding inequalities
then give bounds independent of $n$, allowing us to let $n\to\infty$ in the
next paragraph.

First suppose that $\mu_0$ has compact support.  Run at most $n$ adaptive
intervals and then append the exact diffusion from
Lemma~\ref{app:adaptive-continuation}.  This produces a conservative process
on every finite horizon.  Equations~(\ref{app:adaptive:eq:endpoint-square})
and~(\ref{app:adaptive:eq:endpoint-one}), together with
Lemma~\ref{app:adaptive:lem:gaussian}, imply at the end of each complete
adaptive interval that
\begin{equation*}
\E\!\left[
 \|\nabla V(X_{\tau_{k+1}})\|^2
 \;\middle|\;\mathcal F_k\right]
\leq C(h,d,L_0,L_1,\gamma)
\left(1+\|\nabla V(X_{\tau_k})\|^2\right),
\end{equation*}
where the finite constant does not depend on $k$. For $0<\gamma<1$, this
uses $2\gamma/(1-\gamma)=2/(1-\gamma)-2$. For $\gamma=1$, it uses the
exponential moments in~(\ref{app:adaptive:eq:gaussian}).  Continuity bounds
the initial gradient on the compact support, so induction gives finite
gradient second moments through every one of the first $n$ intervals.

The single-interval estimates in
~(\ref{app:adaptive:eq:local-smooth}),
~(\ref{app:adaptive:eq:local-interior}), and
~(\ref{app:adaptive:eq:local-one}) give finite total error action over the
adaptive intervals. This error is zero on the exact continuation.
Lemma~\ref{app:adaptive-kl} therefore
first establishes a finite Fisher integral, and
Lemma~\ref{app:adaptive-fisher-gradient} then gives a finite current-gradient
integral. All terms in~(\ref{app:adaptive:eq:closure-fi})--\allowbreak
(\ref{app:adaptive:eq:closure}) are consequently
finite, justifying the preceding absorption. The preliminary moment bounds
are used only to establish finiteness. Applying
~(\ref{app:adaptive:eq:scalar-feedback}) and
~(\ref{app:adaptive:eq:closure-fi}) now gives
~(\ref{app:adaptive:eq:error-absorbed}) and the resulting Fisher information
bound with constants independent of $n$.

\paragraph{Passing to the limit and excluding finite-time accumulation.}
Let
\[
N_{T,n}:=\#\{0\leq k<n:\tau_k<T\}.
\]
Every complete interval satisfies
\[
1=h^{-1}H_k
\bigl(1+\omega\|\nabla V(X_{\tau_k})\|^\gamma\bigr),
\]
and at most one occupied interval is truncated by $T$.  Thus, pathwise,
\begin{equation}
\label{app:adaptive:eq:capped-count}
N_{T,n}
\leq1+\frac1h\int_0^{T\wedge\tau_n}
\left(1+\omega\|F_t^{\mathrm{fr}}\|^\gamma\right)\,\mathrm dt.
\end{equation}
Here $T\wedge\tau_n=\min\{T,\tau_n\}$. Bounding the number of steps by
integrating the reciprocal step duration follows the argument of
\citet[Lemma~2]{fang2020adaptive} for adaptive Euler--Maruyama schemes.
Take expectations in
~(\ref{app:adaptive:eq:capped-count}), apply
~(\ref{app:adaptive:eq:frozen-jensen}),
~(\ref{app:adaptive:eq:closure}), and
~(\ref{app:adaptive:eq:error-absorbed}) to obtain, uniformly in $n$,
\begin{equation}
\label{app:adaptive:eq:capped-count-expectation}
\E N_{T,n}
\leq1+\frac{T}{h}
\left[1+\omega
\{8(M_T+2R_\gamma(h,M_T))\}^{\gamma/2}\right].
\end{equation}

Construct all finite prefixes using the same Brownian motion and initial
state.  On $\{\tau_n\leq T\}$, every one of the first $n$ intervals starts
before $T$, so $N_{T,n}=n$.  Markov's inequality and
~(\ref{app:adaptive:eq:capped-count-expectation}) give
\[
\mathbb P(\tau_n\leq T)
\leq\frac{\E N_{T,n}}{n}\longrightarrow0.
\]
Applying this argument to every positive integer horizon proves
$\tau_n\uparrow\infty$ almost surely.  For each fixed $T$, the truncated and
original path laws agree on $\{\tau_n>T\}$ and hence converge in total
variation on $[0,T]$.  Lower semicontinuity of Fisher information and Fatou's
lemma pass~(\ref{app:eq:adaptive-fi}). Monotone convergence passes
~(\ref{app:adaptive:eq:capped-count-expectation}) and proves
~(\ref{app:eq:adaptive-events}).  Because the step durations are strictly
positive, the definition of $N_T$ in
Definition~\ref{def:algorithmic-laws} now agrees with the number of occupied
intervals $\#\{k:\tau_k<T\}$.

\paragraph{General finite KL initialization.}
We now remove the temporary compact support assumption on $\mu_0$ and
extend the preceding bounds to any initial law with
$K_0=\mathrm{KL}(\mu_0\|\pi)<\infty$, as required by the theorem.
We approximate $\mu_0$ by conditioning it on expanding balls and then
pass to the limit.
Keep $h,T$ fixed and use the same update rule and Brownian driving noise
throughout. The base step constraint in Theorem~\ref{thm:full-adaptive}
does not depend on the initial law, so it remains valid for every
approximating initialization.

Choose increasing balls $A_R$ with $p_R:=\mu_0(A_R)>0$ and
$p_R\uparrow1$, and set $\mu_0^R:=\mu_0(\cdot\mid A_R)$. If
$f:=\mathrm d\mu_0/\mathrm d\pi$, then
\[
\mathrm{KL}(\mu_0^R\|\pi)
=p_R^{-1}\int_{A_R}f\log f\,\mathrm d\pi-\log p_R
\longrightarrow K_0.
\]
The negative part is integrable because $z\log z\geq-e^{-1}$, and
$K_0<\infty$ ensures that the positive part is integrable as well.
This justifies the convergence of the restricted integrals.
Apply the compact support argument to each $\mu_0^R$, using
$\mathrm{KL}(\mu_0^R\|\pi)$ in place of $K_0$. Conditioning on the
initial state gives
\[
\mathbb P_{\mu_0}(\tau_\infty<\infty,\ X_0\in A_R)=0.
\]
Since the balls exhaust the initial law, $\tau_\infty=\infty$ almost
surely also under $\mu_0$. The conditional law of the interpolated path
given its initial state defines a probability kernel, defined arbitrarily
on any exceptional $\mu_0$-null set of initial states. Applying this same
kernel to $\mu_0^R$ and $\mu_0$ gives convergence of their path laws in
total variation on every finite time horizon. Lower semicontinuity of
Fisher information and Fatou's lemma then pass the integrated Fisher bound
~(\ref{app:eq:adaptive-fi}). Its right side converges because
$\mathrm{KL}(\mu_0^R\|\pi)\to K_0$ and $R_\gamma(h,M)$ is continuous in $M$.
Furthermore,
\[
p_R\E_{\mu_0^R}N_T
=\E_{\mu_0}[N_T\mathbf1_{\{X_0\in A_R\}}]
\uparrow\E_{\mu_0}N_T.
\]
Here monotone convergence applies to
$N_T\mathbf1_{\{X_0\in A_R\}}$. Since $p_R\to1$ and the right side of
~(\ref{app:eq:adaptive-events}) is continuous in the initial KL divergence,
the expected iteration bound passes to the limit as well. Convexity and lower
semicontinuity of Fisher information follow, for example, from
\[
\mathrm{FI}(\mu\|\pi)
=\sup_{\varphi\in C_c^\infty(\mathbb R^d;\mathbb R^d)}
\int\left\{
2\langle\nabla V,\varphi\rangle
-2\operatorname{div}\varphi-\|\varphi\|^2
\right\}\,\mathrm d\mu.
\]
This also passes the time-averaged-law bound.  Finally, each occupied interval
uses exactly one exact gradient query at its start and no unused terminal
query.  Hence $Q_T=N_T$ pathwise.  This completes the proof of all
finite horizon claims.
\end{proof}

\subsubsection{Accuracy schedule and complexity}
\label{app:proof-adaptive-complexity}

We finish by verifying the explicit schedule in the full theorem and deriving
the stated complexity bounds.

\begin{proof}[Proof of the accuracy and complexity claims in
Theorem~\ref{thm:full-adaptive}]
Assume $K_0>0$ and $0<\varepsilon\leq1$. Recall the accuracy schedule
\[
T:=\frac{4K_0}{\varepsilon},\qquad
h:=\max\left\{u\in[0,(1+2\Delta_\gamma(1))^{-1}]:
R_\gamma(u,M_T)\leq\frac{3\varepsilon}{16}\right\},
\]
where $M_T=1+C_\gamma+K_0/T$. Hence
\[
\frac{K_0}{T}=\frac\varepsilon4,
\qquad
M_T=1+C_\gamma+\frac{\varepsilon}{4},
\qquad
R_\gamma(h,M_T)
\leq\frac{3\varepsilon}{16}.
\]
The set defining $h$ has a positive maximum because
$R_\gamma(\cdot,M_T)$ is continuous, nondecreasing, vanishes at
zero, and the threshold is positive. Since $M_T\geq1$ and all
coefficients are nonnegative,
\begin{equation}
R_\gamma(h,M_T)-\Delta_\gamma(h)=A_\gamma(h)+8B_\gamma(h)(M_T-1)+8^{\gamma/2}Q_\gamma(h)(M_T^{\gamma/2}-1)
\geq0.
\label{app:adaptive:eq:accuracy-absorption}
\end{equation}
Therefore
\[
\Delta_\gamma(h)
\leq\frac{3\varepsilon}{16}<\frac12.
\]
The finite horizon bound then gives the slightly sharper intermediate result
\begin{equation*}
\mathrm{FI}(\bar\mu_T\|\pi)
\leq\frac\varepsilon3
+\frac83\frac{3\varepsilon}{16}
=\frac{5\varepsilon}{6}
\leq\varepsilon.
\end{equation*}

Equations~(\ref{app:adaptive:eq:scalar-feedback}),
~(\ref{app:adaptive:eq:accuracy-absorption}), and
~(\ref{app:adaptive:eq:closure}) give
$e_T\leq3\varepsilon/8$ and $m_T\leq16M_T$. Since
$16^{\gamma/2}\leq4$, the expected count bound becomes
\begin{equation}
\label{app:adaptive:eq:accuracy-count-proof}
\E N_T
\leq1+\frac{4K_0}{\varepsilon h}
\left(1+4\omega M_T^{\gamma/2}\right),
\end{equation}
which is~(\ref{app:eq:adaptive-accuracy-count}).

\paragraph{The regime $L_1>0$ and $0<\gamma<1$.}
The formulas for $A_\gamma,B_\gamma,Q_\gamma$ in the proof of
Lemma~\ref{app:adaptive:lem:exact-freezing} and
$R_\gamma(u,M_T)=A_\gamma(u)+8B_\gamma(u)M_T
+8^{\gamma/2}Q_\gamma(u)M_T^{\gamma/2}$ give, as $u\downarrow0$,
\begin{align*}
R_\gamma(u,M_T)
&=3d\left(
L_0^2(1+2^{\gamma^2/(1-\gamma)})^2
+8^{\gamma/2}2^{2\gamma^2/(1-\gamma)}3^{2\gamma}L_1(1+L_0)
 M_T^{\gamma/2}
\right)u+o(u),
\end{align*}
uniformly for $0<\varepsilon\leq1$. Maximality of $h$ therefore
implies, for sufficiently small $\varepsilon$ at fixed model parameters,
\begin{equation}
\label{app:adaptive:eq:inverse-step-interior}
h^{-1}
=\mathcal O\!\left(
1+\frac{d[L_0^2+L_1(1+L_0)M_T^{\gamma/2}]}
         {\varepsilon}
\right).
\end{equation}
In this regime, the numerical constants hidden by $\mathcal O$ depend
only on $\gamma$.
To justify this inversion explicitly, the positive linear coefficient
above bounds $R_\gamma(u,M_T)/u$ from above by twice that
coefficient once $u$ is sufficiently small. A step proportional to
$\varepsilon$ divided by that coefficient is consequently admissible and,
for sufficiently small $\varepsilon$, lies below the fixed upper bound
in~(\ref{app:eq:adaptive-accuracy-schedule}). Hence the maximal admissible
step is at least as large. The threshold on $\varepsilon$ may depend on the
fixed smoothness parameters and dimension.
Combining~(\ref{app:adaptive:eq:accuracy-count-proof}) and
~(\ref{app:adaptive:eq:inverse-step-interior}), and recalling
$\omega=L_1/(1+L_0)$, bounds the leading $\varepsilon^{-2}$ term by
\begin{align}
\frac{\mathcal O(K_0d)}{\varepsilon^2}
&\left(1+\frac{L_1}{1+L_0}M_T^{\gamma/2}\right)
\left(L_0^2+L_1(1+L_0)M_T^{\gamma/2}\right).
\label{app:adaptive:eq:rate-product}
\end{align}
Because $L_0^2/(1+L_0)\leq1+L_0$, the product is bounded, up to a constant, by
\[
L_0^2+L_1(1+L_0)M_T^{\gamma/2}
+L_1^2M_T^\gamma.
\]
With
$S_\gamma:=1+dL_0+(dL_1)^{2/(2-\gamma)}$, the definition of $C_\gamma$
gives $M_T=\mathcal O(S_\gamma)$. Consequently,
\begin{equation*}
\E N_T=\E Q_T
=\mathcal O\!\left(
\frac{K_0d}{\varepsilon^2}
\left[L_0^2+L_1(1+L_0)S_\gamma^{\gamma/2}
      +L_1^2S_\gamma^\gamma\right]\right).
\end{equation*}
Using $(x+y)^a\leq x^a+y^a$ for $x,y\geq0$ and
$0<a\leq1$, we can further bound the expression
in brackets by
\begin{align*}
L_0^2+L_1(1+L_0)
 \{1+(dL_0)^{\gamma/2}+(dL_1)^{\gamma/(2-\gamma)}\}+L_1^2
 \{1+(dL_0)^\gamma+(dL_1)^{2\gamma/(2-\gamma)}\}.
\end{align*}

\paragraph{The endpoint $\gamma=1$.}
This endpoint is proved directly.  Since
$E_1^{\mathrm{ex}}(u)=1+\mathcal O(\sqrt u)$,
the formulas for $A_1,B_1,Q_1$ in the proof of the discretization error
bound in Lemma~\ref{app:adaptive:lem:exact-freezing} give
\begin{align*}
R_1(u,M_T)
&=(d+1)\left(
4L_0^2+8\sqrt2L_1(1+L_0)M_T^{1/2}
\right)u+o(u).
\end{align*}
Consequently,
\begin{equation*}
h^{-1}
=\mathcal O\!\left(
1+\frac{(d+1)[L_0^2+L_1(1+L_0)M_T^{1/2}]}
         {\varepsilon}
\right).
\end{equation*}
Applying the same product calculation as in
~(\ref{app:adaptive:eq:rate-product}) and using
$S_1=1+dL_0+d^2L_1^2$ gives
\begin{equation*}
\E N_T=\E Q_T
=\mathcal O\!\left(
\frac{K_0(d+1)}{\varepsilon^2}
\left[L_0^2+L_1(1+L_0)S_1^{1/2}+L_1^2S_1\right]
\right).
\end{equation*}
This establishes the endpoint directly, without taking a limit in the
$\gamma<1$ constants.

Recall $\Lambda:=\max\{1,L_0,L_1\}$. For $0<\gamma<1$,
$S_\gamma=\mathcal O((d\Lambda)^{2/(2-\gamma)})$, and at
$\gamma=1$, $d+1\leq2d$.  In the expanded bounds, the largest powers are
supplied by $L_1^2S_\gamma^\gamma$, giving
\[
\E N_T=\E Q_T
=\mathcal O\!\left(
\frac{K_0\Lambda^{4/(2-\gamma)}
d^{(2+\gamma)/(2-\gamma)}}{\varepsilon^2}
\right).
\]
The hidden numerical constant depends only on $\gamma$.
This proves the accuracy and complexity claims in
Theorem~\ref{thm:full-adaptive}.
\end{proof}

\subsection{Proof of Corollary~\ref{cor:adaptive-special-cases}}

\begin{proof}
The first item follows from~(\ref{app:eq:adaptive-compact-rate}) after
imposing $K_0=\mathcal O(d)$.

\paragraph{The globally $L_0$-smooth case $L_1=0$.}
Here $\omega=0$, so $H_k=h$ deterministically. Recall that
$A_\gamma(h)=dL_0^2h$, $B_\gamma(h)=L_0^2h^2/3$, and
$Q_\gamma(h)=0$ in this regime. Substituting these into
$R_\gamma(h,M_T)=A_\gamma(h)+8B_\gamma(h)M_T
+8^{\gamma/2}Q_\gamma(h)M_T^{\gamma/2}$ gives, independently of $\gamma$,
\begin{equation}
\label{app:adaptive:eq:smooth-remainder}
R_\gamma(h,M_T)
=dL_0^2h+\frac83L_0^2M_T h^2.
\end{equation}
Let $\overline M:=1+4dL_0+1/4$.  Whenever
\begin{equation*}
\varepsilon
\leq\min\left\{
1,\frac{32dL_0^2}{3},
\frac{4d^2L_0^2}{\overline M}
\right\},
\end{equation*}
the candidate $u_0=3\varepsilon/(32dL_0^2)$ belongs to $[0,1]$ and makes each
term on the right of~(\ref{app:adaptive:eq:smooth-remainder}) at most
$3\varepsilon/32$.  Maximality in
~(\ref{app:eq:adaptive-accuracy-schedule}) therefore gives
\[
h\geq u_0=\frac{3\varepsilon}{32dL_0^2}.
\]
The count is deterministic and equals
$N_T=\lceil T/h\rceil$. Hence
\[
Q_T=N_T
\leq1+\frac{128K_0dL_0^2}{3\varepsilon^2},
\]
which proves~(\ref{app:eq:adaptive-smooth-rate}).  Under
$K_0=\mathcal O(d)$, this is
$\mathcal O(d^2L_0^2/\varepsilon^2)$, recovering the globally smooth
exact gradient LMC order.  Neither the algorithm nor this specialization of
its accuracy schedule depends on $\gamma$.
\end{proof}


\subsection{Proof of Theorem~\ref{thm:full-combined}}
\label{app:proof-combined}

The robust estimator and the adaptive step size control different errors.
The former turns a finite-$p$ stochastic oracle into a gradient estimate with a
finite conditional second moment, while the latter controls the variation of
$\nabla V$ during an iteration. The adaptive interval length depends on the
same gradient estimator used in the update. Consequently, controlling the
estimator error at each iteration does not directly control its accumulation
over time. The term $\kappa s_p(X_{\tau_k})^2$ in the step-size denominator
allows us to bound this accumulated error through
~(\ref{app:eq:combined-error-count}). We first bound the discretization and
estimation errors, then derive the iteration and oracle complexities.

\subsubsection{Finite iteration localization and proof quantities}

We first suppose that the initial law is supported in a ball, while retaining
its actual finite KL divergence
$K_0:=\mathrm{KL}(\mu_0\|\pi)$, and truncate the RA-LMC construction
after $m$ iterations. If the truncation time $\tau_m$ occurs before $T$, we continue
with the conservative exact Langevin diffusion from
Lemma~\ref{app:adaptive-continuation}.  This continuation is only a proof
device and is not part of RA-LMC. The state at the truncation time has an absolutely
continuous law because its last Brownian increment is nondegenerate.

On an RA-LMC interval beginning at $\tau_k$, use the force from
~(\ref{eq:robust-lmc-adaptive-iterates}) and write
\begin{equation*}
 F_t^{\mathrm{fr}}:=F_k=-\frac1b\sum_{i=1}^b\widehat{\mathcal G}_{k,i},
 \qquad X_t^{\mathrm{fr}}:=X_{\tau_k},
 \qquad t\in[\tau_k,\tau_{k+1}),
\end{equation*}
Here $F_t^{\mathrm{fr}}$ and $X_t^{\mathrm{fr}}$ denote the drift and state
selected at the beginning of the interval and held constant throughout it.
After the truncation time, set $F_t^{\mathrm{fr}}=-\nabla V(X_t)$ and
$X_t^{\mathrm{fr}}=X_t$.  If $\tau_k<T$, let
$\ell_k:=\min\{H_k,T-\tau_k\}$ be the duration of that update before $T$.
Recall the adaptive rule and the definition of $\omega$,
\[
 H_k=\frac{h}{1+\omega\|F_k\|^\gamma
                 +\kappa s_p(X_{\tau_k})^2},
 \qquad \omega:=\frac{L_1}{1+L_0}.
\]
They imply
\begin{equation}
 H_k\leq h,\qquad
 \omega H_k\|F_k\|^\gamma\leq h,\qquad
 \kappa H_ks_p(X_{\tau_k})^2\leq h.
 \label{app:eq:combined-step-properties}
\end{equation}
Let $\mathcal F_k$ denote the history available immediately before the
fresh oracle batch at iteration $k$ is drawn.  In particular,
$X_{\tau_k}$ and $\tau_k$ are $\mathcal F_k$-measurable, whereas $F_k$
and $H_k$ also depend on the fresh batch. The clipping radius is
$\lambda_k=s_p(X_{\tau_k})b^{1/p}$. Lemma~\ref{app:robust-estimation}
gives
\begin{equation}
 \E\!\left[
 \|F_k+\nabla V(X_{\tau_k})\|^2\,\middle|\,\mathcal F_k\right]
 \leq \rho_b s_p(X_{\tau_k})^2,
 \qquad \rho_b=C_pb^{-2(p-1)/p}.
 \label{app:eq:combined-event-mse}
\end{equation}
Here $C_p$ is the estimator error coefficient established in
Lemma~\ref{app:robust-estimation}.
The moment justification proceeds in a fixed order.  Bounded initial support
gives finite initial state and gradient second moments.  The conditional
estimation bound above and $s_p(x)^2\leq C(1+\|x-x_{\mathrm{ref}}\|^2)$
give a finite second moment for each selected force.  The gradient variation
bounds in Lemma~\ref{app:adaptive:lem:endpoint}, together with the adaptive step
bound, then propagate finite state and gradient second moments through any
finite number of updates. For $L_1>0$ and $0<\gamma<1$, the adaptive
step bound gives
\[
 (H_k\|F_k\|)^{2/(1-\gamma)}
 \leq (h/\omega)^{2/(1-\gamma)}\|F_k\|^2.
\]
Thus the required displacement moment is controlled by the force second
moment. For $L_1>0$ and $\gamma=1$, the drift displacement is
bounded by $h/\omega$, and Gaussian exponential moments give the same
conclusion.  When $L_1=0$, global smoothness suffices.

Thus the expected integral of the squared drift error is finite before
truncation, and the drift error is zero on the exact continuation.
Lemma~\ref{app:adaptive-kl} gives finite
integrated FI. Lemma~\ref{app:adaptive-fisher-gradient} then gives a finite
integral of the gradient second moment over time, including on the
continuation. The path identity yields the bounds on expected squared
distances from $x_{\mathrm{ref}}$ used below. This order avoids
assuming the moment bounds that the subsequent closure will prove.

\subsubsection{Discretization, FI, and gradient and drift moment bounds}

For the process constructed above, which follows RA-LMC for at most $m$
iterations and then continues with the exact Langevin diffusion if needed,
define the following quantities.
\begin{align*}
 \overline{\mathrm{FI}}_T
 &:=\frac1T\int_0^T\FI(\mu_t\|\pi)\,\mathrm{d}t,
 &\overline e
 &:=\frac1T\E\int_0^T
   \|F_t^{\mathrm{fr}}+\nabla V(X_t^{\mathrm{fr}})\|^2\,\mathrm{d}t,\\
 \overline g
 &:=\frac1T\E\int_0^T
   \|\nabla V(X_t^{\mathrm{fr}})\|^2\,\mathrm{d}t,
 &\overline a
 &:=\frac1T\E\int_0^T\|F_t^{\mathrm{fr}}\|^2\,\mathrm{d}t,\\
 D_{\mathrm{disc}}
 &:=\E\int_0^T
   \|\nabla V(X_t)-\nabla V(X_t^{\mathrm{fr}})\|^2\,\mathrm{d}t.
\end{align*}
These quantities represent, respectively, FI, the average squared estimation
error, the average squared norm of $\nabla V(X_t^{\mathrm{fr}})$, the average
squared drift norm, and the accumulated discretization error.

Condition after the gradient estimate and step size have been selected but
before the Brownian increment is generated.  By
~(\ref{app:eq:combined-step-properties}),
$\ell_k\leq h/(1+\omega\|F_k\|^\gamma)$, so
Lemma~\ref{app:adaptive-power-freezing} applies.  Summing its local bound over
the occupied intervals gives
\begin{equation*}
 \frac{D_{\mathrm{disc}}}{T}
 \leq \widetilde{A}_\gamma(h)
      +\widetilde{U}_\gamma(h)\overline g
      +\widetilde{V}_\gamma(h)\overline a.
\end{equation*}
Here $\widetilde{A}_\gamma$, $\widetilde{U}_\gamma$, and $\widetilde{V}_\gamma$ are the
coefficients defined above for Lemma~\ref{app:adaptive-power-freezing}. They
multiply the constant term, the average squared gradient norm $\overline g$,
and the average squared drift norm $\overline a$, respectively.
Since $\overline a\leq2\overline e+2\overline g$,
\begin{equation}
 \frac{D_{\mathrm{disc}}}{T}
 \leq \widetilde{A}_\gamma(h)
 +\{\widetilde{U}_\gamma(h)+2\widetilde{V}_\gamma(h)\}\overline g
 +2\widetilde{V}_\gamma(h)\overline e.
 \label{app:eq:combined-freezing-reduced}
\end{equation}

To bound $\overline g$, compare $\nabla V(X_t^{\mathrm{fr}})$ with the
current gradient $\nabla V(X_t)$.
\[
 \overline g\leq\frac{2D_{\mathrm{disc}}}{T}
 +\frac2T\int_0^T\E\|\nabla V(X_t)\|^2\,\mathrm{d}t.
\]
The gradient moment bound in Lemma~\ref{app:adaptive-fisher-gradient} uses
\[
 C_\gamma
 :=4dL_0+(2-\gamma)\gamma^{\gamma/(2-\gamma)}
                  (2dL_1)^{2/(2-\gamma)}.
\]
That lemma and
~(\ref{app:eq:combined-freezing-reduced}) therefore yield
\begin{align}
 &[1-2\{\widetilde{U}_\gamma(h)+2\widetilde{V}_\gamma(h)\}]\overline g \leq4\overline{\mathrm{FI}}_T+2C_\gamma
       +2\widetilde{A}_\gamma(h)+4\widetilde{V}_\gamma(h)\overline e.
 \label{app:eq:combined-gradient-preclosure}
\end{align}
Next apply the KL-dissipation inequality in
Lemma~\ref{app:adaptive-kl} and use
\begin{align*}
 \|\nabla V(X_t)+F_t^{\mathrm{fr}}\|^2\leq2\|\nabla V(X_t)-\nabla V(X_t^{\mathrm{fr}})\|^2+2\|\nabla V(X_t^{\mathrm{fr}})+F_t^{\mathrm{fr}}\|^2.
\end{align*}
Together with~(\ref{app:eq:combined-freezing-reduced}), this gives
\begin{align}
 \frac34\overline{\mathrm{FI}}_T\leq\frac{K_0}{T}+2\widetilde{A}_\gamma(h)+2\{\widetilde{U}_\gamma(h)+2\widetilde{V}_\gamma(h)\}\overline g+\{4\widetilde{V}_\gamma(h)+2\}\overline e.
 \label{app:eq:combined-kl-before-closure}
\end{align}
Suppose for now that
$\widetilde{U}_\gamma(h)+2\widetilde{V}_\gamma(h)\leq1/64$.
The theorem's step-size choice will imply this sufficient condition below.
Substitute
~(\ref{app:eq:combined-gradient-preclosure}) into
~(\ref{app:eq:combined-kl-before-closure}).  For clarity, put
$c_h:=\widetilde{U}_\gamma(h)+2\widetilde{V}_\gamma(h)$ in this calculation.
Collecting the terms involving $\overline{\mathrm{FI}}_T$ on the left gives
the coefficient
$3/4-8c_h/(1-2c_h)\geq1/2$.  Also
$\widetilde{V}_\gamma(h)\leq c_h/2$ and $1-2c_h\geq31/32$.
Recall the discretization coefficient
\[
 \mathcal D_\gamma(h)
 :=\widetilde{A}_\gamma(h)
   +\{\widetilde{U}_\gamma(h)+2\widetilde{V}_\gamma(h)\}(1+C_\gamma).
\]
Bounding the remaining coefficients gives
\begin{align}
 \overline{\mathrm{FI}}_T
 &\leq\frac{2K_0}{T}
   +5\widetilde{A}_\gamma(h)
   +9\{\widetilde{U}_\gamma(h)+2\widetilde{V}_\gamma(h)\}C_\gamma
   +5\overline e
 \notag\\
 &\leq\frac{2K_0}{T}+9\mathcal D_\gamma(h)+5\overline e.
 \label{app:eq:combined-fi-closure}
\end{align}
The last inequality uses this definition of $\mathcal D_\gamma$.
Recall the quantity defined in Theorem~\ref{thm:full-combined},
\[
 M_T:=1+C_\gamma+\frac{K_0}{T}.
\]
Substituting
~(\ref{app:eq:combined-fi-closure}) back into
~(\ref{app:eq:combined-gradient-preclosure}) also gives the following bounds
on the average squared gradient and drift norms,
\begin{equation}
 \overline g\leq9M_T+40\mathcal D_\gamma(h)+21\overline e,
 \qquad
 \overline a\leq18M_T+80\mathcal D_\gamma(h)+44\overline e.
 \label{app:eq:combined-energy-closure}
\end{equation}
The numerical constants in these displays are universal upper bounds and they
hide no dependence on the problem parameters.

\subsubsection{Bounds on squared distance and local noise}

The average squared local noise scale is
\[
 \overline s
 :=\frac1T\E\int_0^Ts_p(X_t^{\mathrm{fr}})^2\,\mathrm{d}t.
\]
To bound it, define the term independent of $\overline e$ in the squared
distance bound~(\ref{app:eq:combined-frozen-spatial}),
\begin{equation*}
 R_{T,h}:=6M_{2,0}+6dT+2dh+(3T^2+2h^2/3)\{18M_T+80\mathcal D_\gamma(h)\}.
\end{equation*}
For the theorem's accuracy schedule, $T\geq8$, $h\leq1$, and
$\mathcal D_\gamma(h)\leq\varepsilon/64\leq1/64$. Since $M_T\geq1$,
\[
 18M_T+80\mathcal D_\gamma(h)\leq20M_T,
 \qquad
 R_{T,h}\leq S_T:=6M_{2,0}+8dT+80T^2M_T.
\]
Indeed, $h\leq T$ gives $6M_{2,0}+6dT+2dh\leq6M_{2,0}+8dT$ and
bounds the remaining term by $(11/3)T^2(20M_T)\leq80T^2M_T$.
Thus the theorem's noise scale and step-size coefficient are
\[
 \bar s_T^{2}:=[B S_T^{p/2}+G]^{2/p},
 \qquad \kappa:=(1+\bar s_T^{2})^{-1}.
\]
For $B>0$, $M_{2,0}$ satisfies
$\E\|X_0-x_{\mathrm{ref}}\|^2\leq M_{2,0}$. For $B=0$, the convention is
$M_{2,0}=0$, and $\bar s_T^{\,2}=G^{2/p}$ does not depend on $S_T$.
The local noise scale is then constant, so
$\overline s=G^{2/p}=\bar s_T^{\,2}$ without a bound on the expected squared
distance from $x_{\mathrm{ref}}$.
We now treat $B>0$.  The path identity and Cauchy--Schwarz imply
\begin{align}
 X_t-x_{\mathrm{ref}}
 &=X_0-x_{\mathrm{ref}}+\int_0^t F_u^{\mathrm{fr}}\,\mathrm{d}u+\sqrt2W_t,
 \notag\\
 \frac1T\int_0^T\E\|X_t-x_{\mathrm{ref}}\|^2\,\mathrm{d}t
 &\leq3M_{2,0}+3dT+\frac32T^2\overline a.
 \label{app:eq:combined-current-spatial}
\end{align}
On an occupied interval, conditional Gaussian calculation gives
\begin{equation*}
 \int_0^{\ell_k}\E\!\left[
 \|X_{\tau_k+s}-X_{\tau_k}\|^2
 \mid X_{\tau_k},F_k,\tau_k\right]\mathrm{d}s=\frac{\ell_k^3}{3}\|F_k\|^2+d\ell_k^2
 \leq\ell_k\left\{\frac{h^2}{3}\|F_k\|^2+dh\right\}.
\end{equation*}
Combining this calculation with
~(\ref{app:eq:combined-current-spatial}) and
~(\ref{app:eq:combined-energy-closure}) yields
\begin{align}
 \frac1T\E\int_0^T\|X_t^{\mathrm{fr}}-x_{\mathrm{ref}}\|^2\,\mathrm{d}t
 \leq R_{T,h}
 +44\left(3T^2+\frac{2h^2}{3}\right)\overline e.
 \label{app:eq:combined-frozen-spatial}
\end{align}

Define $\phi(u):=(Bu^{p/2}+G)^{2/p}$ for $u\geq0$, so that
$\phi(\|x-x_{\mathrm{ref}}\|^2)=s_p(x)^2$.
Lemma~\ref{lem:robust-noise-concavity} shows that $\phi$ is nonnegative,
nondecreasing, and concave. Apply Jensen's inequality to $\phi$ under the
probability measure $T^{-1}\mathrm{d}t\,\mathrm{d}\mathbb P$.
Moreover, these properties imply
$\phi(u+v)\leq(1+v/u)\phi(u)$ for $u>0$. Since
$S_T\geq R_{T,h}\geq18(3T^2+2h^2/3)$,
~(\ref{app:eq:combined-frozen-spatial}) gives
\begin{equation}
 \overline s\leq\bar s_T^{\,2}(1+3\overline e).
 \label{app:eq:combined-noise-closure}
\end{equation}
For $B=0$,~(\ref{app:eq:combined-noise-closure}) holds directly from the
constant scale identity above.

\subsubsection{From per-update estimation error to error over elapsed time}

Let $N_{T,m}$ be the number of RA-LMC intervals among the first $m$ whose
starting times are strictly below $T$.  The indicator that iteration $k$ is
occupied is measurable before its fresh oracle batch is drawn.  Although
$\ell_k$ depends on that batch through the estimated gradient, the
predictable noise term in the adaptive denominator and
~(\ref{app:eq:combined-event-mse}) imply
\begin{align*}
 \E\!\left[
 \ell_k\|F_k+\nabla V(X_{\tau_k})\|^2
 \,\middle|\,\mathcal F_k\right]&\leq
 \frac{h}{1+\kappa s_p(X_{\tau_k})^2}
 \E\!\left[
 \|F_k+\nabla V(X_{\tau_k})\|^2
 \,\middle|\,\mathcal F_k\right]
\\
 &\leq h\rho_b
 \frac{s_p(X_{\tau_k})^2}{1+\kappa s_p(X_{\tau_k})^2}\\
 &\leq\frac{h\rho_b}{\kappa}.
\end{align*}
Consequently,
\begin{equation}
 T\overline e\leq\frac{h\rho_b}{\kappa}\E N_{T,m}.
 \label{app:eq:combined-error-count}
\end{equation}
Every complete interval satisfies
\[
 H_k\{1+\omega\|F_k\|^\gamma
        +\kappa s_p(X_{\tau_k})^2\}=h,
\]
and at most the last occupied interval is truncated at $T$.  It follows
pathwise that
\begin{align*}
 N_{T,m}
 \leq1+\frac1h\int_0^{T\wedge\tau_m}
 \{1+\omega\|F_t^{\mathrm{fr}}\|^\gamma
       +\kappa s_p(X_t^{\mathrm{fr}})^2\}\,\mathrm{d}t.
\end{align*}
Enlarge the integral to $[0,T]$.  Jensen's inequality and
~(\ref{app:eq:combined-energy-closure}) give
\begin{align*}
 \overline a^{\gamma/2}
 &\leq[18M_T+80\mathcal D_\gamma(h)]^{\gamma/2}
       +44^{\gamma/2}\overline e^{\gamma/2}\\
 &\leq[18M_T+80\mathcal D_\gamma(h)]^{\gamma/2}
       +7(1+\overline e).
\end{align*}
The resulting count bound uses the deterministic coefficients
\begin{align*}
 J_{T,h}
 &:=1+\omega\{[18M_T+80\mathcal D_\gamma(h)]^{\gamma/2}+7\}
       +\kappa\bar s_T^{2},\\
 \chi_{T,h}&:=7\omega+3\kappa\bar s_T^{2}.
\end{align*}
Combining the preceding bound with~(\ref{app:eq:combined-noise-closure}) proves
\begin{equation}
 \E N_{T,m}
 \leq1+\frac{T}{h}\{J_{T,h}+\chi_{T,h}\overline e\}.
 \label{app:eq:combined-count-feedback}
\end{equation}
Substituting~(\ref{app:eq:combined-count-feedback}) into
~(\ref{app:eq:combined-error-count}) yields the scalar feedback inequality
\begin{equation*}
 \overline e\leq\rho_b\kappa^{-1}
 \{h/T+J_{T,h}+\chi_{T,h}\overline e\}.
\end{equation*}
To state the bound on $\overline e$, define
\[
 E_{T,h}^{\mathrm{env}}:=\kappa^{-1}(h/T+J_{T,h}).
\]
If $2\rho_b\kappa^{-1}\chi_{T,h}\leq1$, the final term can be absorbed.
We verify this sufficient restriction for the theorem's batch size below.
It gives
\begin{equation}
 \overline e\leq2\rho_bE_{T,h}^{\mathrm{env}}.
 \label{app:eq:combined-error-envelope}
\end{equation}
The same condition gives
$\chi_{T,h}\overline e\leq h/T+J_{T,h}$.  Inserting this sharper consequence
into~(\ref{app:eq:combined-count-feedback}) gives
\begin{equation}
 \E N_{T,m}\leq2+\frac{2TJ_{T,h}}h.
 \label{app:eq:combined-capped-count}
\end{equation}
Under these sufficient restrictions, equations~(\ref{app:eq:combined-fi-closure}),
~(\ref{app:eq:combined-error-envelope}), and
~(\ref{app:eq:combined-capped-count}) prove the finite horizon bounds
uniformly in the truncation level $m$.

\subsubsection{Removing the iteration truncation and the bounded support restriction}

For each integer $m\geq1$, consider the process constructed above that follows
RA-LMC for $m$ iterations and then continues with the exact Langevin diffusion
if $\tau_m<T$. Keep the initial law and algorithm parameters fixed, and
construct these processes using the same initial state, Brownian motion,
and sequence of oracle samples.
Every finite sequence of updates is well defined because each robust estimate is finite
and each step size is positive.  On $\{\tau_m\leq T\}$, all first $m$
iterations start strictly before $T$.  Markov's inequality and
~(\ref{app:eq:combined-capped-count}) therefore give
\begin{equation*}
 \mathbb{P}(\tau_m\leq T)
 \leq\frac{\E N_{T,m}}{m}\longrightarrow0.
\end{equation*}
Because $\tau_m$ is increasing, the iteration times cannot accumulate before
$T$. The truncated and original paths agree on $[0,T]$ whenever $\tau_m>T$,
so their path laws converge in total variation.  The nonnegative estimation
error accumulated over the first $m$ updates increases to the full error,
and $N_{T,m}\uparrow N_T$.  Monotone convergence therefore passes the
error and count bounds to the original process. Lower semicontinuity of
relative FI at each time and Fatou's lemma pass the FI bound as well.
For the original process, let $k(t):=\max\{k:\tau_k\leq t\}$ for
$0\leq t<T$ and define its duration-weighted estimation error by
\[
 \mathcal E_T:=\frac1T\E\int_0^T
 \|F_{k(t)}+\nabla V(X_{\tau_{k(t)}})\|^2\,\mathrm{d}t.
\]
Then $\overline e=\mathcal E_T$ in the limiting bounds.
Convexity of relative Fisher information yields
\[
 \FI(\bar\mu_T\|\pi)
 \leq\frac1T\int_0^T\FI(\mu_t\|\pi)\,\mathrm{d}t.
\]

It remains to remove the temporary bounded-support condition on $\mu_0$.
Let $\mu_0^{(R_0)}$ be $\mu_0$ conditioned on
$\{\|X_0-x_{\mathrm{ref}}\|\leq R_0\}$.  Then
$\mu_0^{(R_0)}\to\mu_0$ in total variation and
\begin{equation}
 \KL(\mu_0^{(R_0)}\|\pi)\longrightarrow K_0.
 \label{app:eq:combined-initial-cutoff}
\end{equation}
For $B>0$, conditioning on this ball also gives
$\E_{\mu_0^{(R_0)}}\|X_0-x_{\mathrm{ref}}\|^2\leq M_{2,0}$.
For $B=0$, no bound on the expected squared distance from
$x_{\mathrm{ref}}$ that is uniform in $R_0$ is needed. The noise scale is
constant, so the squared distance bound is not used.
Keep $T,h,b,\kappa,M_T,S_T$ and all proof coefficients computed from the
original declared parameters fixed.  For all sufficiently large $R_0$,
~(\ref{app:eq:combined-initial-cutoff}) implies
\[
 \KL(\mu_0^{(R_0)}\|\pi)/T\leq K_0/T+1.
\]
To see why this suffices without enlarging the schedules, retain the
actual initial KL in the derivation of the bounds
~(\ref{app:eq:combined-energy-closure}) on $\overline g$ and $\overline a$.
Writing
$K_{0,R_0}:=\KL(\mu_0^{(R_0)}\|\pi)$ gives
\begin{align*}
 \overline g
 &\leq9(C_\gamma+K_{0,R_0}/T)
       +40\mathcal D_\gamma(h)+21\overline e,\\
 \overline a
 &\leq18(C_\gamma+K_{0,R_0}/T)
       +80\mathcal D_\gamma(h)+44\overline e.
\end{align*}
The added one in $M_T=1+C_\gamma+K_0/T$ absorbs the cutoff's
KL difference. Thus the original $R_{T,h}\leq S_T$ and the same noise,
step-size, and count bounds apply to every sufficiently large cutoff,
without recomputing any algorithm parameter. The finite-update laws depend
on the initial law
through a probability kernel, so total-variation convergence of the initial
laws passes to every finite sequence of updates.  In particular, the uniform
bound on $\mathbb P(\tau_m\leq T)$ passes to the original initialization and
rules out accumulation there.  The error integrals and update counts are
nonnegative path functionals under the same algorithm.  Their conditional
expectations converge to the original expectations by monotone convergence
on the increasing initial balls, followed by division by the ball
probability.  Thus the estimation error itself converges, which allows the
FI inequality to retain the actual $\mathcal E_T$ on its right-hand side.
Lower semicontinuity of FI and
$\KL(\mu_0^{(R_0)}\|\pi)\to K_0$ now prove the
finite horizon statements for $\mu_0$.  Finally, every occupied stochastic
interval uses
exactly $4b$ individual oracle evaluations, so
$Q_T=4bN_T$ pathwise and $\E Q_T=4b\E N_T$.

\subsubsection{Accuracy schedule and general complexity}

We now verify the parameter choices in the theorem. For $0<\varepsilon\leq1$,
the prescribed time horizon and base step size are
\begin{equation*}
 T:=\frac{8(K_0+1)}{\varepsilon} \quad \text{and} \quad
 h:=\max\{u\in[0,1]:
              \mathcal D_\gamma(u)\leq\varepsilon/64\}.
\end{equation*}
With $M_T$ and $\kappa=(1+\bar s_T^2)^{-1}$ as above, the batch size is
\[
 b:=\left\lceil\left[
 \frac{260C_p\kappa^{-1}(1+\omega M_T^{\gamma/2})}{\varepsilon}
 \right]^{p/[2(p-1)]}\right\rceil.
\]
By continuity and
monotonicity of the coefficient functions, the maximal $h$ is
positive.  Since
\[
 \mathcal D_\gamma(h)
 \geq\widetilde{U}_\gamma(h)+2\widetilde{V}_\gamma(h),
\]
the choice $\mathcal D_\gamma(h)\leq\varepsilon/64$ with
$0<\varepsilon\leq1$ enforces the discretization stability condition.
Moreover, $18M_T+80\mathcal D_\gamma(h)\leq20M_T$, $M_T\geq1$, and
$\kappa\bar s_T^2\leq1$ imply
\begin{align*}
 J_{T,h}
 &\leq2+(\sqrt{20}+7)\omega M_T^{\gamma/2}
 \leq12(1+\omega M_T^{\gamma/2}),\\
 \chi_{T,h}
 &\leq3+7\omega\leq7(1+\omega M_T^{\gamma/2}).
\end{align*}
Since $h/T\leq1/8$, it follows that
\[
 E_{T,h}^{\mathrm{env}}
 \leq13\kappa^{-1}(1+\omega M_T^{\gamma/2}).
\]
The specified batch size gives
\[
 \rho_b\leq
 \frac{\varepsilon}{260\kappa^{-1}(1+\omega M_T^{\gamma/2})}.
\]
Consequently,
\begin{equation*}
 2\rho_b\kappa^{-1}\chi_{T,h}\leq\frac{14\varepsilon}{260}<1,
 \qquad
 2\rho_bE_{T,h}^{\mathrm{env}}
       \leq\frac{\varepsilon}{10}.
\end{equation*}
All sufficient restrictions used above therefore hold, and
$\mathcal E_T\leq\varepsilon/10$. The FI bound becomes
\begin{align*}
 \FI(\bar\mu_T\|\pi)
 &\leq\overline{\mathrm{FI}}_T
 \leq\frac{2K_0}{T}+9\mathcal D_\gamma(h)+\frac{\varepsilon}{2}
 \\
 &\leq\frac{\varepsilon}{4}
      +\frac{9\varepsilon}{64}
      +\frac{\varepsilon}{2}
 =\frac{57\varepsilon}{64}<\varepsilon.
\end{align*}
Likewise, the count bound gives
\[
 \E N_T\leq2+\frac{24T}{h}(1+\omega M_T^{\gamma/2}),
 \qquad \E Q_T=4b\E N_T.
\]
This proves the theorem.

For the complexity corollary, we retain parameter-sensitive estimates.
Define the coefficient of the term linear in $h$ by
\[
 D_\gamma^{(1)}:=
 \begin{cases}
  dL_0^2,&L_1=0,\\
  \begin{aligned}
   3d\bigl[L_0^2(1+2^{\gamma^2/(1-\gamma)})^2+3^{2\gamma}2^{2\gamma^2/(1-\gamma)}L_1^2(2+C_\gamma)\bigr],
  \end{aligned}
       &L_1>0,\ 0<\gamma<1,\\
  4(d+1)\{L_0^2+L_1^2(1+C_1)\},
       &L_1>0,\ \gamma=1.
 \end{cases}
\]
The explicit coefficient formulas imply
\begin{equation*}
 \lim_{u\downarrow0}\frac{\mathcal D_\gamma(u)}u
 =D_\gamma^{(1)}.
\end{equation*}
The limit is positive whenever a normalizable target satisfying the stated
assumptions exists.  Thus, for fixed problem parameters,
$h^{-1}=\mathcal O(D_\gamma^{(1)}/\varepsilon)$.
The hidden constant in this step-size estimate and the iteration estimate
below depends only on $\gamma$. Constants in the batch-size and oracle
estimates may additionally depend on $p$. The small accuracy thresholds may
depend on the other fixed problem parameters.
To bound $J_{T,h}$ and $S_T$, define
\begin{align*}
 J_\gamma^*&:=2+\omega\{[22(1+C_\gamma)]^{\gamma/2}+7\},\\
 \mathsf R_*^2&:=6M_{2,0}+64(K_0+1)d
       +5760(K_0+1)^2(1+C_\gamma).
\end{align*}
The schedule gives
\begin{align*}
 M_T\leq\frac98(1+C_\gamma),\quad 18M_T+80\mathcal D_\gamma(h)\leq22(1+C_\gamma),\quad S_T\leq\mathsf R_*^2/\varepsilon^2 \quad,J_{T,h}\leq J_\gamma^*.
\end{align*}
Consequently, we have
\begin{equation*}
 1+\omega M_T^{\gamma/2}\leq J_\gamma^*, \quad \text{and}\quad
 \bar s_T^{\,2}\leq[B(\mathsf R_*/\varepsilon)^p+G]^{2/p}.
\end{equation*}
Substitution into the prescribed batch size therefore gives
\begin{align*}
 b
 =\mathcal O\!\left(
 (J_\gamma^*)^{p/[2(p-1)]}
 \varepsilon^{-p/[2(p-1)]}
 [1+B(\mathsf R_*/\varepsilon)^p+G]^{1/(p-1)}
 \right).
\end{align*}
Multiplying by the expected iteration bound proves
\begin{align}
 \E N_T
 &=\mathcal O\!\left(
 (K_0+1)D_\gamma^{(1)}J_\gamma^*\varepsilon^{-2}
 \right),
 \label{app:eq:combined-general-iterations}\\
 \E Q_T
 &=\mathcal O\!\left(
 \frac{(K_0+1)D_\gamma^{(1)}
       (J_\gamma^*)^{1+p/[2(p-1)]}}
      {\varepsilon^{2+p/[2(p-1)]}}
 [1+B(\mathsf R_*/\varepsilon)^p+G]^{1/(p-1)}
 \right).
 \label{app:eq:combined-general-oracles}
\end{align}
These estimates will give the stated complexity specializations.

\subsection{Proof of Corollary~\ref{cor:combined-special-cases}}

We derive the stated complexity specializations. When
$B>0$ is fixed, the final factor in
~(\ref{app:eq:combined-general-oracles}) is, for sufficiently small
$\varepsilon$,
\[
 \mathcal O\!\left(
 B^{1/(p-1)}\mathsf R_*^{p/(p-1)}
 \varepsilon^{-p/(p-1)}\right).
\]
Therefore
\begin{align}
 \E Q_T
 =\mathcal O\!\left(
 \frac{B^{1/(p-1)}(K_0+1)D_\gamma^{(1)}
       (J_\gamma^*)^{1+p/[2(p-1)]}
       \mathsf R_*^{p/(p-1)}}
      {\varepsilon^{2+3p/[2(p-1)]}}
 \right).
 \label{app:eq:combined-polynomial-oracles}
\end{align}
If $B=0$, instead,
\begin{align}
 \E Q_T
 =\mathcal O\!\left(
 \frac{(1+G)^{1/(p-1)}(K_0+1)D_\gamma^{(1)}
       (J_\gamma^*)^{1+p/[2(p-1)]}}
      {\varepsilon^{2+p/[2(p-1)]}}
 \right).
 \label{app:eq:combined-uniform-oracles}
\end{align}
In particular, for $p=2$ and fixed remaining parameters, these bounds give
$\E Q_T=\mathcal O(\varepsilon^{-5})$ under BG-0 noise ($B>0$) and
$\E Q_T=\mathcal O(\varepsilon^{-3})$ under uniformly bounded variance
($B=0$).

We next make the dimension and smoothness dependence explicit.  Let
$\Lambda=\max\{1,L_0,L_1\}$.  For fixed $0<\gamma<1$, the definitions of
the generalized-smoothness coefficients imply the following bounds, with
hidden constants depending only on $\gamma$:
\begin{align}
 1+C_\gamma
 &=\mathcal O((d\Lambda)^{2/(2-\gamma)}),
 \notag\\
 D_\gamma^{(1)}
 &=\mathcal O(d\Lambda^2(1+C_\gamma)),
 \notag\\
 J_\gamma^*
 &=\mathcal O(
 d^{\gamma/(2-\gamma)}\Lambda^{2/(2-\gamma)}),
 \label{app:eq:combined-geometry-orders}\\
 \mathsf R_*
 &=\mathcal O\!\left(
 \sqrt{M_{2,0}}+\sqrt{(K_0+1)d}
 +(K_0+1)(d\Lambda)^{1/(2-\gamma)}
 \right).
 \notag
\end{align}
Under the convention $K_0=\mathcal O(d)$ stated in
Section~\ref{sec:preliminaries}, substituting
~(\ref{app:eq:combined-geometry-orders}) into
~(\ref{app:eq:combined-general-iterations}) gives
\begin{equation*}
 \E N_T
 =\mathcal O\!\left(
 \frac{\Lambda^{(8-2\gamma)/(2-\gamma)}
       d^{(6-\gamma)/(2-\gamma)}}{\varepsilon^2}
 \right).
\end{equation*}

For the compact $B>0$ oracle rate, we can sharpen the radius estimate in the
small-$\varepsilon$ regime.  For each fixed finite $M_{2,0}$,
$M_T\to1+C_\gamma$, so
\begin{equation}
 \varepsilon^2S_T
 \longrightarrow5120(K_0+1)^2(1+C_\gamma).
 \label{app:eq:combined-radius-order}
\end{equation}
In particular, for all sufficiently small $\varepsilon$,
\[
 \sqrt{S_T}
 \leq\frac{\widetilde R_*}{\varepsilon},\qquad
 \widetilde R_*:=80(K_0+1)\sqrt{1+C_\gamma}
 =\mathcal O\!\left(d(d\Lambda)^{1/(2-\gamma)}\right).
\]
Repeating the noise-envelope estimate with $\widetilde R_*$ in place of
$\mathsf R_*$ and using~(\ref{app:eq:combined-geometry-orders}) gives
\begin{align*}
 \E Q_T=\mathcal O\!\bigg(\frac{
 B^{1/(p-1)}
 \Lambda^{\frac{(8-2\gamma)(p-1)+2p}
                  {(p-1)(2-\gamma)}}
 d^{\frac{(3p-2)(6-\gamma)}{2(p-1)(2-\gamma)}}}{
 \varepsilon^{2+3p/[2(p-1)]}}
 \bigg).
\end{align*}
This is the compact $B>0$ rate in
Corollary~\ref{cor:combined-special-cases}. Its small-accuracy threshold may
depend on $M_{2,0}$, $B$, $G$, and the other fixed problem parameters.
Thus it does not replace the explicit initialization dependence in the
finite-accuracy schedule.  In particular, finite $K_0$ alone does not imply
the finite initial second moment used when $B>0$.
For $B=0$, the corresponding substitution in
~(\ref{app:eq:combined-uniform-oracles}) gives
\begin{align*}
 \E Q_T=\mathcal O\!\bigg(\frac{
 (1+G)^{1/(p-1)}
 \Lambda^{\frac{(8-2\gamma)(p-1)+p}
                  {(p-1)(2-\gamma)}}
 d^{1+\frac{8(p-1)+\gamma p}
               {2(p-1)(2-\gamma)}}}{
 \varepsilon^{2+p/[2(p-1)]}}
 \bigg).
\end{align*}
With the constant in $K_0=\mathcal O(d)$ fixed, the hidden constants in
these compact oracle formulas depend only on $p$ and $\gamma$. They are
pointwise in $\gamma$ and are not uniform as $\gamma\uparrow1$.
At $\gamma=1$, direct use of
the corresponding coefficients gives the same displayed powers evaluated at
$\gamma=1$.

When $L_1=0$, $\omega=0$,
$D_\gamma^{(1)}=dL_0^2$, and $J_\gamma^*=2$.  Thus
\begin{equation*}
 \E N_T
 =\mathcal O\!\left(
 \frac{(K_0+1)dL_0^2}{\varepsilon^2}\right)
 =\mathcal O\!\left(
 \frac{d^2L_0^2}{\varepsilon^2}\right)
\end{equation*}
under $K_0=\mathcal O(d)$.  The first equality concerns the noise-aware
RA-LMC rule and its iteration count can remain random. This proves
Corollary~\ref{cor:combined-special-cases}.

\subsection{Verification of the algorithmic reductions}

When $L_1=0$, $\omega=0$. If one additionally sets $\kappa=0$, the
adaptive duration in~(\ref{eq:robust-lmc-adaptive-iterates}) becomes the
constant $h$, and the algorithm is the fixed-step robust method of
Theorem~\ref{thm:full-robust}. This changes the positive value of $\kappa$
prescribed by Theorem~\ref{thm:full-combined}, so the sharper fixed-step
bounds follow from Theorem~\ref{thm:full-robust} and
Corollary~\ref{cor:robust-special-cases}, rather than by substituting into
the combined theorem's accuracy schedule.

If $B=G=0$, the conditional $p$th moment of every oracle error is zero.
Hence every oracle call made by the algorithm equals $\nabla V$ almost
surely. The robust-estimation stage is then redundant. Using one exact
gradient per iteration and the schedule in Theorem~\ref{thm:full-adaptive}
gives the adaptive exact-gradient method and its sharper bounds. When
$L_1=0$ as well, Corollary~\ref{cor:adaptive-special-cases} gives the
classical globally smooth complexity
$\mathcal O(L_0^2d^2/\varepsilon^2)$.

\clearpage
\section{Additional Experimental Details}
\label{app:experiments}

\subsection{Synthetic Experiments}
\label{app:synthetic-setting}

\subsubsection{Experimental Setting}
The potentials used in the synthetic experiments and their geometry is provided in Table~\ref{tab:synthetic-potentials}. All three corresponding target distributions are non-log-concave. 
\begin{table}[htbp]
  \centering
  \begin{tabular}{lll}
    \hline
    Setting & Potential $V(x)$ & Geometry \\
    \hline
    State-dependent heavy tails
      & $\norm{x}^2/2-\log\cosh(2x_1)$ & Globally $3$-smooth \\
    Generalized smoothness
      & $(\norm{x}^2-1)^2/4$ & $\gamma=2/3$ \\
    Combined setting
      & $\norm{x}^3/3-\norm{x}^2/2$ & $\gamma=1/2$ \\
    \hline
  \end{tabular}
  \caption{Target potentials used in the experiments of Fig.~\ref{fig:synthetic-fi} and the supplementary figures. We note that the two radial potentials are not globally smooth.}
  \label{tab:synthetic-potentials}
\end{table}
For the experiments with state-dependent heavy-tailed noise oracle, we define the oracle as follows
\begin{equation*}
  \gG(x, \xi)=\nabla V(x)+s_p(x)\xi,\qquad
  s_p(x)=\bigl(8+8\norm{x}^{p}\bigr)^{1/p},\qquad
  p=1.2,\quad \nu=1.4.
\end{equation*}
Here $\xi$ is a normalized spherical Student random vector. Specifically,
if $Z\sim\mathcal N(0,I_2)$ and $W\sim\chi^2_\nu$ are independent, then
\begin{equation*}
  \xi=\frac{Z}{\sqrt{W/\nu}\,m_{p,\nu}^{1/p}},\qquad
  m_{p,\nu}=\nu^{p/2}
    \frac{\Gamma(1+p/2)\Gamma((\nu-p)/2)}{\Gamma(\nu/2)},
\end{equation*}where $\Gamma(x)= \int_0^\infty t^{x-1}e^{-1}dt$ is gamma function defined for $x>0$ and $\chi^2_\nu$ is chi-squared distribution with $\nu$ degrees of freedom. Consequently, $\E\xi=0$, $\E\norm{\xi}^p=1$, and
$\E\norm{Y(x)-\nabla V(x)}^p=8+8\norm{x}^p$.
Since $1<p<\nu<2$, the oracle is unbiased with a finite $p$-th moment and
infinite variance, and satisfies the BG-$p$ condition with $B=G=8$ and $x_{\mathrm{ref}}=0$.

LMC uses one exact-gradient evaluation per update. Both SG-LMC and R-LMC use
$16$ independent oracle queries per update. SG-LMC averages all $16$ queries.
R-LMC averages $b=4$ robust replicates: each replicate uses the geometric
median of three queries as its center and clips the residual of a fourth
query at radius $s_p(x)4^{1/p}$. Thus the raw and robust stochastic methods
have the same per-update oracle budget, although the adaptive methods may
require different numbers of updates to reach a given physical time. We use a
fixed time horizon $T$ for the primary comparison since it matches our
theoretical $\mathrm{FI}$ guarantees for the time-averaged law $\Bar{\mu}_T$.
We also report time-averaged FI after a fixed number $N$ of updates as a
complementary comparison at equal iteration budgets.

For Figure~\ref{fig:synthetic-fi}, we use
$h=0.02$ in panels (a) and (b), and $h=0.16$
in panel (c). Here $h$ is the step size for
fixed-step methods and the base step-size
parameter for adaptive methods. We run the experiments in this section on CPU. 

An update with Langevin force $F_k$ is
$X_{k+1}=X_k+H_kF_k+\sqrt{2H_k}\,Z_k$.
The state-dependent-heavy-tail setting compares LMC, SG-LMC, and R-LMC, all
with fixed step-size $H_k=h$. The generalized-smoothness setting compares fixed-step
LMC with A-LMC, using the adaptive step-size
\begin{equation*}
  H_k=\frac{h}{1+\norm{\nabla V(X_k)}^{2/3}}.
\end{equation*}
We note that this is random variable as it depends on $X_k$. The combined setting compares all six methods. A-LMC uses
$H_k=h/(1+\norm{\nabla V(X_k)}^{1/2})$, whereas A-SG-LMC and RA-LMC use
\begin{equation*}
  H_k=\frac{h}{1+\norm{F_k}^{1/2}+\kappa s_p(X_k)^2},\qquad
  \kappa=\frac{1}{1+8^{2/p}}=\frac{1}{33},
\end{equation*}which follows~(\ref{eq:robust-lmc-adaptive-iterates}) for $p=1.2$. A-SG-LMC uses the negative of the $16$-query minibatch average as $F_k$, whereas RA-LMC uses the negative of the robust estimator with batch size $b=4$.

\begin{figure}[htbp]
  \centering
  \includegraphics[width=\textwidth]{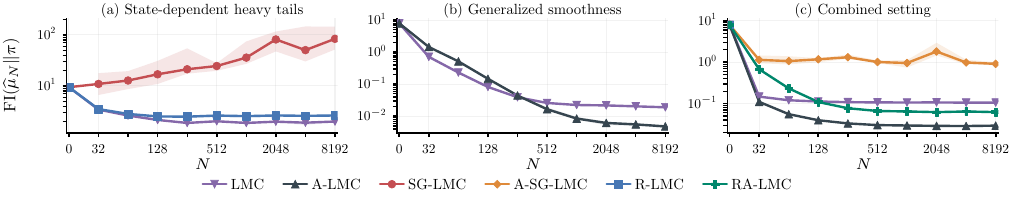}
  \caption{Fisher information of the time-averaged laws $\Bar{\mu}_{\tau_N}$ after \(N\) updates for the three settings and step sizes of Figure~\ref{fig:synthetic-fi}. For each trajectory, we sample a time uniformly from \([0,\tau_N]\), where \(\tau_N\) is the time reached after its \(N\)th update and may vary across trajectories for adaptive methods. }
  \label{fig:synthetic-occupation-iterations}
\end{figure}

\subsubsection{Calculation of Fisher information}
\label{app:synthetic-fi}

At a deterministic time $T$, the time-averaged law $\Bar{\mu}_T$ is the law of interpolation at an independently sampled uniform time in $[0,T]$, and $\mu_T$ is the terminal law at $T$. In our theoretical results, we prove nonasmyptotic convergence guarantees to achieve $\mathrm{FI}(\Bar{\mu}_T\|\pi)\leq \varepsilon$. Therefore our primary results show comparison of $\mathrm{FI}$ convergence of different LMC methods with respect to $\Bar{\mu}_T$. To estimate $\Bar{\mu}_T$, for each 8192 trajectory generated by an LMC method after time $T$, we randomly sample $t$ uniformly from $[0,T]$, and find the interval such that $t\in[t_k, t_{k+1})$ for $k\in\{0,\ldots,N\}$. Then we use the interpolation argument to obtain
\begin{equation*}
    X_t = X_{t_k} + uF_k + \sqrt{2}(W_{t_k + u} - W_{t_k}).
\end{equation*}Conditional on $X_k$ and $u$, $X_{t_k+u} \sim \gN(X_k + uF_k, 2uI_2)$. Then each $X_t$ sampled from a trajectory provides a Gaussian component with a known mean and covariance. Using this fact, the density of $\Bar{\mu}_T$ is approximated by the following density
\begin{equation}\label{eq:estimated-fi-density}
    p(x) = \frac{1}{8192}\sum_{i=1}^{8192}\gN(x; m_i, 2u_iI_2), \quad m_i = X_{k_i} + u_iF_{k_i},
\end{equation}where $k_i$ denotes the iteration index corresponding to the interval $t_i\in[t_{k_i}, t_{k_i+1})$ and $t_i$ is a sample uniformly drawn from $[0,T]$ for $i$-th trajectory. To ensure numerical stability for the narrow finite-mixture components, we smooth both the simulated law and the target law with $\gN(0,0.08I_2)$ and compute the Fisher information at each iteration or time $T$ according to~(\ref{eq:estimated-fi-density}) and the target density. For each LMC method, we apply this procedure at times $T\in\{0.25,0.5,1,2,4,8,16,32\}$. For plots against iteration count, we use $N\in\{32,64,128,256,512,1024,2048,4096,8192\}$. The point at $T=0$ shows the initial FI with common Gaussian initialization. 


\subsubsection{Additional Results}
\label{app:synthetic-additional-results}

Each curve shows the median FI across ten seeds, with the shaded band indicating the interquartile range. Some curves end early because at least one seed does not yield a valid FI estimate at later checkpoints due to numerical instability, a simulation limit, or a failed FI calculation. Figure~\ref{fig:synthetic-fi} compares methods at the same physical time \(T\), whereas Figure~\ref{fig:synthetic-occupation-iterations} compares them after the same number \(N\) of updates with $h=0.02,0.02,0.16$. At fixed \(N\), adaptive methods may reach different physical times, but SG-LMC and R-LMC have equal oracle-query budgets because they use the same number of queries per update. In panel (a), R-LMC closely follows exact-gradient LMC, while SG-LMC has substantially higher FI at larger \(N\). Panel (b) considers exact-gradient sampling from a target whose curvature can grow without bound; over the observed horizon, A-LMC reaches lower FI than fixed-step LMC. In panel (c), SG-LMC and R-LMC diverge due to numerical instability, while A-SG-LMC converges to a suboptimal FI value. RA-LMC, by contrast, converges to a lower FI than LMC and follows A-LMC, illustrating the benefit of combining robust estimation with adaptive step sizes under a fixed iteration budget.

To complement Figures~\ref{fig:synthetic-fi} and~\ref{fig:synthetic-occupation-iterations}, we repeat the three experiments with larger step-size parameters $h=0.04$, $0.04$, and $0.32$ in panels (a)--(c), respectively, while keeping the targets, oracles, methods, seeds, and FI calculations unchanged. Although A-LMC, A-SG-LMC, and RA-LMC use adaptive steps, their step sizes $H_k$ scale with $h$ at a given state, so increasing $h$ increases their steps as well. Figures~\ref{fig:synthetic-larger-steps-time} and~\ref{fig:synthetic-larger-steps-iterations} report the FI of the resulting time-averaged laws relative to the target against time $T$ and iteration count $N$, respectively.

\begin{figure}[htbp]
  \centering
  \includegraphics[width=\textwidth]{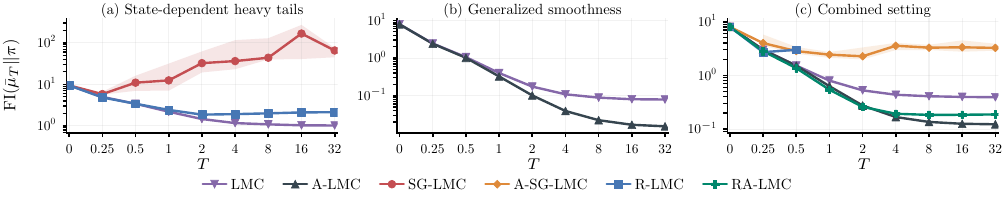}
  \caption{Fisher information of the time-averaged laws $\Bar{\mu}_T$ across the three settings of Figure~\ref{fig:synthetic-fi}, using larger step sizes $h=0.04$, $0.04$, and $0.32$ in panels (a)--(c), respectively. For each trajectory, a time is sampled uniformly from $[0,T]$ to estimate the time-averaged law.}
  \label{fig:synthetic-larger-steps-time}
\end{figure}

\begin{figure}[htbp]
  \centering
  \includegraphics[width=\textwidth]{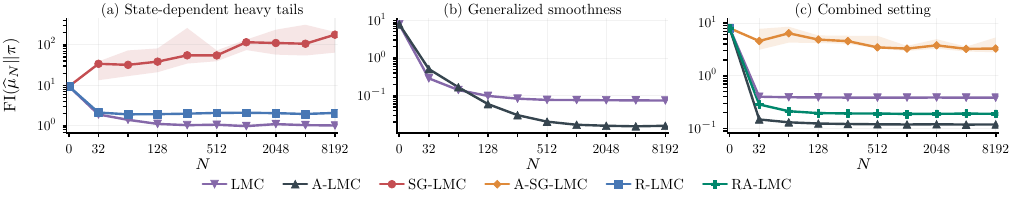}
  \caption{Fisher information of the time-averaged laws $\Bar{\mu}_{\tau_N}$ after $N$ updates for the same settings and step sizes as Figure~\ref{fig:synthetic-larger-steps-time}. For each trajectory, a time is sampled uniformly from $[0,\tau_N]$, where $\tau_N$ is the time reached after its $N$th update and may vary across trajectories for adaptive methods.}
  \label{fig:synthetic-larger-steps-iterations}
\end{figure}

Figures~\ref{fig:synthetic-larger-steps-time} and~\ref{fig:synthetic-larger-steps-iterations} shows the same trends in Figure~\ref{fig:synthetic-fi} and Figure~\ref{fig:synthetic-occupation-iterations} at larger step sizes. In the heavy-tailed setting, SG-LMC’s estimated FI rises sharply, while R-LMC remains much closer to exact-gradient LMC. Under generalized smoothness, A-LMC achieves lower FI than fixed-step LMC at later iterations. In the combined setting, RA-LMC remains close to A-LMC and achieves lower FI than fixed-step LMC and A-SG-LMC, whereas the fixed-step stochastic LMC methods become numerically unstable, leaving their curves absent or truncated. The fixed-\(N\) results complement the fixed-\(T\) comparison by showing that these qualitative advantages persist at matched iteration budgets.

\subsection{Bayesian Fine-Tuning Experiments}
\label{app:assumption-verification}

\subsubsection{Experimental Setting}
\label{app:oracle-trajectory}

This subsection uses two separate fine-tuning protocols. To study the
BG-$p$ oracle, we generate trajectories using minibatch-mean gradients
and AdamW, and make all histogram and moment measurements along these
trajectories. For the baseline comparison, empirical RA-LMC retains
AdamW and the Gaussian perturbation but replaces the minibatch mean with
our geometric-median-centered, residual-clipped estimate. Thus
(\ref{app:eq:oracle-adamw-update}) defines the trajectories used for the
oracle evidence, whereas (\ref{app:eq:empirical-ra-update}) defines the
RA-LMC runs reported in Table~\ref{tab:obqa-slake-sensitivity}.

\paragraph{Shared model and optimizer settings.}
We fine-tune Qwen3-4B on OpenBookQA and Qwen3-VL-4B-Instruct on the
SLAKE yes/no subset, following the model and dataset settings in Bayesian
Adaptation Gym~\citep{samplawski2026bayesian}. The training splits contain
4,957 and 1,681 examples, respectively. We freeze each backbone and train
rank-8 LoRA factors on \texttt{q\_proj}, \texttt{v\_proj}, and
\texttt{lm\_head}, using LoRA scaling 16 and zero dropout. The vector
$w$ concatenates all trainable LoRA entries, and all distances and
gradient norms are unnormalized Euclidean norms over this vector.

Both protocols use 5,000 AdamW updates with moment coefficients
$(0.9,0.999)$, numerical stabilizer $10^{-6}$, and zero weight decay.
The scalar learning rate $\eta_k$ warms up over 300 updates to
$10^{-4}$ and then decreases linearly. We use AdamW-based updates rather
than a fixed step size because generalized smoothness permits the local
geometry and gradient scale to vary along a fine-tuning trajectory.
AdamW provides a practical adaptation mechanism without requiring the
smoothness and local noise-scale quantities in our theoretical
step-size rule. We run the fine-tuning and evaluation experiments
on single NVIDIA H100 GPU 40GB per run. 

\paragraph{Trajectories used for oracle evidence.}
To study the ordinary stochastic-gradient noise, these runs use minibatch-mean gradients rather than our
robust estimator.

At update $k$, the dataset is shuffled by epoch and the loss is the
average negative log-likelihood over the current minibatch. With
$b=4$, the update is
\begin{equation}
  w_{k+1}
  =\operatorname{AdamW}\!\left(
      w_k,\frac{1}{b}\sum_{j=1}^{b}
      \gG(w_k,\xi_{k,j});\eta_k
    \right)
    +\gamma\sqrt{2\eta_k}\,Z_k,
  \qquad Z_k\sim\mathcal N(0,I).
  \label{app:eq:oracle-adamw-update}
\end{equation}
Here $\xi_{k,j}$ indexes the $j$th example in the minibatch and
$\gG(w_k,\xi_{k,j})$ is its stochastic gradient. The noise discount factor
is selected as $\gamma=10^{-5}$ and is selected according to a grid search $\gamma\in\{10^{-6}, 10^{-5}, 10^{-4}, 10^{-3}\}$ over the validation sets, similar to the approach in~\citet{kim2023bayestune}. The isotropic Gaussian increment is added to the FP32 LoRA
parameters after the AdamW update. Figure~\ref{fig:llm-noise-histograms} uses the OpenBookQA fine-tuning trajectory,
while Figure~\ref{fig:llm-noise-first-moment} uses both fine-tuning trajectories.

\paragraph{RA-LMC runs used for the baseline comparison.}
For the empirical RA-LMC results in
Table~\ref{tab:obqa-slake-sensitivity}, each update draws 4 distinct
training examples and evaluates their gradients at the same parameter
vector. Let
\begin{equation}
  \begin{aligned}
    g_{k,j}&:=\gG(w_k,\xi_{k,j}),\qquad j=1,\ldots,4,\\
    C_k&:=\mathrm{GeoMed}(g_{k,1},g_{k,2},g_{k,3}),\\
    \lambda_k&:=c_{\mathrm{clip}}s_k,\\
    \widehat g_k
      &:=C_k+\mathrm{clip}_{\lambda_k}(g_{k,4}-C_k),
  \end{aligned}
  \label{app:eq:empirical-robust-estimator}
\end{equation}
where $\mathrm{clip}_{\lambda}(v)
=v\min\{1,\lambda/\|v\|\}$ for $v\neq0$ and
$\mathrm{clip}_{\lambda}(0)=0$. Thus $c_{\mathrm{clip}}$ is the tuned
multiplier, while $\lambda_k$ is the resulting clipping threshold at
update $k$. The geometric median is computed over the
entire LoRA-gradient vectors, rather than coordinatewise. Because the
center uses exactly three gradients, we do not solve the geometric-median
problem iteratively. Although the gradients are high-dimensional, any three gradient vectors lie in a common two-dimensional plane or on a line. Their geometric median can therefore be computed as the Fermat–Torricelli point of the triangle they form~\citep{mordukhovich2019fermat}. Let $a_j$ denote the
side length opposite the angle $\theta_j$ at $g_{k,j}$. If some
$\theta_j\geq 2\pi/3$, then $g_{k,j}$ is the geometric median. Otherwise, we
use the explicit barycentric construction
\[
  \omega_j\ \propto\ \frac{a_j}{\sin(\theta_j+\pi/3)},
  \qquad
  C_k=\sum_{j=1}^{3}\omega_j g_{k,j},
  \qquad
  \sum_{j=1}^{3}\omega_j=1.
\]
Computing $C_k$ therefore requires only three pairwise norms and one weighted
average. Its cost is $\mathcal O(d)$ for $d$ trainable LoRA parameters, with
no numerical optimization loop or stopping tolerance. The implementation
also handles repeated and collinear gradients explicitly.

The theoretical local $p$-moment noise scale depends on the unknown full
gradient and gradient-error distribution, so evaluating it at every update is prohibitively expensive. We therefore use $s_k$ as an online
empirical dispersion proxy. We initialize $s_0^2$ with the Bessel-corrected
dispersion of 128 distinct training-example gradients, providing a stable
initial clipping radius. After forming
$\bar g_k=\frac14\sum_{j=1}^4g_{k,j}$, we update the scale by
\begin{equation}
  s_{k+1}^2
  =\max\!\left\{10^{-20},
    0.95s_k^2+\frac{0.05}{3}
    \sum_{j=1}^{4}\|g_{k,j}-\bar g_k\|^2\right\}.
  \label{app:eq:empirical-noise-scale}
\end{equation}
This exponential moving average tracks changes in local gradient dispersion
while reusing the four gradients already required by the estimator, so it
introduces no additional oracle complexity. The current robust estimate in
(\ref{app:eq:empirical-robust-estimator}) uses the lagged value $s_k$, before
the scale update, which prevents a current outlier from increasing its own
clipping threshold. The tuned multiplier $c_{\mathrm{clip}}$ accounts for the
difference between this empirical proxy and the theoretical noise scale.
Empirical RA-LMC then applies
\begin{equation}
  w_{k+1}
  =\operatorname{AdamW}(w_k,\widehat g_k;\eta_k)
    +\gamma_{\mathrm{RA}}\sqrt{2\eta_k}\,Z_k,
  \qquad Z_k\sim\mathcal N(0,I).
  \label{app:eq:empirical-ra-update}
\end{equation}
The Gaussian increment is again added after AdamW and does not enter its
running moments. OpenBookQA uses
$(c_{\mathrm{clip}},\gamma_{\mathrm{RA}})
=(1.191,3\times10^{-6})$, while SLAKE uses
$(5.554,3\times10^{-4})$, which were selected by hyperparameter tuning on the validation sets using grid search. Since RA-LMC uses $4b$ stochastic gradients per update, we use $b=1$ for a fair comparison with other baselines. For each reported random seed, ten independent
runs of (\ref{app:eq:empirical-ra-update}) produce ten terminal LoRA
adapters whose class probabilities are averaged at test time. We follow a similar procedure for each baseline.

We also evaluated otherwise identical runs with a zero-mean Gaussian
prior on the trainable LoRA parameters but observed no meaningful
difference in performance, so the reported experiments omit an explicit
prior.

\paragraph{Gradient-error measurements.}
The remainder of this subsection considers only the experiments related to BG-$p$ oracle with trajectories
from (\ref{app:eq:oracle-adamw-update}). The heavy-tail and
state-dependence experiments both use the norm of the difference between a minibatch-mean gradient and the gradient calculated over the full training set at a fixed LoRA parameter vector. The heavy-tail experiment studies its distribution across minibatches at selected checkpoints. The state-dependence experiment instead studies how its conditional first moment changes along the trajectory. We first describe this common error measurement and then describe two minibatch-sampling procedures.

At each checkpoint, we hold $w_k$ fixed and place the model in evaluation
mode. Writing $\ell_r(w)$ for the classification loss of training example
$r$, we compute the gradient over the entire training set and the $i$th
minibatch-mean gradient as
\begin{equation}
  \begin{aligned}
    \nabla V(w_k)
      &=\frac{1}{n}\sum_{r=1}^{n}\nabla\ell_r(w_k),\\
    \widehat{\nabla V}_{i,k}
      &:=\frac{1}{b}\sum_{j=1}^{b}
          \gG(w_k,\xi_{k,i,j})
       =\frac{1}{b}\sum_{j=1}^{b}
          \nabla\ell_{\xi_{k,i,j}}(w_k).
  \end{aligned}
  \label{app:eq:oracle-reference}
\end{equation}
Here $\xi_{k,i,j}$ indexes the $j$th example in the $i$th measurement
minibatch at checkpoint $k$, and $b=4$. In both
(\ref{app:eq:oracle-adamw-update}) and
(\ref{app:eq:oracle-reference}), $\gG(w,\xi)$ denotes one stochastic
gradient, as in the preliminaries, while the minibatch estimate is
written explicitly as an average. Training also includes the final singleton.
The reference uses every training example, evaluated in
chunks of four. We weight each chunk's mean gradient by its number of
examples, include the final singleton, accumulate in FP64, and divide by
$n$. Thus the reference is computed from the full empirical objective,
up to numerical precision. For each measurement minibatch $i$ at checkpoint $k$, we record
\begin{equation}
  R_{i,k}=\norm{\widehat{\nabla V}_{i,k}-\nabla V(w_k)}_2.
  \label{app:eq:llm-noise-norm}
\end{equation}
Both gradients in~(\ref{app:eq:llm-noise-norm}) are measured before any
optimizer transformation or robustification. We add no Gaussian noise
to these gradients; the Gaussian increments affect only the trajectory
at which they are evaluated. No validation or test examples enter the
oracle measurements.
The histogram and moment analyses below use this same error norm,
but draw their measurement minibatches in different ways.

\subsubsection{Gradient-Error Histograms}
\label{app:oracle-tail-diagnostic}

We examine checkpoints at updates $0$, $1250$, $2500$, $3750$, and $5000$.
At each checkpoint, we partition a deterministic permutation of the
training examples into disjoint complete four-example minibatches.
OpenBookQA provides 1,239 complete minibatches from 4,957 training
examples, and SLAKE provides 420 from 1,681 examples.

To assess whether the observed gradient error histograms resemble a
theoretical heavy tailed model with an infinite second moment, at each
checkpoint we fit the following positive $\beta$-stable reference model to the
squared error norms and compare the empirical norm histograms with synthetic
norm histograms generated from the fit:
\begin{equation}
  R^2=a+cS_\beta,\qquad
  \E\bigl[e^{-tS_\beta}\bigr]=e^{-t^\beta},
  \qquad 0<\beta<1,\quad t\geq0,
  \label{app:eq:llm-positive-stable-surrogate}
\end{equation}
where $S_\beta$ is a standard positive $\beta$-stable random variable,
equivalently a one-sided L\'evy-stable random variable under the normalization in
(\ref{app:eq:llm-positive-stable-surrogate}). Since
$\E[S_\beta]=\infty$ for $0<\beta<1$, the reference model satisfies
$\E[R^2]=\infty$. The parameters $a$ and $c$ are a nonnegative location and
a positive scale, respectively. The fit minimizes the normalized 80-bin conditional multinomial
negative log-likelihood plus ten times the Bernoulli KL divergence
between the observed and fitted probabilities of the tallest observed
bin. Thus the objective gives additional weight to matching the
histogram peak. We use L-BFGS-B with multiple initializations, constrain
$\beta\in[0.2,0.98]$ and $c\in[10^{-12},10^5]$, and bound $a$ above by
$0.999$ times the smallest retained squared norm.

The synthetic panel draws $R=\sqrt{a+cS_\beta}$ from the fitted model
and retains draws in the same interval until their number equals the
retained empirical count. Both histograms use the same 80 equal-width
bins and divide bin counts by the retained count times bin width.
Although $R^2=a+cS_\beta$ follows a shifted and scaled positive $\beta$-stable
distribution, the plotted synthetic quantity is
$R=\sqrt{a+cS_\beta}$. Thus, the synthetic error norms are generated from a
fitted positive $\beta$-stable reference model for the squared error norms but are
not themselves stable random variables. For Figure~\ref{fig:llm-noise-histograms}, the fitted parameters are $\beta=0.2$, $a=0.0729963$, and $c=1.3290073$ with $\beta$ at its lower search bound.

\subsubsection{State Dependence of Gradient-Error Moments}
\label{app:oracle-first-moment}

For Figure~\ref{fig:llm-noise-first-moment}, we save 4,000 approximately
equally spaced checkpoints across the 5,000 update trajectory and compute
their distances $\norm{w_k-w_{5000}}_2$ from the final LoRA iterate, which is denoted as $\omega_{5000}$. We
place 50 equally spaced target distances between zero and the largest
observed distance, then select 50 distinct saved checkpoints by minimizing
the sum of absolute differences between their distances and these
targets. This one-to-one assignment gives approximately uniform
coverage in distance. Figure~\ref{fig:llm-noise-first-moment} includes the final iterate for both datasets. The selected states farthest from the final iterate correspond to updates 1 and 3 for OpenBookQA and SLAKE, respectively.

At each selected checkpoint, we approximate the exact gradient by computing the gradients across the entire training set and compute 256 independent minibatch gradients. Each minibatch consists of four examples sampled uniformly from the dataset. We partition the resulting error norms
in~(\ref{app:eq:llm-noise-norm}) into eight disjoint groups of 32 and plot
\begin{equation}
  \left(\norm{w_k-w_{5000}}_2,\ \widehat m_{1,k,j}\right),
  \qquad
  \widehat m_{1,k,j}
  =\frac{1}{32}\sum_{i=32(j-1)+1}^{32j} R_{i,k},
  \quad j=1,\ldots,8.
  \label{app:eq:llm-first-moment-estimate}
\end{equation}
This process was repeated for 50 checkpoints. The eight estimates
at a given horizontal coordinate quantify variation across independently
sampled minibatch groups. Each vertical coordinate averages
error norms, rather than taking the norm of an averaged error vector. The results in Figure~\ref{fig:llm-noise-first-moment} shows these estimates as scatter plots for OpenBookQA and SLAKE datasets. We report the first moment because it is well defined under every BG-$p$ model with $p>1$ and is less sensitive to extreme gradient errors than higher empirical moments. This provides a stable and interpretable view of how the typical gradient-error magnitude changes along the finite-tuning trajectory. The resulting plots in Figure~\ref{fig:llm-noise-first-moment} provide evidence of state dependence of the oracle noise.

\subsubsection{Additional Results}
\label{app:oracle-additional-results}

\paragraph{Histograms across checkpoints.}
Figures~\ref{fig:obqa-batch4-noise-five-stages} and
\ref{fig:slake-batch4-noise-five-stages} compare empirical and synthetic
histograms at five checkpoints during Qwen3-4B fine tuning on OpenBookQA
and Qwen3-VL-4B-Instruct fine tuning on SLAKE, respectively. We use the
procedure described in Subsection~\ref{app:oracle-tail-diagnostic}. In
both settings, the empirical error norm distributions closely resemble
the synthetic distributions generated from this fitted reference model.
This agreement becomes more pronounced at and
after step 1,250. These results provide further empirical support for
the observation in Figure~\ref{fig:llm-noise-histograms} that stochastic
gradient errors exhibit heavy tailed behavior.

\begin{figure}[htbp]
  \centering
  \includegraphics[width=\textwidth]{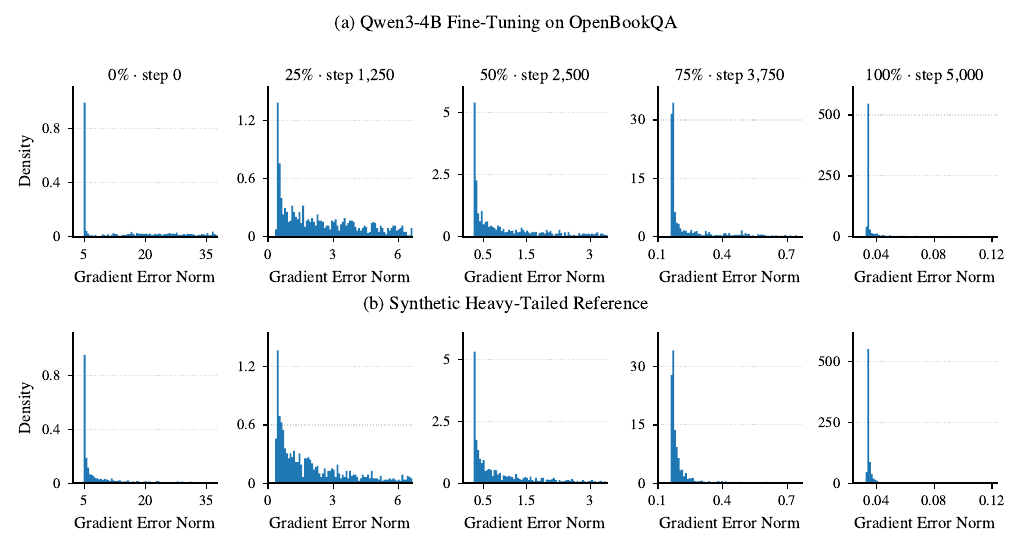}
  \caption{Minibatch gradient-error norm histograms across Qwen3-4B LoRA
    fine-tuning on OpenBookQA. Columns show steps $0$, $1250$, $2500$,
    $3750$, and $5000$. The top row contains empirical norms and the
    bottom row contains samples generated from the fitted positive $\beta$-stable
    reference model in~(\ref{app:eq:llm-positive-stable-surrogate}).
    Both rows use the same 80 bins and axes, and fitted $\beta$ values are
    $(0.200,0.200,0.200,0.607,0.837)$.}
  \label{fig:obqa-batch4-noise-five-stages}
\end{figure}

\begin{figure}[htbp]
  \centering
  \includegraphics[width=\textwidth]{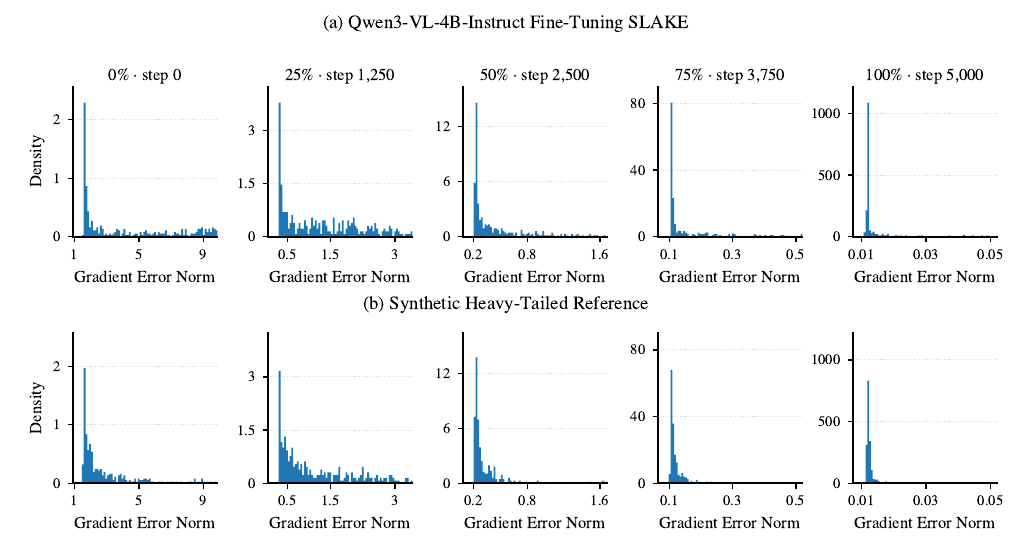}
  \caption{Minibatch gradient-error norm histograms across
    Qwen3-VL-4B-Instruct LoRA fine-tuning on SLAKE. The rows, columns,
    bins, and density normalization follow Figure~\ref{fig:obqa-batch4-noise-five-stages}.
    Fitted $\beta$ values are $(0.287,0.200,0.578,0.730,0.972)$.}
  \label{fig:slake-batch4-noise-five-stages}
\end{figure}

\paragraph{Gradient-error moments of different orders.}
We additionally repeat the state-dependence calculation for
$p\in\{1,1.25,1.5,1.75,2\}$. Using the same checkpoints and measurement
minibatches as in~(\ref{app:eq:llm-first-moment-estimate}), we compute
\begin{equation}
  \widehat m_{p,k,j}
  =\frac{1}{32}\sum_{i=32(j-1)+1}^{32j}R_{i,k}^{p},
  \qquad j=1,\ldots,8.
  \label{app:eq:llm-p-moment-estimate}
\end{equation}
Figures~\ref{fig:obqa-gradient-error-p-moments} and
\ref{fig:slake-gradient-error-p-moments} plot these estimates against the
distance from the final LoRA iterate. Across the displayed moment orders,
the empirical moments generally increase with distance, supporting a
state-dependent noise scale beyond the first-moment results in
Figure~\ref{fig:llm-noise-first-moment}.

\begin{figure}[htbp]
  \centering
  \includegraphics[width=\textwidth]{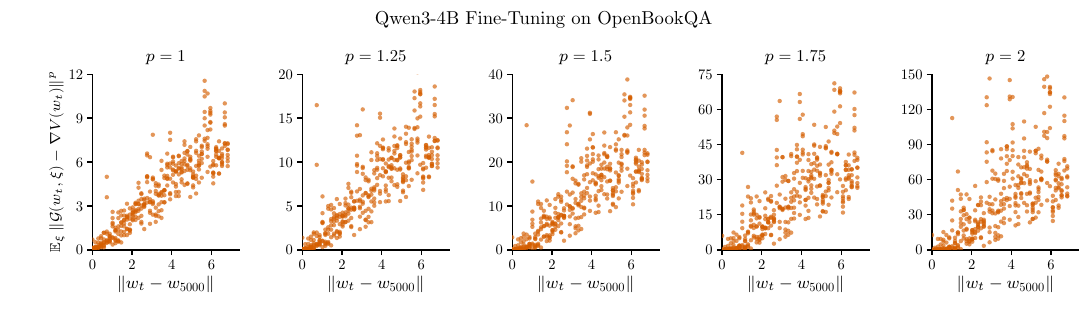}
  \caption{Empirical $p$-th moments of minibatch gradient-error norms versus
    distance from the final LoRA iterate during Qwen3-4B fine tuning on
    OpenBookQA. Panels report $p\in\{1,1.25,1.5,1.75,2\}$. Each panel contains
    eight estimates at each of 50 checkpoints, and every estimate averages
    $R_{i,k}^{p}$ over 32 independently sampled minibatch estimates with batch size of 4 as in~(\ref{app:eq:llm-p-moment-estimate}).}
  \label{fig:obqa-gradient-error-p-moments}
\end{figure}

\begin{figure}[htbp]
  \centering
  \includegraphics[width=\textwidth]{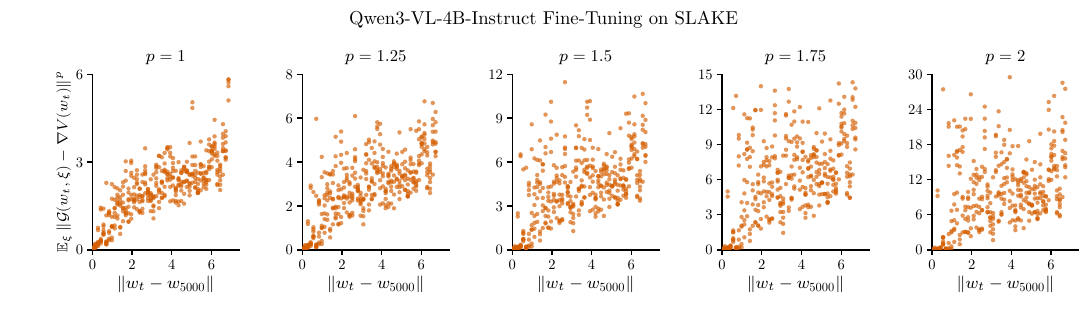}
  \caption{Empirical $p$-th moments of minibatch gradient-error norms versus
    distance from the final LoRA iterate during Qwen3-VL-4B-Instruct fine
    tuning on SLAKE. The moment orders, checkpoints, and grouped estimates
    follow Figure~\ref{fig:obqa-gradient-error-p-moments}.}
  \label{fig:slake-gradient-error-p-moments}
\end{figure}



\FloatBarrier

\clearpage
\section{Related Work}
\label{app:related-work}

\subsection{Non-Log-Concave Langevin Sampling}
\label{app:related-sampling}

Earlier non-log-concave sampling results often seek convergence in a
distributional metric such as KL divergence, total variation, or a
Wasserstein distance. Such guarantees describe global mixing, which
relative Fisher information (FI) alone does not ensure: a law can have
small FI while assigning incorrect mass to well-separated modes.
The stronger conclusion generally comes with assumptions on the target.
For example, analyses based on a log-Sobolev or Poincar\'e inequality
use that inequality to turn local dissipation into global convergence
\citep{vempala2019rapid,chewi2022analysis}. Other finite-time analyses
assume dissipativity or convexity outside a compact set
\citep{durmus2017nonasymptotic,cheng2018sharp,majka2020nonasymptotic},
or exploit an explicit mixture structure~\citep{lee2018beyond}.
These distributional guarantees are not weaker versions of FI
guarantees; they address a different, more global question under
additional structure.

Global gradient Lipschitzness is not universal in this literature.
For example, \citet{erdogdu2021convergence} derive finite-time KL bounds
with H\"older gradients, together with prescribed tail growth.
Tamed Langevin methods accommodate faster-growing drifts:
\citet{brosse2019tamed} impose polynomial local regularity and an
outward-pointing gradient at infinity, obtaining finite-time weighted
total-variation bounds; their Wasserstein result additionally assumes
global strong convexity. \citet{lytras2024tamed} obtain finite-time
distributional bounds under dissipativity and a Poincar\'e or
log-Sobolev inequality. For a neural-network objective,
\citet{lovas2023taming} use polynomial regularity and a
superquadratic regularizer to enforce dissipativity, together with
high-order data and initialization moments. Their guarantees concern
the Gibbs law of the regularized objective. Thus, relaxing global
smoothness has generally required other control of the tails or
the target geometry.

Uniformly bounded stochastic-gradient variance is not assumed by every
non-convex sampling analysis either. \citet{raginsky2017nonconvex}
allow the conditional variance to grow quadratically with the distance
from the origin and prove non-asymptotic Wasserstein and excess-risk
bounds for SGLD. Their component losses nonetheless have globally
Lipschitz gradients, the potential is dissipative, and the oracle has
a finite second moment at every state. Their result therefore covers
state dependence of a BG-0 type, but not the possibly infinite
variance permitted by BG-$p$ when $p<2$. In a different robustness
model, \citet{bhatia2019bayesian} use robust gradients to handle
adversarially contaminated data under global gradient Lipschitzness
and strong convexity. These are finite-time results, so the distinction
from our setting is in the assumptions and convergence criterion,
not simply asymptotic versus non-asymptotic analysis.

The FI framework of \citet{balasubramanian22a} avoids a functional
inequality and dissipativity by bounding stationarity of a
time-averaged law. Its direct LMC analysis assumes a globally
Lipschitz gradient, while its standard SG-LMC theorem uses a
state-uniform bound on squared gradient error. The same work gives
two important qualifications: Gaussian smoothing handles globally
H\"older gradients, and a finite-sum PAGE construction obtains an
FI guarantee without a state-uniform variance assumption when
component gradients are globally Lipschitz. Subsequent work studies
the meaning and complexity of FI guarantees, as well as other
algorithms in this framework
\citep{chewi2023fisher,cheng2023fast,chewi2026complexity,
sahin2026variance,sahin2026zerothorder}. These results establish
non-asymptotic progress without requiring full mixing, but the
aforementioned extensions do not provide a direct SG-LMC FI bound
under both gradient-dependent curvature and a state-dependent
finite-$p$-moment oracle. Our theorems address that setting while
retaining the more limited FI conclusion.

\subsection{Generalized Smoothness and Heavy-Tailed Noise in Optimization}
\label{app:related-optimization}

Generalized smoothness lets the local change in a gradient depend on
the gradient magnitude rather than a single global Lipschitz
constant. Under a symmetric generalized-smoothness condition,
\citet{chen2023generalized} derive descent bounds and use normalized
gradient descent for non-convex objectives; they analyze a
variance-reduced stochastic method as well. \citet{li2023generalized}
instead control gradients along the optimization trajectory, allowing
standard gradient methods to recover familiar rates under their
generalized-smoothness conditions. Their stochastic results permit
heavy-tailed distributions with finite variance, which is distinct
from assuming only a finite $p$-th moment for $p<2$.

For noise with a finite $p$-th moment and potentially infinite
variance, optimization methods often limit the influence of extreme
gradients. \citet{cutkosky2021high} use clipping, momentum, and
normalization to obtain high-probability stationarity guarantees.
Related clipped-gradient bounds are given by
\citet{nguyen2023improved}, while \citet{hubler2025gradient}
analyze normalized SGD under a finite central $p$-th moment.
State-dependent second-moment assumptions have motivated a different
set of techniques. \citet{alacaoglu2025towards} study algorithms
under quadratically growing variance for convex problems, and
\citet{fazla2026lower} use anchoring and regularization for
non-convex problems under BG-0 noise. Work addressing more than one
relaxation includes clipped SGD for convex $(L_0,L_1)$-smooth
objectives with heavy-tailed noise~\citep{chezhegov2025convergence}.
These papers target objective suboptimality or small gradient norm,
not the distribution of Langevin iterates.

This change of objective matters when importing an optimization
update into a sampler. The overdamped Langevin diffusion has drift
$-\nabla V$ and stationary density $\pi\propto e^{-V}$. If one
replaces that drift by a clipped or normalized vector field $b$ but
leaves the Brownian term unchanged, the adjoint of the new generator
acts on the target density as
\begin{equation}
\label{eq:related-clipped-drift}
  \mathcal L_b^*\pi
  =-\nabla\!\cdot\!\bigl(\pi(b+\nabla V)\bigr),
\end{equation}
which generally does not vanish. A descent argument for an
optimizer therefore does not show that its modified drift samples
from $\pi$. Tamed and adaptive sampling methods account for the
target law explicitly~\citep{brosse2019tamed,leroy2024adaptive}.

There are two further obstacles for our FI objective. First, the
KL-dissipation calculation involves a time integral of the
\emph{squared} error in the Langevin drift. An unbiased oracle with
only a finite $p$-th moment may have infinite squared error. Clipping
the entire stochastic gradient would control its norm but would
distort the target drift, whereas centering and clipping its residual
allow us to control estimation bias and squared error separately.
Second, a step chosen from a noisy gradient is correlated with its
estimation error. FI bounds then require control of the error
\emph{weighted by the duration of that step}; an ordinary
optimization descent bound does not provide this control. Our
adaptive algorithms use physical elapsed time, with the drift and
Brownian increment scaled by the same step duration. This is also why
an average over physical time must be distinguished from an average
over iteration indices.

\subsection{Bayesian Fine-Tuning and Experimental Baselines}
\label{app:baseline-methods}

Low-rank adaptation (LoRA) fine-tunes a model through trainable
low-rank factors while freezing the backbone~\citep{hu2022lora}.
Bayesian LoRA methods aim to represent uncertainty in these
additional parameters without maintaining a posterior over all
backbone weights. We use the reference implementations of nine
methods from Bayesian Adaptation Gym (BAG)
\citep{samplawski2026bayesian} in Table~\ref{tab:obqa-slake-sensitivity}.
For methods that produce multiple predictions, we average their
class probabilities before evaluating predictive metrics.

The \emph{MLE} baseline minimizes classification negative
log-likelihood and returns one LoRA adapter~\citep{hu2022lora}.
The BAG \emph{MAP} baseline likewise returns one adapter but trains
it with AdamW weight decay of $0.01$
\citep{loshchilov2019decoupled,samplawski2026bayesian}.
Neither procedure represents a posterior over adapter parameters.
\emph{Temperature scaling} changes the logits of an MLE adapter by
a fitted scalar without retraining the adapter
\citep{guo2017calibration}. \emph{Monte Carlo dropout} instead
keeps LoRA dropout active at prediction time and averages ten
stochastic forward passes; the BAG implementation uses a dropout
probability of $0.1$~\citep{gal2016dropout}.

\emph{Deep Ensemble} trains ten LoRA adapters independently with
the MLE objective and averages their predictions
\citep{lakshminarayanan2017simple}. \emph{Laplace LoRA} starts from
a fine-tuned adapter, fits a local Gaussian approximation using a
Kronecker-factored curvature estimate, and samples approximate
logits for prediction~\citep{yang2024bayesianlora}. These methods
obtain predictive diversity by independent optimization or a local
posterior approximation rather than by following a Langevin chain.

Other baselines learn or construct distributions over low-rank
updates. \emph{BLoB} jointly learns a variational mean and
covariance during fine-tuning, using a diagonal Gaussian
approximation for a LoRA factor~\citep{wang2024blob}.
\emph{ScalaBL} performs variational inference in a latent space
whose dimension is the LoRA rank and maps latent samples through
the LoRA factors~\citep{samplawski2025scalable}.
\emph{Training-Free Bayesianization} introduces Gaussian
uncertainty around a trained adapter and searches for an acceptable
noise level without further adapter training
\citep{shi2025trainingfree}. The BAG implementations average ten
predictive samples for each of these three methods.

Our additional \emph{SG-LMC} comparator uses minibatch-averaged
gradients, a fixed step size, and discounted Gaussian perturbations,
following the stochastic Langevin principle~\citep{welling2011sgld}.
It differs from RA-LMC in its gradient estimator and step-size
choice. Related fine-tuning work also uses Langevin sampling for a
different purpose: \emph{BayesTune} samples sparsity scale variables
under a hierarchical prior and uses them to select weights for
subsequent sparse fine-tuning~\citep{kim2023bayestune}.
Its Langevin iterates are not posterior samples of the final
fine-tuned LoRA adapter, so BayesTune addresses parameter selection
rather than the predictive-sampling comparison studied here.

\end{document}